\newif\ifisTR

\isTRtrue

\documentclass{article} 
\usepackage{fullpage}

\usepackage[utf8]{inputenc} 
\usepackage[T1]{fontenc}    
\usepackage{url}            
\usepackage{booktabs}       
\usepackage{amsfonts}       
\usepackage{nicefrac}       
\usepackage{microtype}      
\usepackage{xcolor}         

\usepackage{subcaption}

\usepackage{wrapfig}
\usepackage{overpic}
\usepackage{placeins}
\usepackage{algpseudocode}
\usepackage{multirow}
\usepackage{graphicx}
\usepackage{url}

\usepackage[round,semicolon,sort]{natbib}

\usepackage{tcolorbox}
\usepackage{microtype}
\usepackage{graphicx}
\usepackage{subcaption}
\usepackage{booktabs}
\usepackage{nicefrac}

\definecolor{mydarkblue}{rgb}{0,0.08,0.45}
\usepackage[colorlinks,citecolor=mydarkblue,urlcolor=mydarkblue,linkcolor=mydarkblue]{hyperref}
\usepackage{url}

\renewcommand{\cite}[1]{\citep{#1}}

\usepackage{amsmath,amsfonts,bm}
\usepackage{dsfont}
\usepackage[english]{babel}
\usepackage{amssymb,amsthm}

\newtheorem{theorem}{Theorem}
\newtheorem{prop}{Proposition}

\def\eqref#1{equation~(\ref{#1})}

\def\1{\bm{1}}

\def\rvvarepsilon{{\boldsymbol{\varepsilon}}}

\def\rvtheta{{\boldsymbol{\theta}}}

\def\rvv{{\mathbf{v}}}

\def\rvy{{\mathbf{y}}}

\def\rmA{{\mathbf{A}}}
\def\rmB{{\mathbf{B}}}

\def\rmD{{\mathbf{D}}}

\def\rmT{{\mathbf{T}}}
\def\rmU{{\mathbf{U}}}
\def\rmV{{\mathbf{V}}}

\def\vx{{\bm{x}}}
\def\vy{{\bm{y}}}

\DeclareMathAlphabet{\mathsfit}{\encodingdefault}{\sfdefault}{m}{sl}
\SetMathAlphabet{\mathsfit}{bold}{\encodingdefault}{\sfdefault}{bx}{n}

\newcommand{\E}{\mathbb{E}}

\usepackage{placeins}
\usepackage{hyperref}
\usepackage{url}
\usepackage{graphicx}

\usepackage[utf8]{inputenc}
\usepackage{xcolor} 
\usepackage{amsmath}
\usepackage{amsfonts}
\usepackage{dcolumn} 
\usepackage{tabu}
\usepackage{array}
\usepackage{xspace}
\usepackage{makecell}
\usepackage{amsthm}

\usepackage{enumitem}

\usepackage{colortbl}
\usepackage{booktabs, multirow} 
\usepackage{soul}
\usepackage{changepage,threeparttable} 

\newcommand{\ourmethod}{\texttt{SharpBalance}\xspace}
\newcommand{\sam}{\texttt{SAM}\xspace}

\begin{document}

\title{Sharpness-diversity tradeoff: \\ improving flat ensembles with SharpBalance}
\date{}
\author{
  Haiquan Lu$^{1}\footnote{First four authors contributed equally.}$, 
  Xiaotian Liu$^{2}\footnotemark[1]$, 
  Yefan Zhou$^{2}\footnotemark[1]$, 
  Qunli Li$^{3}\footnotemark[1]$,  \\
  Kurt Keutzer$^{4}$, 
  Michael W. Mahoney$^{4, 5, 6}$, 
  Yujun Yan$^{2}$,
  Huanrui Yang$^{4}$,
  Yaoqing Yang$^{2}$ \\
  $^1$ Nankai University\\
  $^2$ Dartmouth College\\
  $^3$ University of California San Diego\\
  $^4$ University of California at Berkeley\\
  $^5$ International Computer Science Institute\\
  $^6$ Lawrence Berkeley National Laboratory\\
}

\maketitle

\begin{abstract}
\noindent
Recent studies on deep ensembles have identified the sharpness of the local minima of individual learners and the diversity of the ensemble members as key factors in improving test-time performance.
Building on this, our study investigates the interplay between sharpness and diversity within deep ensembles, illustrating their crucial role in robust generalization to both in-distribution (ID) and out-of-distribution (OOD) data.
We discover a trade-off between sharpness and diversity: minimizing the sharpness in the loss landscape tends to diminish the diversity of individual members within the ensemble, adversely affecting the ensemble's improvement. 
The trade-off is justified through our theoretical analysis and verified empirically through extensive experiments. 
To address the issue of reduced diversity, we introduce \ourmethod, a novel training approach that balances sharpness and diversity within ensembles.
Theoretically, we show that our training strategy achieves a better sharpness-diversity trade-off.
Empirically, we conducted comprehensive evaluations in various data sets (CIFAR-10, CIFAR-100, TinyImageNet) and showed that \ourmethod not only effectively improves the sharpness-diversity trade-off, but also significantly improves ensemble performance in ID and OOD scenarios. 
\end{abstract}
\section{Introduction}
\label{sec:introduction}

\begin{figure}[!th]
\hspace{7cm}
\begin{subfigure}{0.48\linewidth}
\includegraphics[width=\linewidth,keepaspectratio]{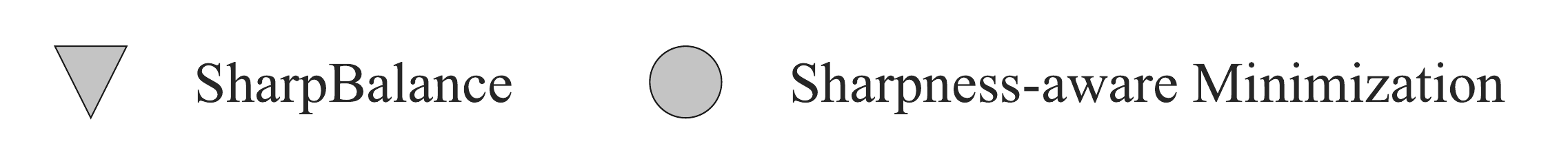} 
\end{subfigure} \\
\centering
\begin{subfigure}{0.4\linewidth}
\includegraphics[width=\linewidth,keepaspectratio]{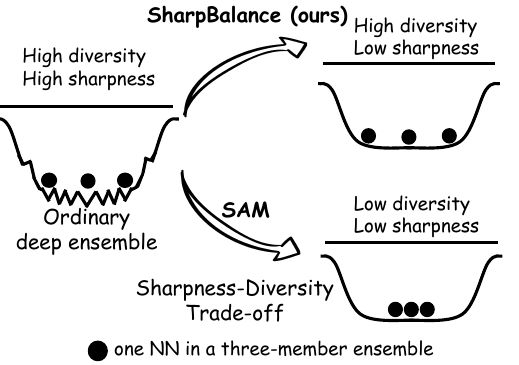} 
\caption{Overview}~\label{fig:teaser-cari}
\end{subfigure} 
\centering     
\begin{subfigure}{0.28\linewidth} 
\includegraphics[width=\linewidth,keepaspectratio]{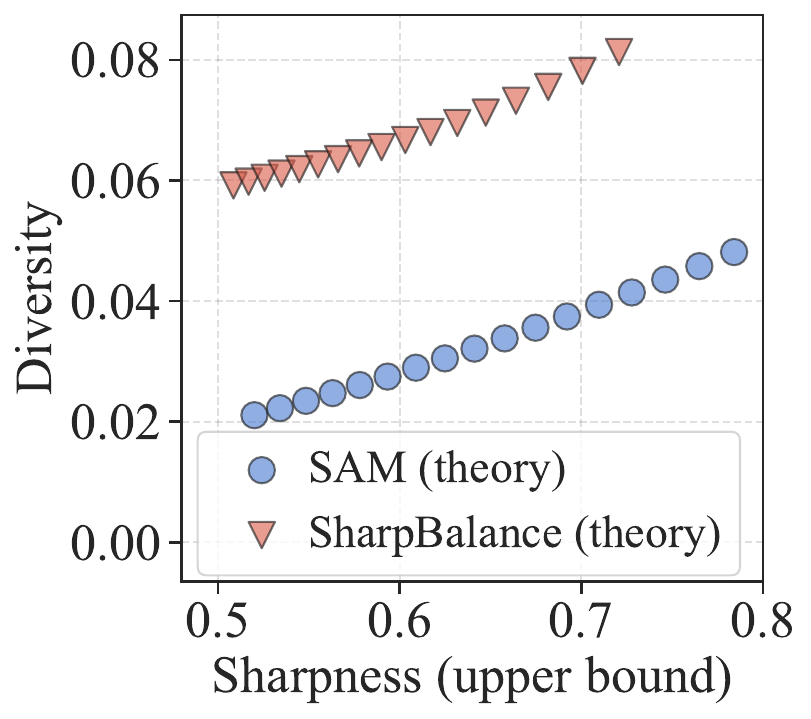} 
\caption{Theoretical results}~\label{fig:teaser-theory}
\end{subfigure} 
\centering     
\begin{subfigure}{0.28\linewidth} 
\includegraphics[width=\linewidth,keepaspectratio]{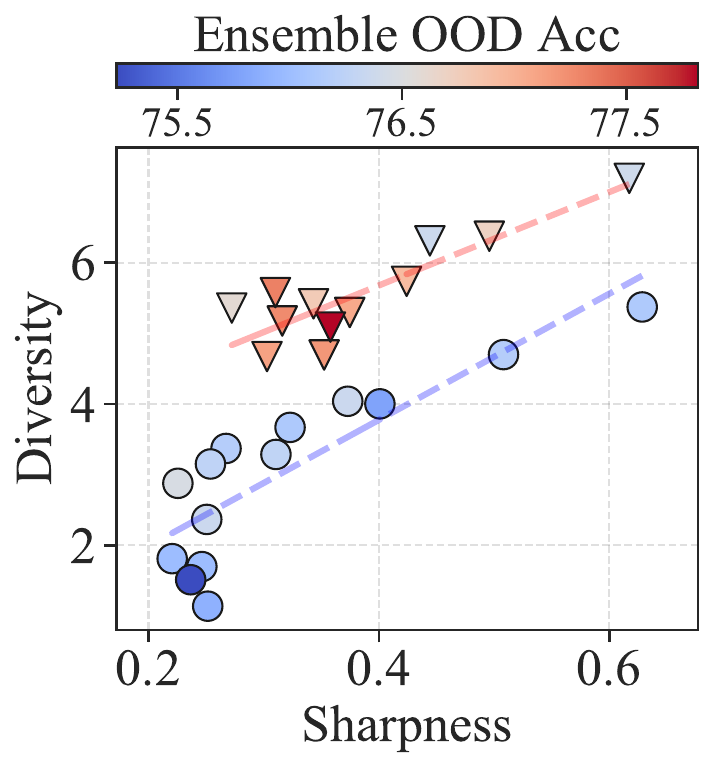} 
\caption{Empirical results}~\label{fig:teaser-cifar}
\end{subfigure} 
\caption{\textbf{(Sharpness-diversity trade-off and \ourmethod).} 
\textbf{(a)} Caricature illustrating the sharpness-diversity trade-off that emerges in an ensemble's loss landscape induced by the Sharpness-aware Minimization (\sam) optimizer.
We propose \ourmethod to address this trade-off.
Each black circle represents an individual NN in a three-member ensemble.
The distance between circles represents the diversity between NNs and the ruggedness of the basin represents the sharpness of each NN.
\textbf{(b)} Theoretically proving the existence of the sharpness-diversity trade-off and improvement from \ourmethod, plotting the analytic representation of sharpness and diversity from Theorem~\ref{thm:varSAM} and Theorem~\ref{thm:sub} by changing the perturbation radius $\rho$ of \sam.
\ourmethod achieves a larger diversity for the same level of sharpness.
\textbf{(c)} Empirical results of verifying sharpness-diversity trade-off improvement from \ourmethod. Each marker represents a three-member ResNet18 ensemble trained on CIFAR-10. Diversity is measured by the variance of individual models' predictions, and sharpness is measured by the adaptive worst-case sharpness, both defined in Section~\ref{sec:notation}.\looseness-1
}
\label{fig:overview} 
\end{figure}

There has been interest in understanding the properties of neural networks (NNs) and their implications for robust generalization to both in-distribution (ID) and out-of-distribution  (OOD) data~\citep{hendrycks2018benchmarking}.
Two properties of particular importance, sharpness (or flatness)~\citep{granziol2020flatness,andriushchenko2023modern,yang2021taxonomizing,dinh2017sharp,yao2020pyhessian} and diversity~\citep{laviolette2017risk, fort2019deep, yao2020pyhessian, dietterich2000ensemble, ortega2022diversity, theisen2023ensembles}, have been shown to have a significant influence on performance.
In the context of \emph{deep ensembles}~\citep{ovadia2019can, lakshminarayanan2017simple, fort2019deep, mehrtash2020pep, ganaie2022ensemble}, diversity (which measures the variance in output between independently-trained models) is shown to be critical in enhancing ensemble accuracy.
Sharpness, on the other hand, quantifies the curvature of local minima and is believed to be empirically correlated with an individual model's generalization ability.

Recent research on loss landscapes~\citep{yang2021taxonomizing} highlights that a single structural property of the loss landscape is insufficient to fully capture a model's generalizability, and it underscores the importance of a joint analysis of sharpness and diversity.
Despite significant efforts in studying sharpness and diversity individually, a gap persists in understanding their relationship, particularly in the context of ensemble learning. 
Our work seeks to bridge this gap by investigating ensemble learning through the lens of loss landscapes, with a specific focus on the interplay between sharpness and diversity.

\noindent {\bf Sharpness-diversity trade-off.}~Our examination of loss landscape structure for ensembling revealed a ``trade-off’’ between the diversity of individual NNs and the sharpness of the local minima to which they converge.
This trade-off introduces a potential limitation to the achievable performance of the deep ensemble: the test accuracy of individual NN can be improved as the sharpness is reduced, but it simultaneously reduces diversity, thereby compromising the ensembling improvement (evidence in Section~\ref{sec:emp-trade-off} and \ref{sec:eval-perf}).
This trade-off is visually summarized in the lower transition branch in Figure~\ref{fig:teaser-cari}.
We also developed theories (in Section~\ref{sec:theory1}) to verify the trade-off. The theoretical results characterizing this phenomenon are visualized in Figure~\ref{fig:teaser-theory}, and the experimental observation is presented in Figure~\ref{fig:teaser-cifar}.
In Section~\ref{sec:emp-trade-off}, we also verified the existence of the trade-off by varying the experimental setting to include different datasets and different levels of overparameterization (e.g., changing model width).

\looseness-1

\noindent {\bf \ourmethod mitigates the trade-off and improves ensembling performance.} 
To address the challenge presented by the sharpness-diversity tradeoff, we propose a novel ensemble training method called \ourmethod.
This method aims to simultaneously reduce the sharpness of individual NNs and prevent diversity reduction among them, as demonstrated in the upper transition branch of Figure~\ref{fig:teaser-cari}.
This method is designed based on our theoretical results, which suggest that training different ensemble members using a loss function that aims to reduce sharpness on different subsets of the training data can improve the trade-off between sharpness and diversity. Our theoretical results are summarized in Figure~\ref{fig:teaser-theory}.
Aligned with theoretical insights, our \ourmethod method lets each ensemble member minimize the sharpness objective exclusively on a subset of training data, termed the \emph{sharpness-aware set}.
The sharpness-aware set of each ensemble member is diversified by an adaptive strategy based on data-dependent sharpness measures.
As shown in Figure~\ref{fig:teaser-cifar}, we verify that \ourmethod improves the sharpness-diversity tradeoff in training the ResNet18 ensemble on CIFAR10. We conducted experiments on three classification datasets to show that \ourmethod boosts ensembling performance in ID and OOD data.

Our contributions are summarized as follows: 
\begin{itemize}[leftmargin=*]
\item \textbf{New discovery:} We identify a phenomenon in ensemble learning called the sharpness-diversity trade-off, where reducing the sharpness of individual models can decrease diversity between models within an ensemble.
We show that this trade-off can negatively affect the ensemble improvements.

\item  \textbf{Novel theory:} We prove the existence of the trade-off under a novel theoretical framework based on rigorous analysis of sharpness-aware training objectives~\citep{foret2020sharpness, behdin2023statistical}. Our analysis borrows tools from analyzing Wishart moments~\citep{bishop2018introduction}, and it characterizes the exact dynamics of training, bias-variance tradeoff, and the upper and lower bounds of sharpness. Notably, our novel theoretical analysis generalizes existing analysis to ensemble members trained with different data, which is the key to analyzing our own training method \ourmethod.

\item \textbf{Effective approach:} To mitigate the sharpness-diversity trade-off, we introduce \ourmethod, an ensemble training approach.
Our theoretical framework demonstrates that \ourmethod improves the sharpness-diversity trade-off by reducing sharpness while mitigating the decrease in diversity.
Empirically, we confirm this improvement and demonstrate that \ourmethod enhances overall ensemble performance, outperforming baseline methods in CIFAR-10, CIFAR-100~\citep{Krizhevsky2009LearningML}, TinyImageNet~\citep{Le2015TinyIV}, and their corrupted versions to assess OOD performance.

\end{itemize}

We provide a more detailed discussion on related work in Appendix~\ref{sec:related-work}.

\section{Background}~\label{sec:notation} \vspace{-6mm}

\noindent
{\bf Preliminaries.}
We use a NN denoted as $f_\rvtheta:\mathbb{R}^{d_{\text{in}}} \rightarrow \mathbb{R}^{d_{\text {out }}}$, where $\rvtheta \in \mathbb{R}^p$ denotes the trainable parameters. 
The training dataset comprises $n$ data-label pairs $\mathcal{D}=\left\{\left(\vx_1, \vy_1\right), \ldots,\left(\vx_n, \vy_n\right)\right\}$.
The training loss of NN $f_\rvtheta$ over a dataset $\mathcal{D}$ can be defined as $\mathcal{L}_\mathcal{D}(\rvtheta) = 
\frac{1}{n}\sum_{i=1}^n \ell\left(f_{\rvtheta}\left(\vx_i\right), \vy_i\right)$.
Here $\ell(\cdot)$ is a loss function, which, for instance, can be the cross entropy loss or $\ell_2$ loss.
We construct a deep ensemble consisting of $m$ distinct NNs $f_{\rvtheta_1}$, \dots, $f_{\rvtheta_m}$.
For classification tasks, the ensemble's output is derived by averaging the predicted logits of these individual networks.
We use \emph{flat ensemble} to mean the deep ensemble in which each ensemble member is trained using a sharpness-aware optimization method~\citep{foret2020sharpness}, differentiating it from other ensemble approaches.

\noindent
{\bf Diversity metrics.} Distinct measures of diversity have been proposed in the literature~\citep{laviolette2017risk, fort2019deep, dietterich2000ensemble, baek2022agreement, ortega2022diversity, theisen2023ensembles}, and they are primarily calculated using the predictions made by individual models. 
\citet{ortega2022diversity} define diversity $\mathbb{D}(\rvtheta)$ to be the variance of model outputs averaged over the data-generating distribution, which we adopt in the theoretical analysis:
\looseness-1
\begin{equation}
\label{eq:defdiv}
    \mathbb{D}(\rvtheta) =  \E_{\mathcal{D}}[\text{Var}(f_\rvtheta(\mathcal{D}))].
\end{equation}

In our experiments, diversity is measured using variance defined above, as well as two other widely used metrics in ensemble learning, namely Disagreement Error Ratio (DER) \citep{theisen2023ensembles} defined in \eqref{eqn:disagreement}, and KL divergence~\citep{10.1214/aoms/1177729694} defined in \eqref{eqn:kl-divergence} in the appendices. 
We show in Section~\ref{sec:emp-trade-off} that our main claim generalizes to these three metrics in characterizing the diversity between members within an ensemble.
Specifically, denote $\mathcal{P}$ as the distribution of model weights $\rvtheta$ after training.
Then, the DER is defined as 
\begin{equation}\label{eqn:disagreement}
\text{DER}=\frac{E_{\rvtheta,\rvtheta^{'} \sim \mathcal{P}}[\operatorname{Dis}(f_\rvtheta,f_{\rvtheta^{'}})]}{E_{\rvtheta \sim \mathcal{P}}[\mathcal{E}(f_\rvtheta)]} ,
\end{equation}
where $\operatorname{Dis}(f_\rvtheta,f_{\rvtheta^{'}})$ is the prediction disagreement~\citep{masegosa2020learning, mukhoti2021deterministic, jiang2022assessing} between two classifier $f_\rvtheta,f_{\rvtheta^{'}}$, and $\mathcal{E}(f_\rvtheta)$ is the prediction error.\looseness-1

\noindent {\bf Sharpness Metric.} In accordance with the definition proposed by ~\citet{foret2020sharpness}, we characterize the \emph{first-order sharpness} of a model as the worst-case perturbation within a radius of $\rho_0$. Mathematically, the sharpness $\kappa$ of a model $\rvtheta$ is expressed as follows:
\begin{equation*}
    \kappa(\rvtheta;\rho_0) = \max_{\|\rvvarepsilon\|_2 \leq \rho_0} \mathcal{L}_\mathcal{D}(\rvtheta +\rvvarepsilon) - \mathcal{L}_\mathcal{D}(\rvtheta).
\end{equation*}
Empirically, we measure the sharpness of the NN via the adaptive worst-case sharpness~\citep{kwon2021asam, andriushchenko2023modern}.
The adaptive worst-case sharpness captures how much the loss can increase within the perturbation radius $\rho_0$ of $\rvtheta$:
\begin{equation}\label{eqn:sharpness}
    {\underset{\| T_\rvtheta^{-1} \rvvarepsilon \| _2 \leq \rho_0}{\text{max}} }\mathcal{L}_{\mathcal{D}}(\rvtheta+\rvvarepsilon)-\mathcal{L}_{\mathcal{D}}(\rvtheta),  \\
\end{equation}  
where $\rvtheta = \left[\theta_1, \ldots, \theta_l \right]$, and $T_\rvtheta = \operatorname{diag}\left(\left|\theta_1\right|, \ldots,\left|\theta_l\right|\right)$.
$T_\rvtheta^{-1}$ is a normalization operator to make sharpness ``scale-free'', that is, such that scaling operations on $\boldsymbol{\theta}$ that do not alter NN predictions will not impact the sharpness measure.

\noindent
{\bf Ensembling.}
We characterize the effectiveness of ensembling by the metric called ensemble improvement rate (EIR)~\citep{theisen2023ensembles}, which is defined as the ensembling improvement over the average performance of single models. Let $\mathcal{E}_\text{ens}$ denote the test error of an ensemble; 
the EIR is then defined as follows:
\begin{equation}\label{eqn:eir}
\text{EIR}=\frac{ E_{\rvtheta \sim \mathcal{P}}[\mathcal{E}(f_\rvtheta)] - \mathcal{E}_\text{ens} }{ E_{\rvtheta \sim \mathcal{P}}[\mathcal{E}(f_\rvtheta)] }.
\end{equation}

\noindent 
{\bf Sharpness Aware Minimization (\sam).}
\sam~\citep{foret2020sharpness} has been shown to be an effective method for improving the generalization of NNs by reducing the sharpness of local minima. 
It essentially functions by penalizing the maximum loss within a specified radius $\rho$ of the current parameter $\theta$.
The training objective of \sam is to minimize the following loss function:
\begin{equation}\label{eqn:SAM_objective}
\mathcal{L}_\mathcal{D}^\text{SAM}(\rvtheta) := 
 \underset{\| \rvvarepsilon \| _2 \leq \rho}{\text{max}} \mathcal{L}_\mathcal{D}(\rvtheta+\rvvarepsilon) + \lambda \| \rvtheta \|^2_2,
\end{equation}
where $\lambda$ is the hyperparameter of a standard $\ell_2$ regularization term.

\section{Theoretical Analysis of Sharpness-diversity Trade-off}
\label{sec:theory1}
This section theoretically analyzes the sharpness-diversity trade-off. The diversity among individual models is quantified using~\eqref{eq:defdiv}.
The first theorem establishes the existence of a trade-off between sharpness and diversity. The second theorem demonstrates that training models with only a subset of data samples leads to a more favorable trade-off between these two metrics.\looseness-1

\noindent \textbf{Sharpness and Diversity of \sam.}
Assume the training data matrix $\rmA \in \mathbb{R}^{n_{\text{tr}} \times d_{\text{in}}}$ and test data matrix $\rmT \in \mathbb{R}^{n_{\text{te}} \times d_{\text{in}}}$ are random with entries drawn from Gaussian $\mathcal{N}(0,  \frac{1}{d_{\text{in}}} \mathbf{I})$. 
Suppose the model weight at the 0-th time step, $\rvtheta_0$, is initialized randomly such that $\mathbb{E}[\rvtheta_0] = \mathbf{0}$ and $\mathbb{E}[\rvtheta_0 \rvtheta_0^T] = \sigma^2\mathbf{I}$ and updated with a quadratic optimization objective through \sam. The learned weight matrix after $k$ time steps is denoted as $\rvtheta_k$.
Let $\rvtheta^*$ be the teacher model (i.e., ground-truth model) such that $\rmA\rvtheta^* = \rvy^{(\rmA)}$ and $\rmT\rvtheta^* = \rvy^{(\rmT)}$, where $\rvy^{(\rmD)}$ is the label vector for data matrix $\rmD$. Given a perturbation radius $\rho_0$, the sharpness of a model after $k$ iteration under the random matrix assumption is defined as 
    \begin{equation*}
        \kappa(\rvtheta_k;\rho_0) = \mathbb{E}_{\rmA}[\max_{\|\rvvarepsilon\|_2 \leq \rho_0} f \left( \mathbb{E}_{\rvtheta_0} \left[ \rvtheta_k \right] + \rvvarepsilon; \rmA \right) - f \left( \mathbb{E}_{\rvtheta_0} \left[ \rvtheta_k \right] ;\rmA \right)], 
    \end{equation*}
 which is the expected fluctuation of the model output after perturbation over the data distribution. For simplicity, we denote $\kappa(\rvtheta^{SAM}_k;\rho_0) = \kappa_k^{SAM}$ for the rest of the paper. We derive an explicit formulation of diversity and upper and lower bounds of sharpness for models optimized with \sam in Theorem \ref{thm:varSAM}. Detailed proof can be found in Appendix \ref{apd:thm1}.

\begin{theorem}[Diversity and Sharpness of \sam]
    \label{thm:varSAM}
    Let $\rvtheta_0$ be initialized randomly such that $\mathbb{E}[\rvtheta_0] = \mathbf{0}$ and $\mathbb{E}[\rvtheta_0 \rvtheta_0^T] = \sigma^2\mathbf{I}$. Suppose $\rvtheta^{SAM}_k$ is the model weight after $k$ iterations of training with \sam on $\rmA \in \mathbb{R}^{n_{\textup{tr}} \times d_{\textup{in}}}$ and evaluated on $\rmT \in \mathbb{R}^{n_{\textup{te}} \times d_{\textup{in}}}$. Let $\eta$ be the step size, $\rho$ be the perturbation radius in \sam and $\rho_0$ be the radius for measuring sharpness $\kappa^{SAM}_k$. Then
    \begin{align*}
    \mathbb{D}(\rvtheta_k^{SAM}) &=\phi(2k,0)\sigma^2,\\
     \frac{\rho_0^2}{2}\left(\sqrt{\frac{n_{\textup{tr}}}{d_{\textup{in}}}} - 1\right)^2+\rho_0\sqrt{\phi(2k, 2)}\|\rvtheta^*\|_2 - G \leq &\kappa_k^{SAM} \leq \frac{\rho_0^2}{2}\left(\sqrt{\frac{n_{\textup{tr}}}{d_{\textup{in}}}} + 1\right)^2+\rho_0\sqrt{\phi(2k, 2)}\|\rvtheta^*\|_2,
    \end{align*}
where 
    \begin{align*}
        \phi(i,j) :=& \mathds{1}_{j=0}+\sum\limits_{k_1+k_2+k_3=i}\frac{i!}{k_1!k_2!k_3!}(-\eta)^{k_2+k_3}\rho^{k_3}\left(\frac{n_{\textup{tr}}}{d_{\textup{in}}}\right)^{m}\sum\limits_{l=1}^{m}\left(\frac{d_{\textup{in}}}{n_{\textup{tr}}}\right)^{m-l}\mathcal{O}(1+1/d_\textup{in})N_{m,l},
    \end{align*}
$ G = \frac{\phi(4k,4)-\phi(2k,2)^2}{2\phi(2k,2)^{3/2}\|\rvtheta^*\|_2}$, and $m = k_2+2k_3+j$. $N_{m,l} = \frac{1}{l} \binom{m-1}{l-1}\binom{m}{l-1}$ is the Narayana number.
\end{theorem}

To provide a clearer understanding of the relationship between sharpness and diversity, Figure \ref{fig:theorySAM} presents a trade-off curve between these two metrics.
The estimated sharpness and diversity are displayed on the $x$ and $y$ axes, respectively. Each point in the plot corresponds to a model trained using \sam with a different $\rho$ value, showcasing the outcome of varying perturbation radius. In these experiments, we evaluated the sharpness and diversity of the models empirically and compared them to the estimates obtained using Theorem \ref{thm:varSAM}. The soundness of Theorem \ref{thm:varSAM} and tightness of the derived bounds are further supported by empirical evidence, as depicted in Figure \ref{fig:theorySAM}. Further verification results supporting our theoretical analysis are provided in Appendix~\ref{apd:verify}

\noindent \textbf{Training with Data Subsets.} Assume $\rmA$ is partitioned into $S$ horizontal submatrices, such that $\rmA=[\rmA_1^T,\rmA_2^T,\ldots,\rmA_S^T]^T$. We show in Theorem \ref{thm:sub} a similar analysis of the sharpness and diversity of ensembles for which each model is trained with a submatrix. Under this setting, we first selected a subset of data $\rmA_s$ uniformly at random and then train the model with the selected subset with \sam.

\begin{wrapfigure}{r}{0.45\textwidth}
\centering
\vspace{-5mm}
\includegraphics[width=1.0\linewidth,keepaspectratio]{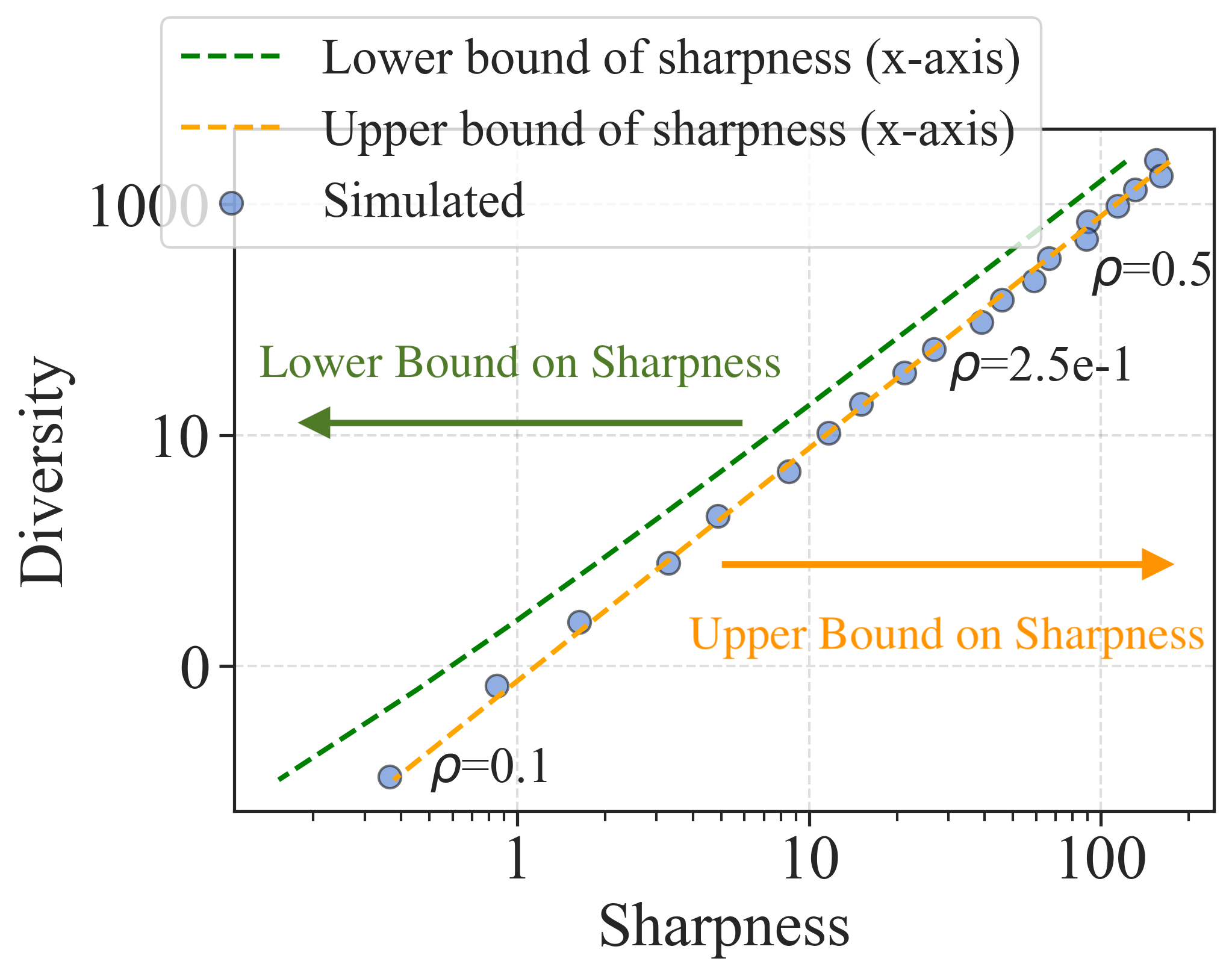}
\caption{\textbf{(Theoretical vs. Simulated sharpness-diversity trade-off)}. This figure illustrates the relationship between sharpness (upper and lower bounds) and diversity as predicted by Thereom \ref{thm:varSAM} and as observed in simulations. Note that the upper and lower bounds correspond to the sharpness values plotted along the x-axis, with the upper bound positioned to the right and the lower bound to the left. Also, note that the bounds provided are for the expected sharpness, which means that random fluctuations can cause the simulation results to move beyond these bounds.\looseness-1
}

\label{fig:theorySAM}
\end{wrapfigure}
\begin{theorem}[Diversity and Sharpness when Models are Trained on Subsets]
\label{thm:sub}
Suppose the training data matrix $\rmA$ is partitioned into $S$ horizontal submatrices. Let $\rvtheta_0$ be initialized randomly such that $\mathbb{E}[\rvtheta_0] = \mathbf{0}$ and $\mathbb{E}[\rvtheta_0 \rvtheta_0^T] = \sigma^2\mathbf{I}$. Let $\rvtheta_k^{SharpBal}$ be the model weight trained with \sam for $k$ iterations on the submatrix $\rmA_s \in \mathbb{R}^{\frac{n_{\textup{tr}}}{S}\times d_\textup{in}}$, selected uniformly at random, and evaluated on test data $\rmT \in \mathbb{R}^{n_{\textup{te}}\times d_\textup{in}}$. Let $\eta$ be the step size, $\rho$ be the perturbation radius in \sam, $\rho_0$ be the radius for measuring sharpness $\kappa^{SAM}_k$, and $r = \frac{n_{\textup{tr}}}{Sd_{\textup{in}}}$. Then
     \begin{align*}
        \mathbb{D}(\rvtheta_k^{SharpBal}) = & \phi'(2k,0)\sigma^2 + \frac{S-1}{d_\textup{in}S}\left(\phi'(2k,0)-\phi'(k,0)^2\right)\|\rvtheta^*\|^2_2,
    \end{align*}
and
\begin{equation*}
    \kappa_k^{SharpBal}(\rho_0) \leq \frac{\rho_0^2}{2}\left(\sqrt{\frac{n_{\textup{tr}}}{d_{\textup{in}}}} + 1\right)^2+\frac{\rho_0}{S}\sqrt{C}\|\rvtheta^*\|_2,
\end{equation*}
where 
    \begin{align*}
        C =& S\phi'(2k, 2)+2rS(S-1)\phi'(2k, 1) + 2S(S-1)  \phi'(k, 2)  \phi'(k, 0)\\
        & +r(1+r) S(S-1)   \phi'(2k, 0)+ 2S(S-1) \phi'(k, 1) \phi'(k, 1) \\ 
        &+ \frac{3}{2}r(1+r)S(S-1)(S-2)  \phi'(k, 0)^2+ \frac{3}{2}r^2S(S-1)(S-2)\phi'(2k, 0)\\
        &+ 3rS(S-1)(S-2)\phi'(k, 0)\phi'(k, 1) + r^2S(S-1)(S-2)(S-3)\phi'(k, 0)^2,\\
        \phi'(i,j) &:= \mathds{1}_{j=0}+\sum\limits_{k_1+k_2+k_3=i}\frac{i!}{k_1!k_2!k_3!}(-\eta)^{k_2+k_3}\rho^{k_3}\left(\frac{n_{\textup{tr}}}{Sd_{\textup{in}}}\right)^{m}\sum\limits_{l=1}^{m}\left(\frac{Sd_{\textup{in}}}{n_{\textup{tr}}}\right)^{m-l}\mathcal{O}(1+\frac{1}{d_\textup{in}} )N_{m,l},
    \end{align*}
    where $m = k_2+2k_3+j$. $N_{m,l} = \frac{1}{l} \binom{m-1}{l-1}\binom{m}{l-1}$ is the Narayana number.
\end{theorem}

The proof of Theorem \ref{thm:sub} is provided in Appendix \ref{apd:proof2}. Similar experimental validations are conducted to verify Theorem \ref{thm:sub}, with results also presented in Appendix \ref{apd:proof}. The main insight from Theorem \ref{thm:sub} is that training models on a randomly selected data subset offers a better trade-off between sharpness and diversity compared to training on the complete dataset. This idea is further illustrated in Figure \ref{fig:teaser-theory}, where we compare the sharpness upper bound and diversity of models trained on the full dataset (labeled as \sam) and those trained on subsets (labeled as \ourmethod).
The results demonstrate that \ourmethod achieves a more favorable trade-off. For a given level of sharpness, deep ensembles with models trained on subsets of the data exhibit higher diversity compared to those trained on the entire dataset.
This indicates that minimizing sharpness on randomly sampled data subsets for each model within the ensemble promotes the diversity among the models, thereby enhancing the sharpness-diversity trade-off.

\section{Experiments}
\label{sec:flat-div}
In this section, we describe our experiments. 
In particular, following Section~\ref{sec:flat-div_SETUP} where we describe our experimental setup, in
Section~\ref{sec:emp-trade-off}, we provide an empirical evaluation across various datasets to explore the trade-off between sharpness and diversity. 
We also examine how this trade-off changes with different levels of overparameterization.
Then, in Section~\ref{sec:sharpbalance-method} and \ref{sec:eval-perf}, we elaborate the \ourmethod algorithm and compare its performance with baseline methods.

\subsection{Experimental setup} 
\label{sec:flat-div_SETUP}
Here, we describe the experiment setup for Section~\ref{sec:emp-trade-off}.
Each ensemble member is trained individually using \sam with a consistent perturbation radius $\rho$, as defined in \eqref{eqn:SAM_objective}. 
We adjust $\rho$ across different ensembles to achieve varying levels of minimized sharpness.
Sharpness for each NN was measured using the adaptive worst-case sharpness metric, defined in \eqref{eqn:sharpness}.
The sharpness measurement was done on the training set, using 100 batches of size 5.
The diversity between NNs is measured using DER defined in \eqref{eqn:disagreement}.
The diversity between ensemble members is tested on OOD data.
We evaluated this trade-off using a variety of image classification datasets, including CIFAR-10, CIFAR-100~\citep{Krizhevsky2009LearningML}, TinyImageNet~\citep{Le2015TinyIV}, and their corrupted versions~\citep{hendrycks2019benchmarking}. 
For the setup of Section~\ref{sec:eval-perf}, we used the same datasets and architecture.
The hyperparameters of the baseline methods has been carefully tuned.
The hyperparameters for conducting the experiments are detailed in Appendix~\ref{abl:hyper}.
\looseness-1

\begin{figure*}[!th]
\centering
\begin{subfigure}{0.32\linewidth}
\includegraphics[width=\linewidth,keepaspectratio]{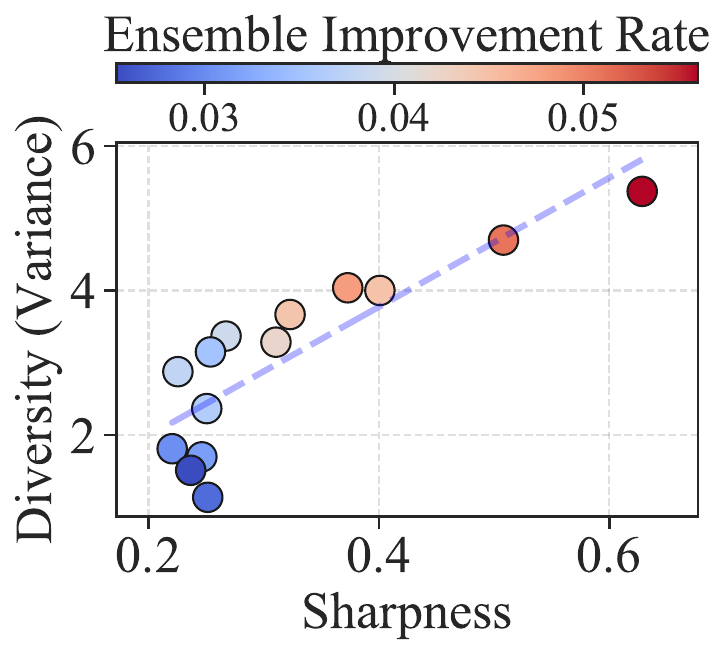}
\caption{Measuring diversity via Variance}
\end{subfigure} 
\centering
\begin{subfigure}{0.33\linewidth}
\includegraphics[width=\linewidth,keepaspectratio]{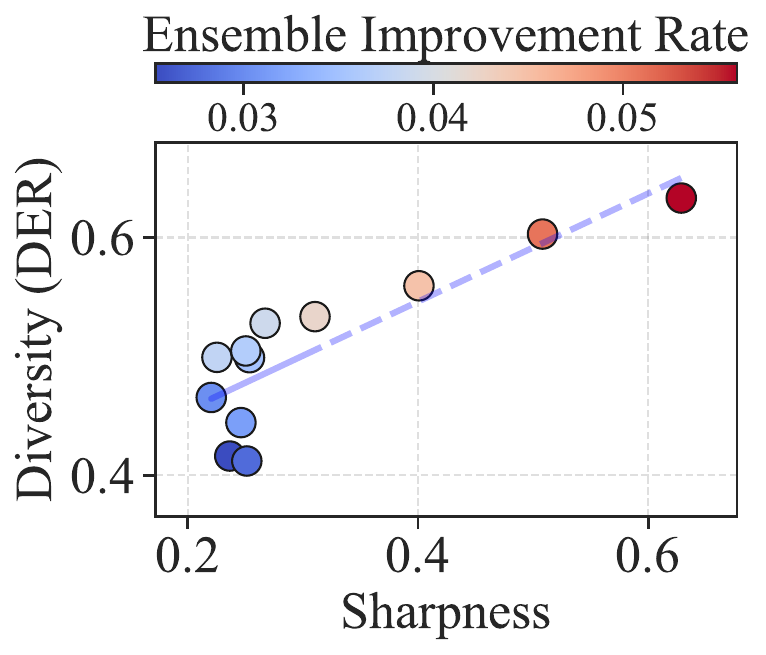}
\caption{Measuring diversity via DER}
\end{subfigure}
\begin{subfigure}{0.33\linewidth}
\includegraphics[width=\linewidth,keepaspectratio]{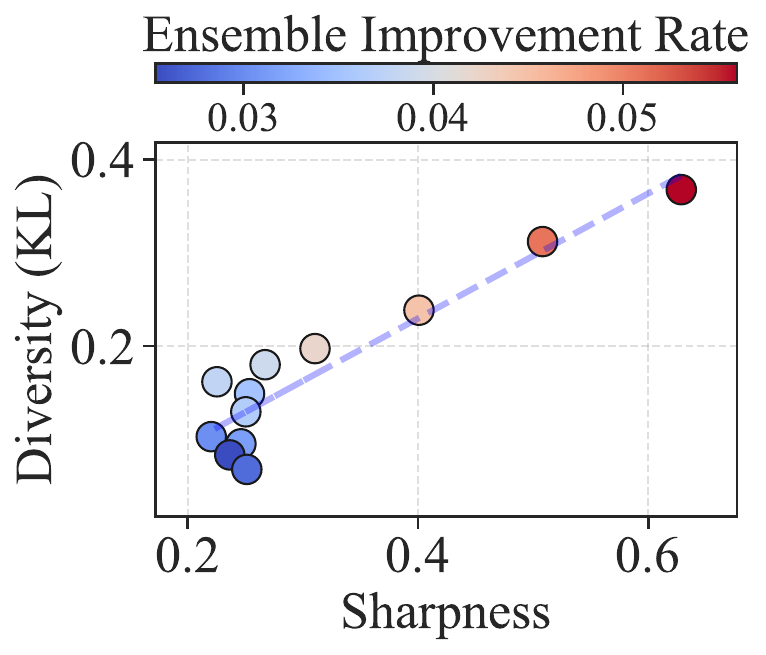}
\caption{Measuring diversity via KL}
\end{subfigure}  
\caption{\textbf{(Varying diversity measure in empirical study).} Three different metrics are employed to measure the diversity of individual models within an ensemble, i.e., Variance in \eqref{eq:defdiv}, DER in \eqref{eqn:disagreement}, and KL divergence in \eqref{eqn:kl-divergence}.
The results of the three metrics show consistent trends, demonstrating the sharpness-diversity trade-off: lower sharpness is correlated with lower diversity. 
The experiment is conducted by training a three-member ResNet18 ensemble on CIFAR10.
}
\label{fig:trade-off-evidence-vary-diversity}  
\end{figure*}

\begin{figure*}[!th]
\centering
\begin{subfigure}{0.32\linewidth}
\includegraphics[width=\linewidth,keepaspectratio]{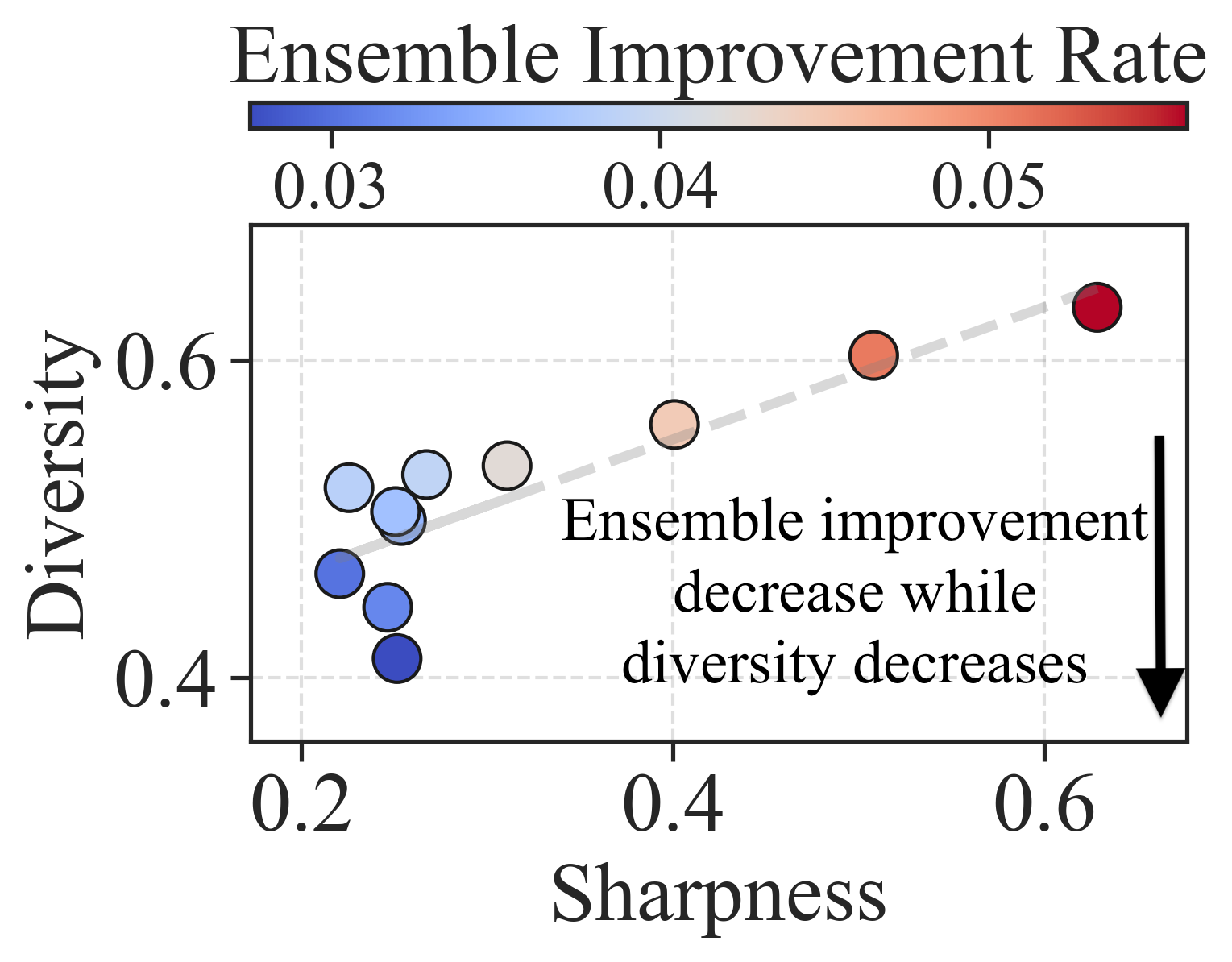}
\end{subfigure}
\centering
\begin{subfigure}{0.32\linewidth}
\includegraphics[width=\linewidth,keepaspectratio]{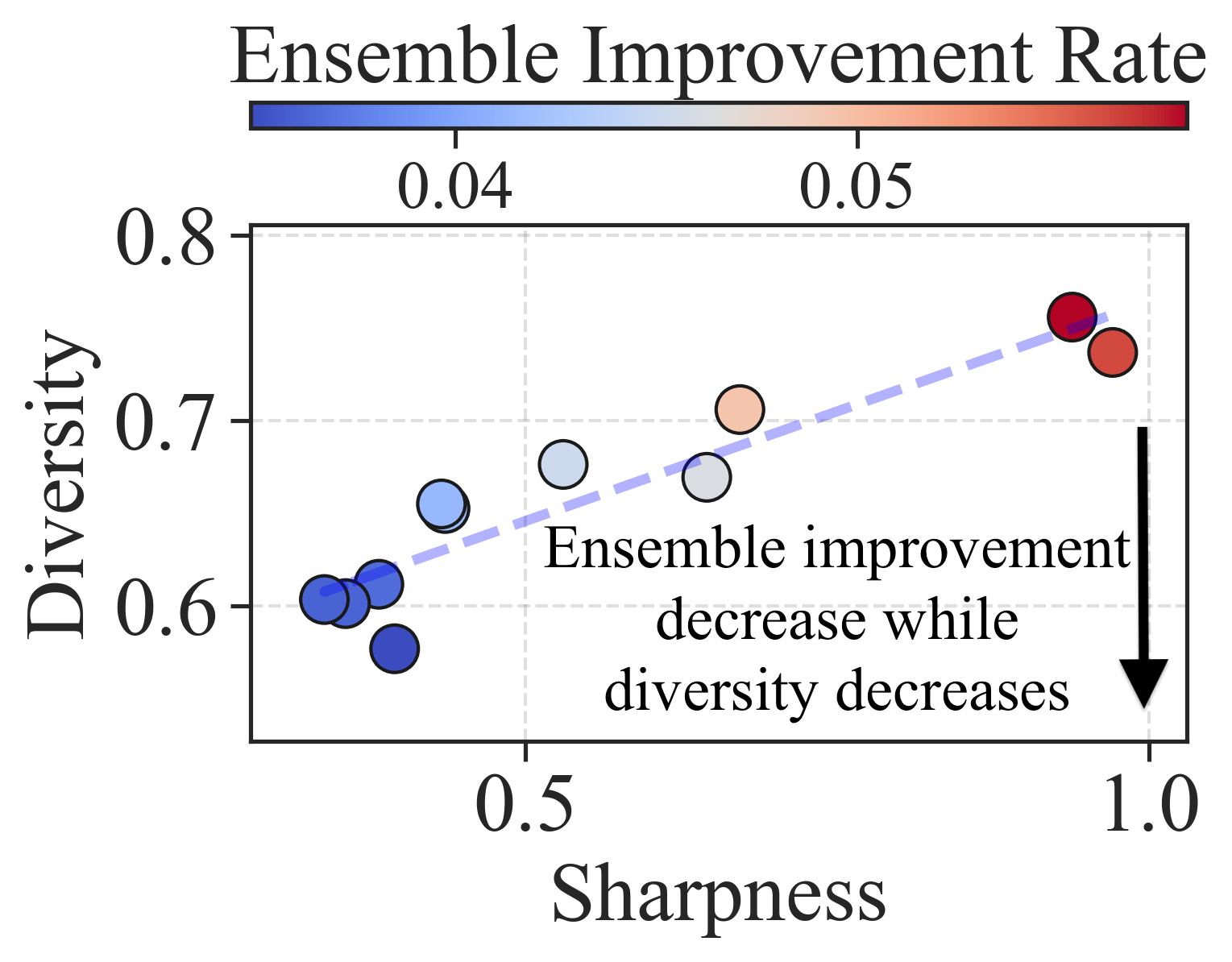} 
\end{subfigure}  
\begin{subfigure}{0.32\linewidth}
\includegraphics[width=\linewidth,keepaspectratio]{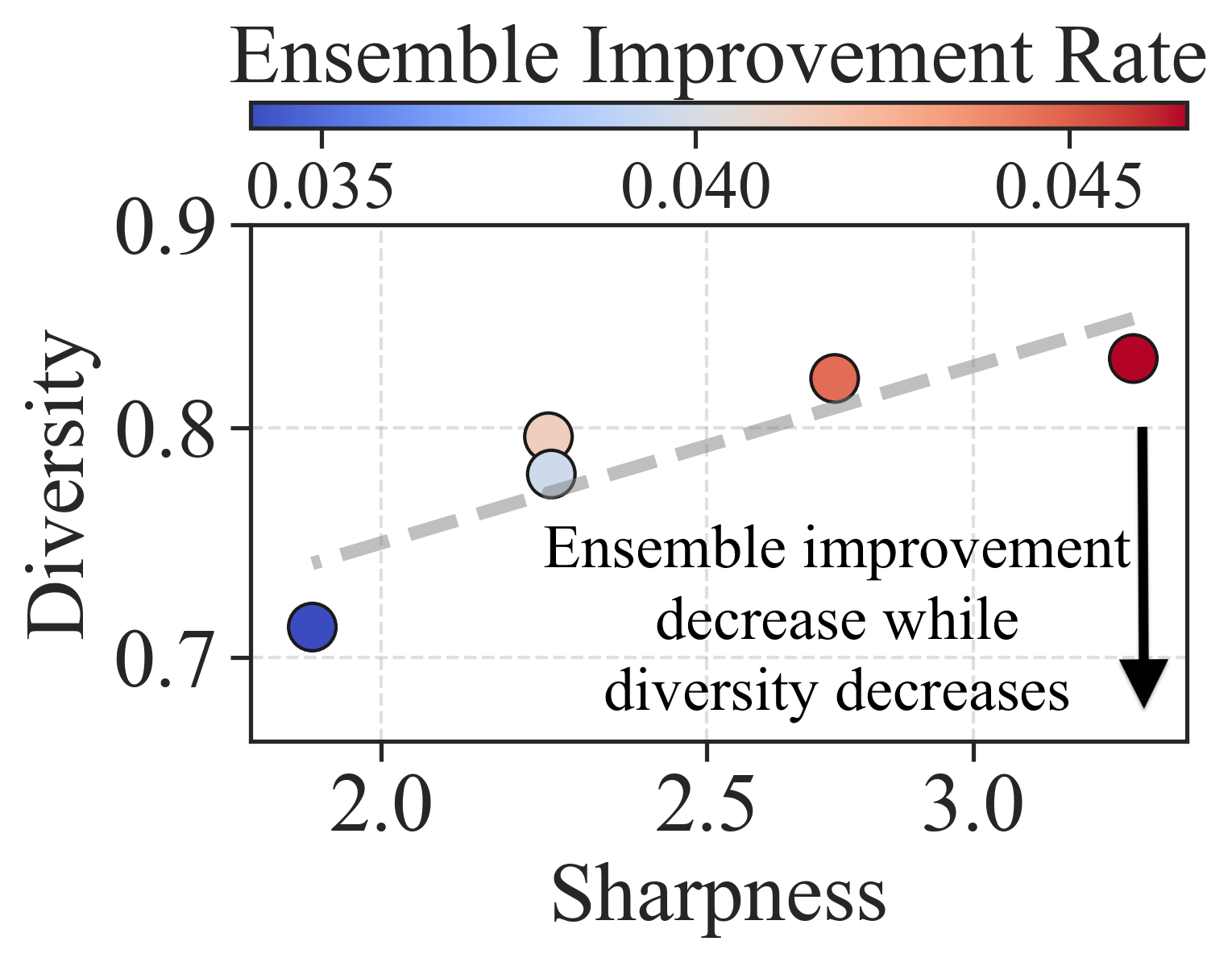} 
\end{subfigure}  
\begin{subfigure}{0.32\linewidth}
\includegraphics[width=\linewidth,keepaspectratio]{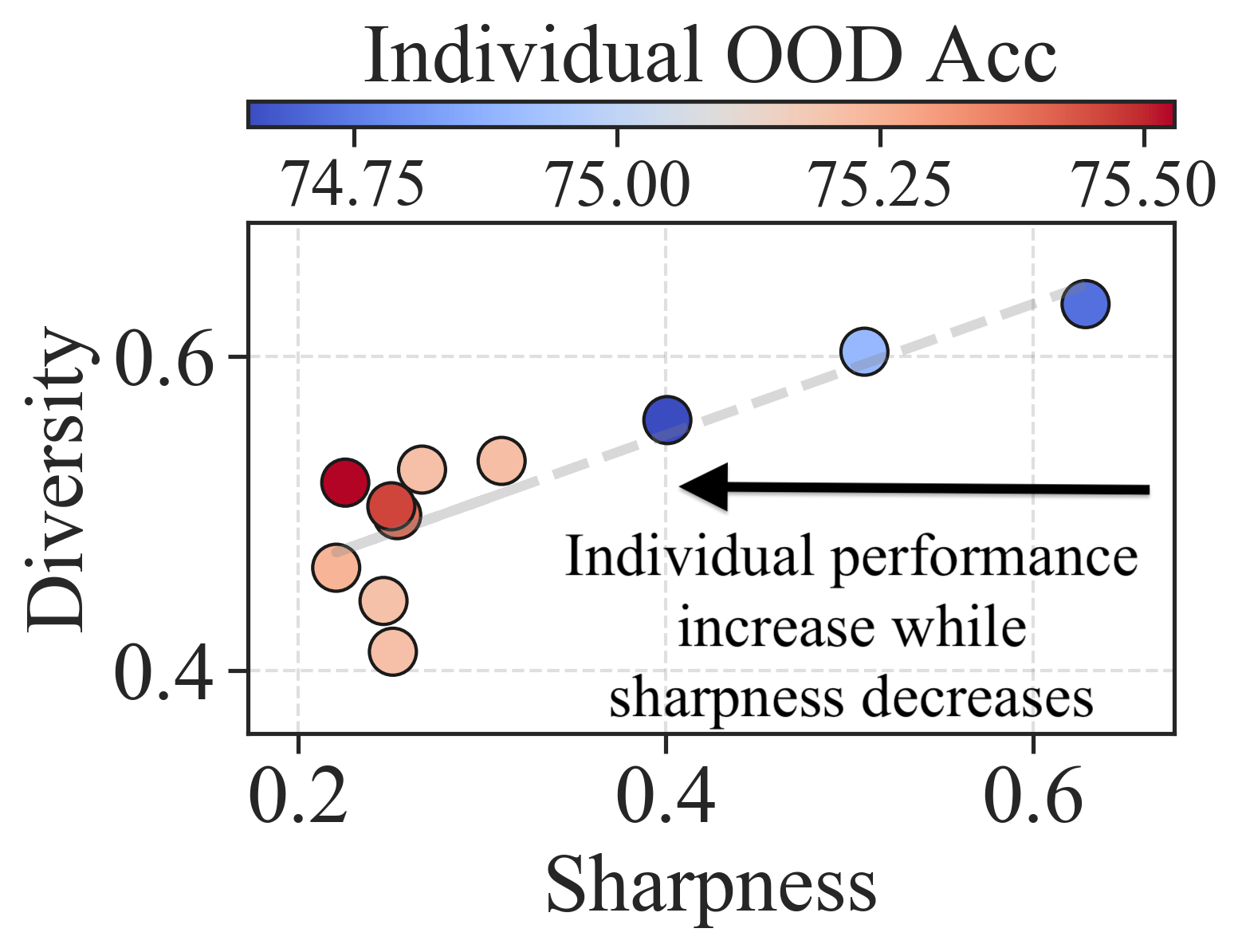}
\caption{CIFAR-10}  
\end{subfigure}
\centering
\begin{subfigure}{0.32\linewidth}
\includegraphics[width=\linewidth,keepaspectratio]{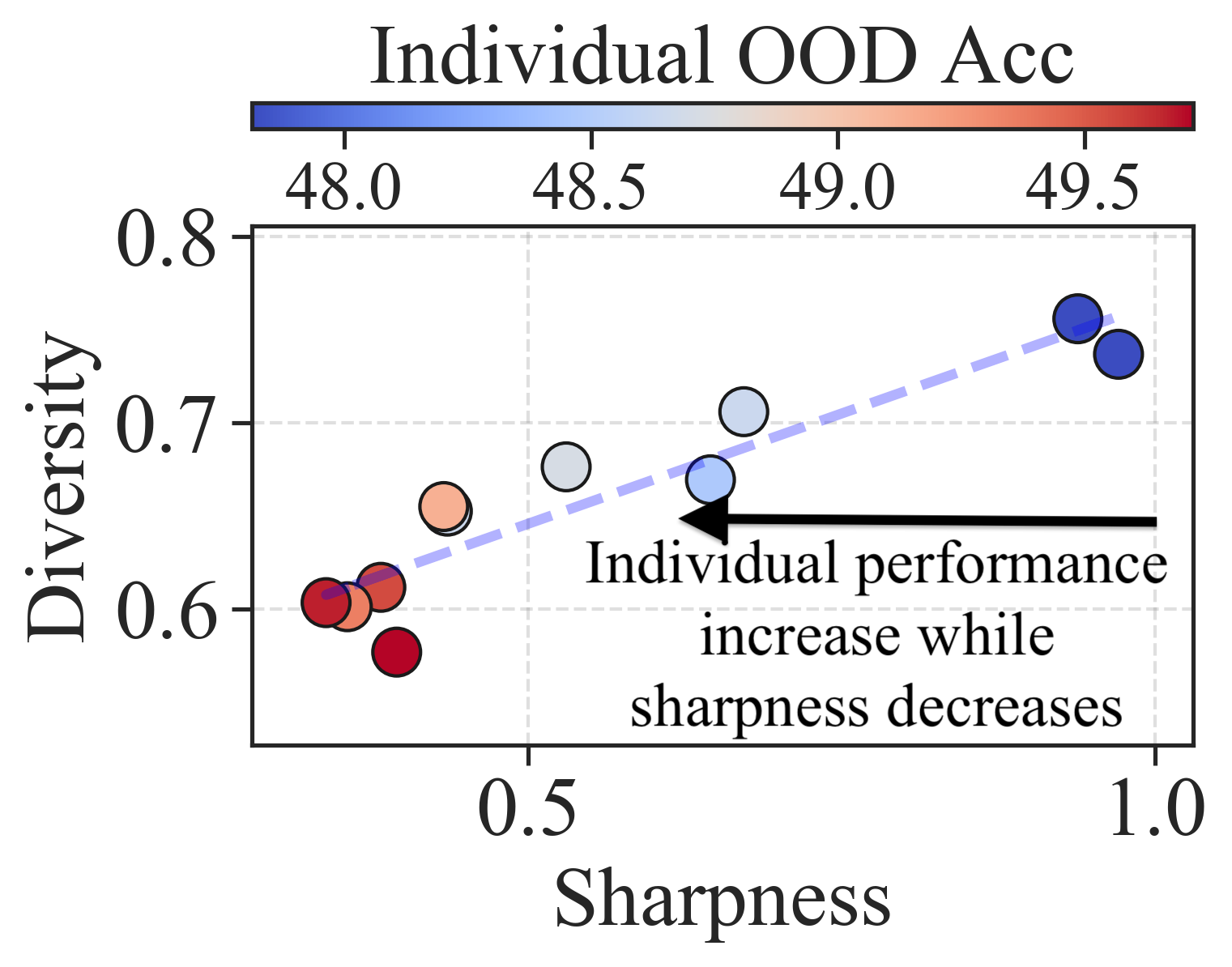} 
\caption{CIFAR-100} 
\end{subfigure}  
\begin{subfigure}{0.32\linewidth}
\includegraphics[width=\linewidth,keepaspectratio]{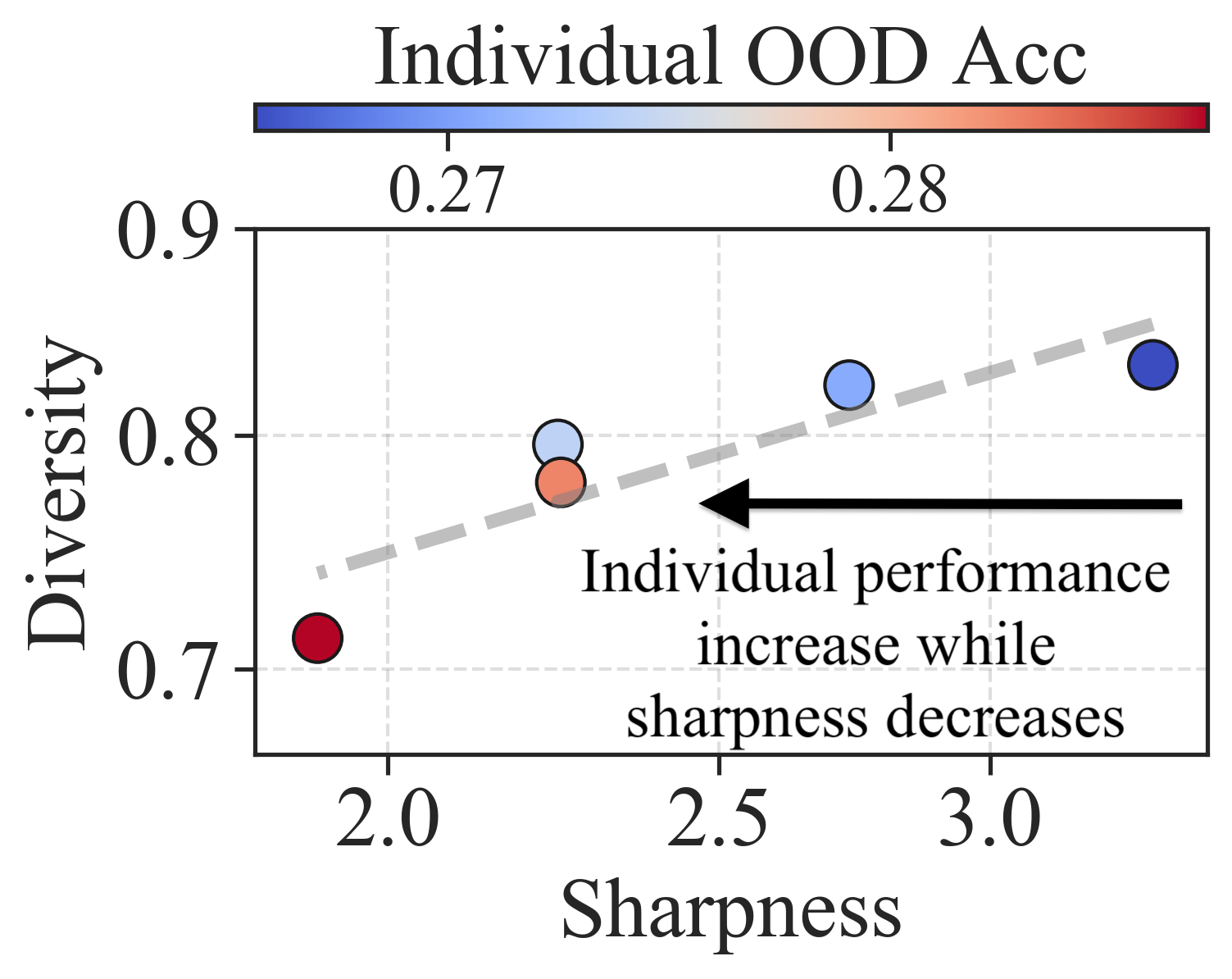} 
\caption{TinyImageNet} 
\end{subfigure}  
\caption{\textbf{(Empirical observations of sharpness-diversity trade-off).} The identified trade-off shows that while reducing sharpness enhances individual model performance, it concurrently lowers diversity and thus diminishes the ensemble improvement rate.
\emph{First row}: the color encoding represents the ensemble improvement rate (EIR) defined in \eqref{eqn:eir}, from red to blue means ensembling improvement decreases.
\emph{Second row}: the color encoding represents the individual ensemble member's OOD accuracy, from blue to red means individual performance becomes better.
Each marker represents a three-member ResNet18 ensemble trained with \sam with a different perturbation radius. 
}
\label{fig:trade-off-evidence}  
\end{figure*}

\subsection{Empirical validation of Sharpness-diversity trade-off}\label{sec:emp-trade-off}

We provide empirical observation to validate and explore the sharpness-diversity trade-off.
Figure~\ref{fig:trade-off-evidence-vary-diversity} presents the validation of observing the trade-off phenomenon on training ResNet18 ensembles on CIFAR10 applying three different metrics to measure the diversity. 
The results demonstrate that this trade-off phenomenon generalizes to the three diversity metrics defined in Section~\ref{sec:notation}.
Figure~\ref{fig:trade-off-evidence} presents the validation on three different datasets. 
In the following empirical study, DER will be the primary metric for measuring diversity of models.
 
Experimental results obtained with the other two metrics are available in Appendix \ref{abl:abl-loss}. 
The three sets of results first verify that minimizing individual member's sharpness indeed reduces diversity. 
This is confirmed by the consistent trends of markers moving from upper right to lower left.
Second, the first row of Figure~\ref{fig:trade-off-evidence} shows that an ensemble with decreased diversity (lower in $y$-axis) shows a lower ensemble improvement rate (from red to blue), highlighting the negative impact of this trade-off.
Lastly, the second row shows when the sharpness of the individual model is reduced (lower in $x$-axis), the individual model's OOD accuracy is improved (from blue to red), demonstrating the benefits of minimizing sharpness. We verify the robustness of the phenomenon by measuring the sharpness and diversity using different metrics in Appendix~\ref{abl:abl-loss}.

Figure~\ref{fig:vary-overpara} illustrates the trade-off curves as the overparameterization level of the model is adjusted by changing width or sparsity (introduced using model pruning).
This visualization confirms that the trade-off is a consistent phenomenon across models of different sizes, and the ensemble provides less improvement (blue color) at the lower left end of each trade-off curve. 
It also highlights that models with smaller or sparser configurations show a more significant trade-off effect, as evidenced by the steeper slopes and higher coefficient values of the linear fitting curves.
As sparse ensembles are now being used to demonstrate the benefits of ensembling for efficient models~\citep{liu2021deep, diffenderfer2021winning, whitaker2022prune, kobayashi2022diverse, zimmer2023sparse}, addressing the conflict between sharpness and diversity becomes particularly crucial.

\begin{figure*}[!th]
\centering
\begin{subfigure}{0.45\linewidth}
\includegraphics[width=\linewidth,keepaspectratio]{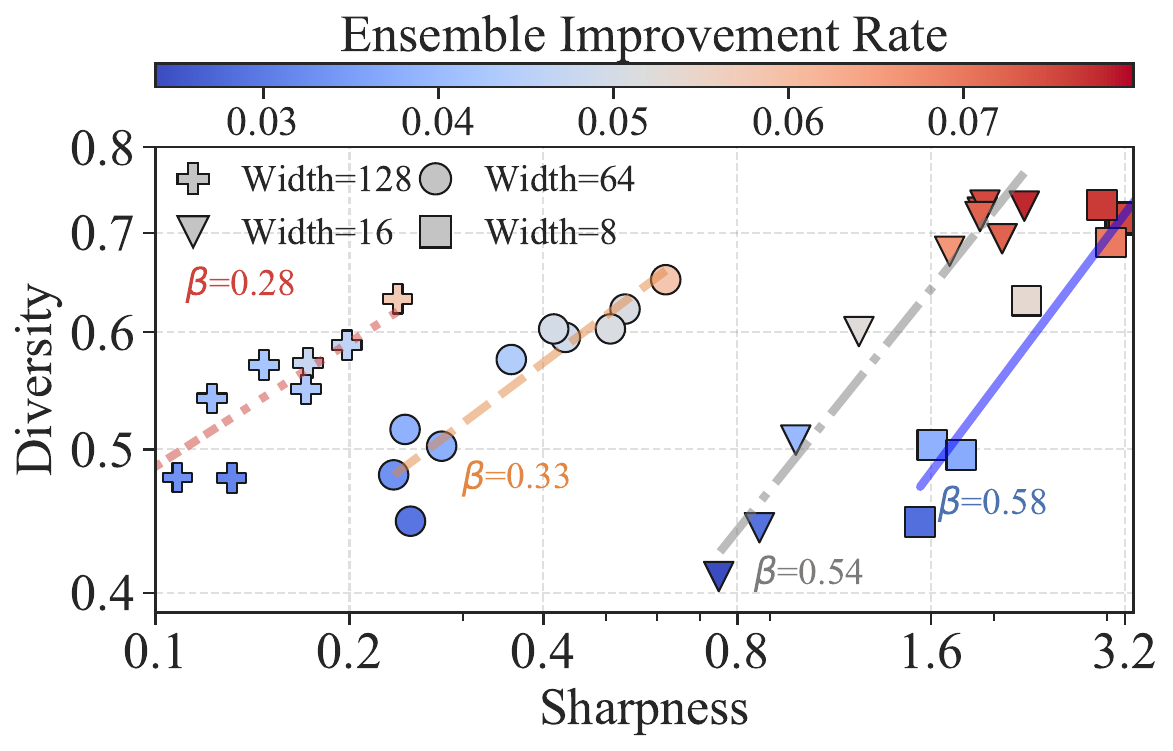}
\caption{Varying model width}  \label{fig:vary-width}
\end{subfigure}
\centering
\begin{subfigure}{0.45\linewidth}
\includegraphics[width=\linewidth,keepaspectratio]{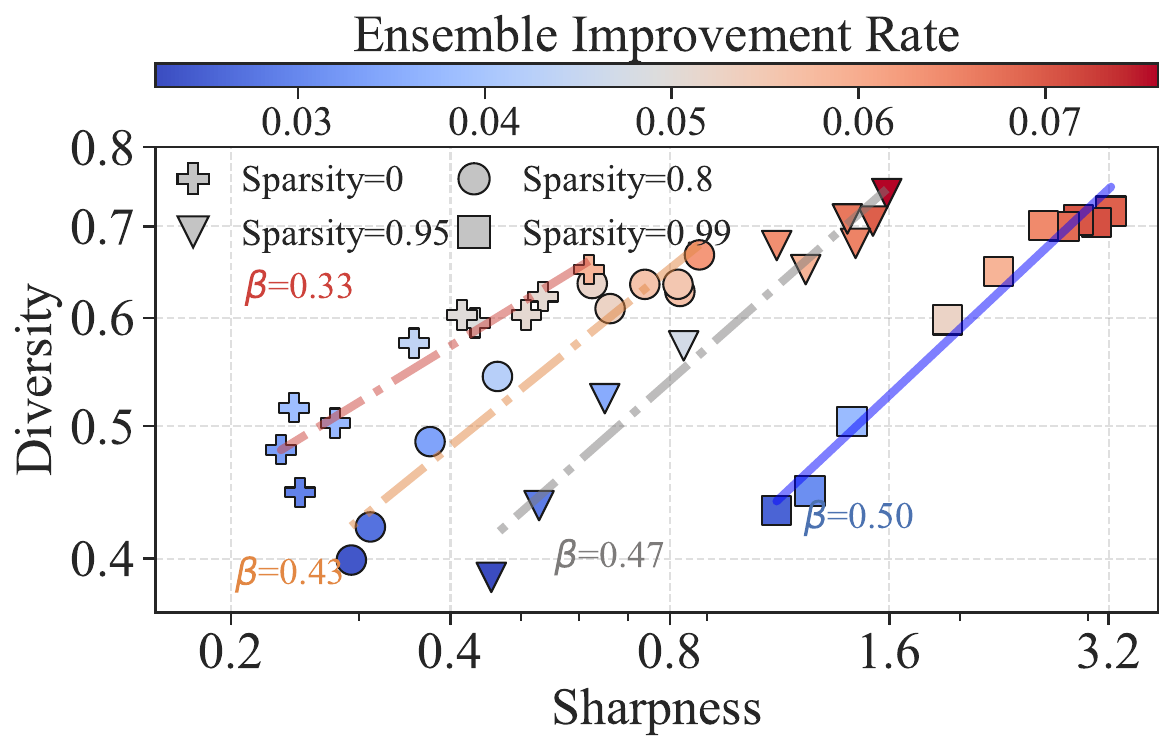} 
\caption{Varying model sparsity} \label{fig:vary-sparsity}
\end{subfigure} 
\caption{
\textbf{
(Sharpness-diversity trade-off in models varying overparameterization levels).} 
Different types of markers represent models with varying degrees of overparameterization, determined by changing the model width (a) or sparsity (b). 
Each marker represents a three-member ensemble trained with \sam with a different perturbation radius.
The $\beta$ reflects the rate of decline in the trade-off curve, calculated via applying linear fitting over the ensembles at each level of overparameterization.
A higher $\beta$ points to a steeper decline in the trade-off. Ensembles with narrower widths or increased sparsity display more pronounced trade-off effects. The model used in ResNet18 and the dataset is CIFAR-10.\looseness-1}
\label{fig:vary-overpara}  
\end{figure*}

\subsection{Our \ourmethod method}~\label{sec:sharpbalance-method} \vspace{-6mm}

Here, we describe the design and implementation of our main method, \ourmethod. Figure~\ref{fig:sharpb-sys} provides an overview. Our approach is motivated by the theoretical analysis in Section~\ref{sec:theory1}, which suggests that having each ensemble member minimize sharpness on diverse subsets of the data can lead to a better trade-off between sharpness and diversity.
\ourmethod aims to achieve the optimal balance by applying \sam to a carefully selected subset of the data, while performing standard optimization on the remaining samples.
More specifically, for each ensemble member NN $f_{\rvtheta_i}$, our method divides the entire training dataset $\mathcal{D}$ into two distinct subsets: sharpness-aware set $\mathcal{D}^{i}_\text{SAM}$ and normal set $\mathcal{D}^{i}_\text{Normal}$. 
The model is trained to optimize the sharpness reduction objective on $\mathcal{D}^{i}_\text{SAM}$, while it optimizes the normal training objective on $\mathcal{D}^{i}_\text{Normal}$.
These training objectives are denoted as $\mathcal{L}_{\mathcal{D}^{i}_\text{SAM}}^\text{SAM}(\rvtheta_i)$ and  $\mathcal{L}_{\mathcal{D}^{i}_\text{Normal}}(\rvtheta_i)$, respectively.
The $\mathcal{D}^{i}_\text{SAM}$ is selected by an adaptive strategy from the whole dataset $\mathcal{D}$: it is composed of the union of samples that are deemed ``sharp'' by all other members of the ensemble except the $i$-th.
Specifically, for each model, we pick the subset of data samples with the top-$k$\% highest ``per-data-sample sharpness.'' Then, we take the union of all such subsets expect the $i$-th for creating the subset $\mathcal{D}^{i}_\text{SAM}$.

\begin{figure}[!ht]
\centering
\includegraphics[width=0.8\linewidth,keepaspectratio]{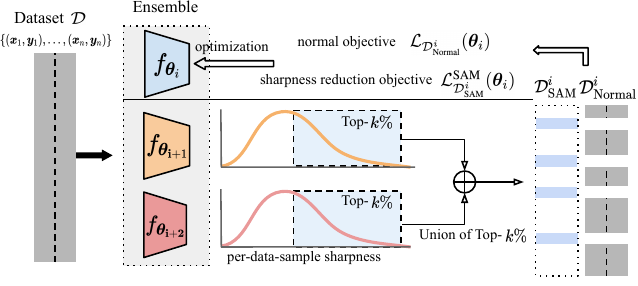}
\caption{\textbf{(System diagram of \ourmethod).} Each ensemble member $f_{\rvtheta_i}$ optimizes the sharpness reduction objective on subset $\mathcal{D}^{i}_\text{SAM}$ and the normal training objective on $\mathcal{D}^{i}_\text{Normal}$. $\mathcal{D}^{i}_\text{SAM}$ is formed by selecting data samples from $\mathcal{D}$ that significantly affect the loss landscape sharpness of other ensemble members. \looseness-1
} 
\label{fig:sharpb-sys}
\end{figure}

\noindent
{\bf Per-data-sample sharpness.} This metric is designed to efficiently assess the sharpness of a model for individual data samples. 
For each data point $(\vx_j, \vy_j)$, sharpness is quantified using the Fisher Information Matrix (FIM), which is expressed as $\nabla_\rvtheta \ell( f_{\rvtheta}(\vx_j), \vy_j ) \nabla_{\rvtheta} \ell( f_{\rvtheta}(\vx_j), \vy_j )^T$. 
Following a well-established approach~\citep{bottou2018optimization}, we approximate the trace of the FIM by computing the squared $\ell_2$ norm of the gradient: $\| \nabla_{\rvtheta} \ell( f_{\rvtheta}(\vx_j), \vy_j ) \|^2_2$.
Other common sharpness metrics, such as worst-case sharpness, trace of the Hessian, or Hessian eigenvalues, are computationally slightly more expensive to approximate~\citep{yao2020pyhessian,yao2021adahessian}, but are expected to lead to similar results.
\subsection{Empirical evaluation of \ourmethod} \vspace{-0mm}
\label{sec:eval-perf}
\begin{figure}[!h]
    \centering
    \begin{subfigure}{0.8\textwidth}
    \includegraphics[width=\textwidth]{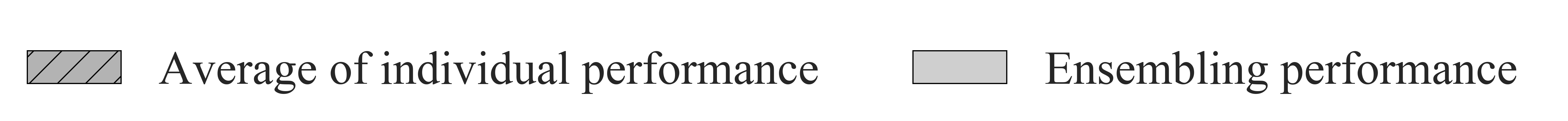}
    \end{subfigure} \\
    \centering
    \begin{subfigure}{0.30\textwidth}
    \includegraphics[width=\textwidth]{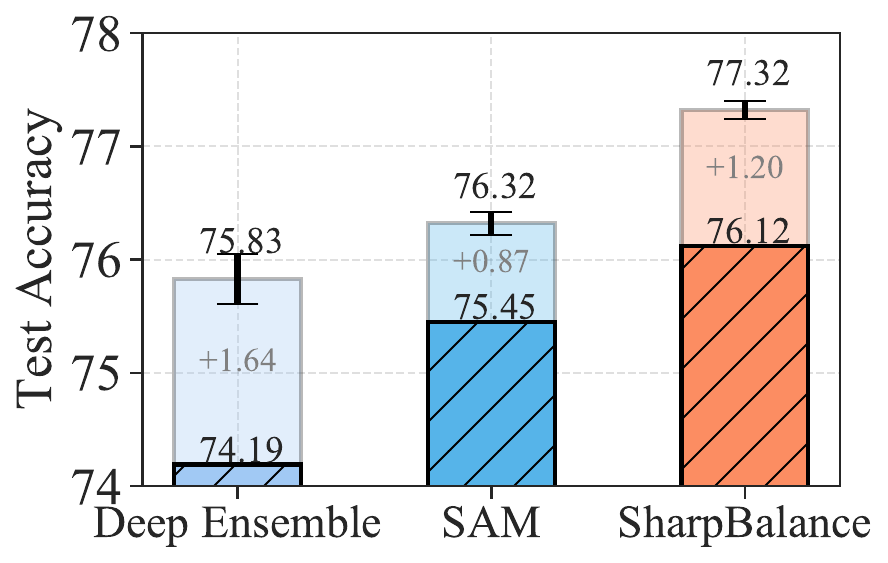}
    \caption{CIFAR10-C}
    \end{subfigure}
    \begin{subfigure}{0.30\textwidth}
    \includegraphics[width=\textwidth]{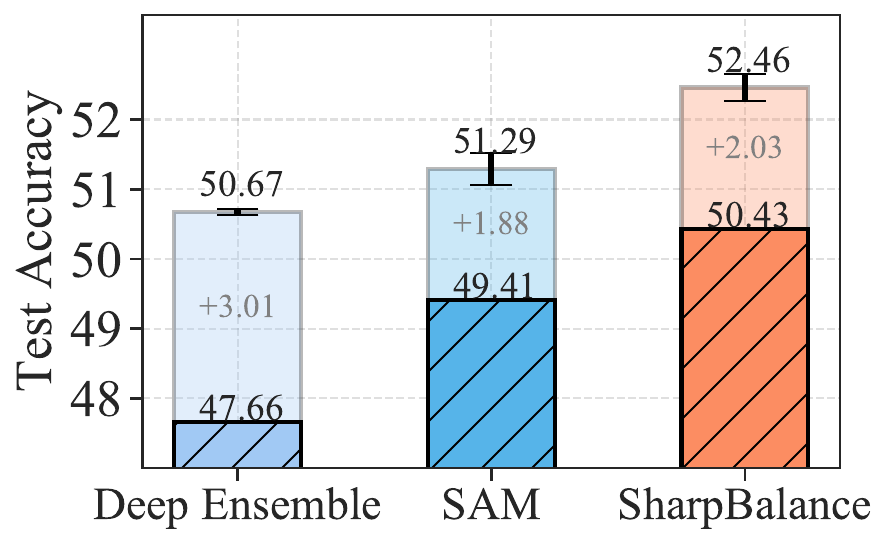}
    \caption{CIFAR100-C}
    \end{subfigure}
    \begin{subfigure}{0.30\textwidth}
    \includegraphics[width=\textwidth]{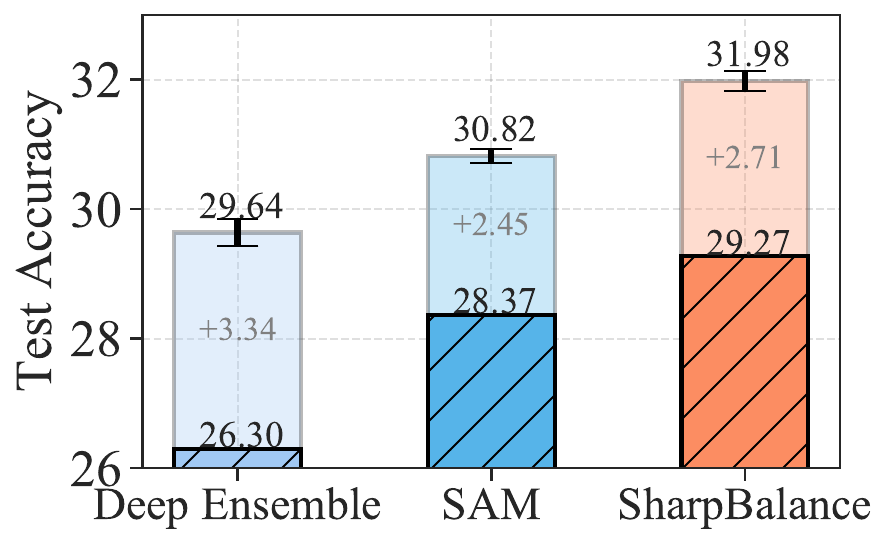}
    \caption{TinyImageNet-C}
    \end{subfigure} \\
    \begin{subfigure}{0.30\textwidth}
    \includegraphics[width=\textwidth]{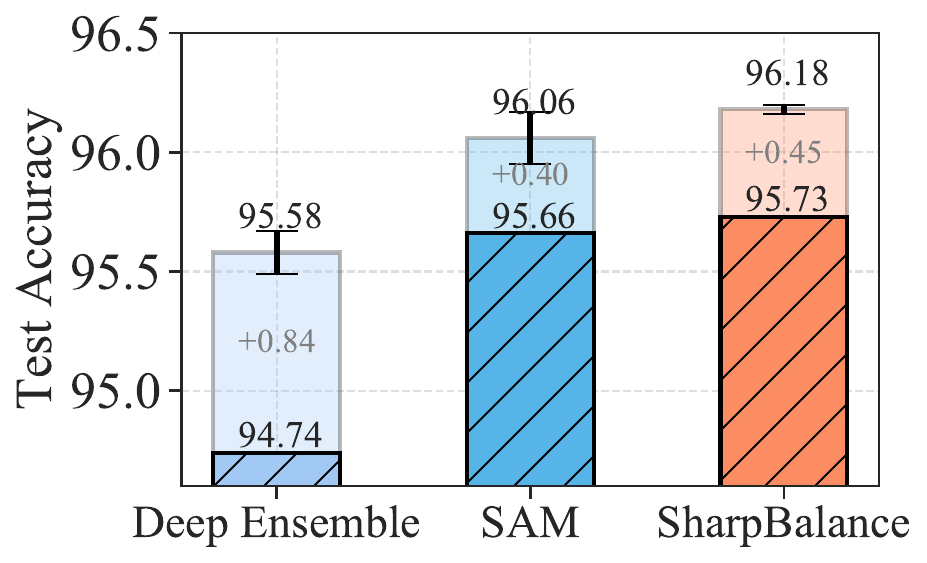}
    \caption{CIFAR10}
    \end{subfigure}
    \begin{subfigure}{0.30\textwidth}
    \includegraphics[width=\textwidth]{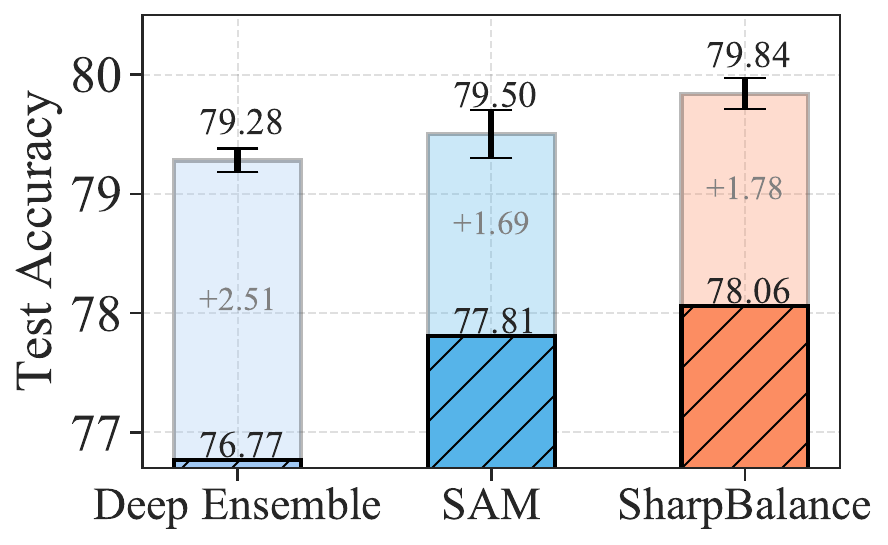}
    \caption{CIFAR100}
    \end{subfigure}
    \begin{subfigure}{0.30\textwidth}
    \includegraphics[width=\textwidth]{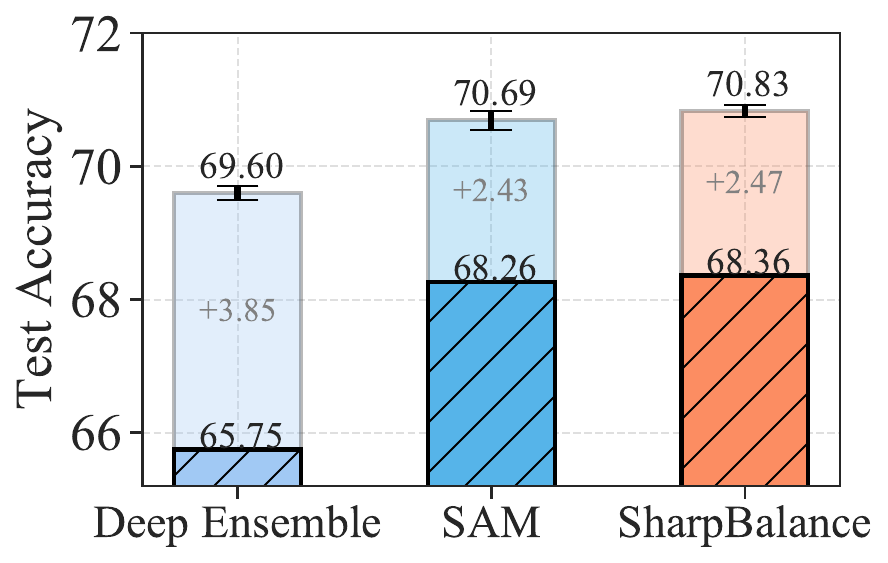}
    \caption{TinyImageNet}
    \end{subfigure}
    \caption{\textbf{(Main results: \ourmethod improves the overall ensembling performance and mitigates the reduced ensembling improvement caused by sharpness-diversity trade-off).} The three-member ResNet18 ensemble is trained with different methods on three datasets. The first row reports the OOD accuracy and the second row reports the ID accuracy. 
    The lower part of each bar with the diagonal lines represents the individual model performance. 
    The upper part of each bar represents the ensembling improvement.
    The results are reported by averaging three ensembles, and each ensemble is comprised of three models.
 \looseness-1} \vspace{-0mm}
    \label{fig:compare_results} 
\end{figure}

We evaluate \ourmethod by benchmarking it against both a standard Deep Ensemble, trained using SGD, and a Deep Ensemble enhanced with \sam. 
The results are presented in Figure~\ref{fig:compare_results} for CIFAR-10, CIFAR-100, and TinyImageNet. 
The comparison between the middle and left bars shows that \sam enhances individual model performance by reducing sharpness. 
However, this reduction in sharpness also diminishes the overall ensemble effectiveness by lowering diversity, exemplifying the sharpness-diversity trade-off discussed in Section~\ref{sec:emp-trade-off}.
Further comparison between the right and middle bars shows that \ourmethod maintains or improves individual performance while improving ensemble effectiveness. 
These results confirm that \ourmethod consistently boosts both ID and OOD performance across the datasets studied.

To further evaluate \ourmethod, we provide corroborating results in Appendix~\ref{sec:more-results}, which includes:
\begin{itemize}[leftmargin=*] 
    \item We evaluate \ourmethod on different severity of the corruption on CIFAR10-C, CIFAR100-C and Tiny-ImageNet-C. \ourmethod increasingly outperforms the baselines as the severity of the corruption increases.
    \item We further evaluate \ourmethod on another model WRN-40-2.
    \item We compare our method of measuring sharpness with another method of measuring the curvature of the loss around a data point~\citep{garg2023samples} and show the strong correlation between these two methods.
    \item  We further compare \ourmethod with EoA~\citep{arpit2022ensemble} and an improved version of \sam (for which individual models in an ensemble are trained with different $\rho$ values). Results show that \ourmethod can significantly outperform the baselines.
    \item We demonstrate that, compared to training a deep ensemble with \sam, our method adds only minimal computational cost. The extra time complexity is dominated by the computation of Fisher trace for evaluating per-sample sharpness, which empirically increases the training time by 1\%.
\end{itemize}
\section{Conclusion}\label{sec:conclusions}

Our theoretical and empirical analyses demonstrate the existence of a sharpness-diversity trade-off when sharpness-minimization training methods are applied to deep ensembles.
This leads to two main insights that are relevant for improving model performance.
First, reducing the sharpness in individual models proves to be beneficial in enhancing the performance of the ensemble as a whole.
Second, the accompanying reduction in diversity suggests that popular ensembling methods have limitations, and also highlights the potential for more sophisticated designs that promote diversity among models with lower sharpness.
These results are particularly timely, given recent theoretical work on characterizing ensemble improvement~\citep{theisen2023ensembles}.
In response to these findings, we have proposed \ourmethod, which ``diagnoses'' the training data by evaluating the sharpness of each sample and then fine-tunes the training of individual models to focus on a diverse subset of the sharpest training data samples. 
This targeted approach helps maintain diversity among models while also reducing their individual sharpness.
Extensive evaluations indicate that \ourmethod not only improves the sharpness-diversity trade-off but also delivers superior OOD performance for both dense and sparse models across various datasets and architectures when compared to other ensembling approaches.

{\bf Limitations.} One limitation of the study is that our theoretical analysis in Section \ref{sec:theory1} relies on the assumption that the data matrices $\rmA, \rmT$ follow a Gaussian distribution and assumed the optimization objective to be quadratic, which may not always hold in practice. Despite the potentially strong assumptions, our empirical findings in Section \ref{sec:flat-div} show that the conclusions remain robust in real-world datasets with various model architectures. This suggests the insights discovered in our study are applicable to a wider range of real-world scenarios, beyond just those strictly adhering to the Gaussian assumption. Nevertheless, future research could explore how such assumptions can be relaxed and extend the theoretical analysis to a weaker condition.

\bibliographystyle{plainnat}

\bibliography{ref}

\newpage
\appendix
\section*{Appendix}
\section{Impact Statement}\label{sec:impact-statement}
This paper uncovers a trade-off between sharpness and diversity in deep ensembles and introduces a novel training strategy to achieve an optimal balance between these two crucial metrics. While the proposed method could potentially be misused for malicious purposes, we believe that the study itself does not pose any direct negative societal impact. More importantly, this research advances the field of ensemble learning and contribute to the development of more reliable deep ensemble models. These advancements consequently result in enhanced robustness when dealing with OOD data and enable the quantification of uncertainty, thereby strengthening the reliability and applicability of deep learning systems in real-world scenarios. 
\section{Related work}\label{sec:related-work}

\noindent \textbf{Ensembling}.
Diversity is one of the major factors that contribute to the success of the ensembling method. 
Popular ensemble techniques have been developed for tree-type individual learners, which are known to have a high variance. This is evident such as in~\citep{breiman2001random, chen2016xgboost, freund1995boosting, freund1997decision}. In contrast, more stable algorithms, such as support vector machines (SVM) type learners, are less commonly used for ensembles, unless they are tuned to a low-bias, high-variance regime, as explored in \citep{valentini2003low}.
When it comes to diversity and ensembling, NNs are known to exhibit properties different than traditional models, e.g., as described in recent theoretical and empirical work on loss landscapes and emsemble improvement~\citep{theisen2023ensembles,yang2021taxonomizing}.
Therefore, ensembling techniques that work well for traditional models (e.g., tree-type models) often underperform the simple yet efficient deep ensembles method~\citep{fort2019deep,ortega2022diversity} that uses the independent initialization and optimization.
Previous literature has explored various new methods to learn diverse NNs \citep{lee2022diversify,rame2022diverse,pang2019improving,parker2020ridge}.
Our work is different from previous work in that we study flat ensembles obtained from sharpness-aware training methods, especially focusing on diversifying flat ensembles by reducing the overlap between sharpness-aware data subsets.
While our work demonstrates significant improvements in OOD generalization, it is known that (in some cases, see also~\citet{theisen2023ensembles}) deep ensembling is a simple, yet effective method to improve OOD performance~\citep{diffenderfer2021winning}. Therefore, we compare the OOD performance of \ourmethod to deep ensembles.

\noindent
\textbf{Sharpness and generalization}. 
A large body of work has studied the relationship between the sharpness (or flatness) of minima and the generalizability of models~\citep{ 10.1162, 10.1145, keskar2016large, neyshabur2018a, yang2021taxonomizing,kaddour2022fair,yao2021adahessian,yao2020pyhessian}.
Works such as those by \citep{10.1162} and \citep{10.1145} use Bayesian learning and minimum description length to explain why we should train models to flat minima.
\citep{keskar2016large} introduces a sharpness-based metric, demonstrating how large-batch training can skew NNs towards sharp local minima, adversely affecting generalization. 
In addition, \citep{neyshabur2018a} uses a PAC-Bayesian framework to prove bounds on generalization, which can be interpreted as the relationship between sharpness and test accuracy. Furthermore, \citep{cha2021swad} presents a theoretical exploration of the link between the sharpness of minima and OOD generalization.

Motivated by the good generalization property of flat minima, variants of sharpness-guided optimization techniques have been proposed~\citep{yao2018hessian,yao2021adahessian,du2021efficient,jiang2023adaptive}, including sharpness-aware minimization ~\citep{foret2020sharpness}.
The DiWA method~\citep{rame2022diverse} observed that \sam can decrease the diversity of models in the context of weight averaging (WA)~\citep{izmailov2018averaging}. 
However, WA imposes constraints on different models, requiring them to share the same initialization and stay close to each other in the parameter space.  
In contrast, our work focuses on deep ensembles that do not pose additional constraints on the training trajectories of individual ensemble members. Previous work by \citep{behdin2023statistical} provided a theoretical characterization of important statistical properties for kernel regression models and single-layer ReLU networks, optimized using \sam on noisy datasets. Our theoretical analysis borrows ideas from~\citep{behdin2023statistical} and extends the analysis using random matrix theory.\looseness-1

\section{Proof of Theorems in Section \ref{sec:theory1}}
\label{apd:proof}

Recall that \sam updates the model weights, ignoring the normalization constant and regularization, through the following recursive rule 
\begin{equation*}
    \rvtheta_{k+1}^{SAM} = \rvtheta_{k}^{SAM} -\eta \nabla f\left(\rvtheta_{k}^{SAM} + \rho \nabla f(\rvtheta_{k}^{SAM})\right).
\end{equation*}
We first show an unrolling of the iterative optimization on a quadratic objective. \looseness-1
\begin{theorem}[Unrolling \sam]
\label{thm:unroll}
Let $\rvtheta^*$ be the teacher model. Let $\rvtheta_0$ be randomly initialized and updated with \sam to solve a quadratic objective $\mathcal{L}_\rmA(\rvtheta)=\frac{1}{2}(\rvtheta-\rvtheta^*)^T \rmA^T\rmA(\rvtheta-\rvtheta^*)$. Then, \looseness-1

    \begin{equation*}
        \rvtheta^{SAM}_{k+1} = \eta \sum_{i=0}^{k} \rmB^i \left(\rmA^T\rmA+ \rho (\rmA^T\rmA)^2\right)\rvtheta^* + \rmB^{k+1} \rvtheta_0,
    \end{equation*}
where $\rmB = \mathbf{I}- \eta \rmA^T\rmA - \eta \rho (\rmA^T\rmA)^2$.
\begin{proof}
    The gradient of the objective $f$ is given by $\nabla f(\rvtheta) = \rmA^T\rmA(\rvtheta-\rvtheta^*)$. Therefore,
    \begin{equation*}
        \rvtheta_k^{SAM} + \rho \nabla f(\rvtheta_k^{SAM}) = (\mathbf{I} + \rho \rmA^T\rmA)\rvtheta_k^{SAM} - \rho \rmA^T\rmA \rvtheta^* .
    \end{equation*}
    With SAM update,
    \begin{align*}
        \rvtheta^{SAM}_{k+1} &= \rvtheta_{k}^{SAM} -\eta \nabla f\left(\rvtheta_k^{SAM} + \rho \nabla f(\rvtheta_k^{SAM})\right)\\
        &=\rvtheta_{k}^{SAM} -\eta\rmA^T\rmA\left(\rvtheta_k^{SAM} + \rho \nabla f(\rvtheta_k^{SAM}) - \rvtheta^*\right)\\
        &= \rvtheta_{k}^{SAM} -\eta\rmA^T\rmA\left( (\mathbf{I} + \rho \rmA^T\rmA)\rvtheta_k^{SAM} - \rho \rmA^T\rmA \rvtheta^*- \rvtheta^*\right)\\
        &= \left(\mathbf{I}-\eta \rmA^T\rmA - \eta \rho (\rmA^T\rmA)^2\right)\rvtheta_{k}^{SAM} + \eta\left(\rmA^T\rmA+\rho(\rmA^T\rmA)^2\right)\rvtheta^*\\
        &= \eta \sum\limits_{i=0}^k\rmB^i\left(\rmA^T\rmA+\rho(\rmA^T\rmA)^2\right)\rvtheta^*+ \rmB^{k+1} \rvtheta_0,
    \end{align*}
where the last equation is obtained by recursively unrolling the weight by previous updates.
\end{proof}
\end{theorem}
Theorem \ref{thm:unroll} offers a valuable tool to analyze the statistical behavior of the models optimized by \sam. However, one more ingredient is required to arrive at the interesting conclusions claimed in Section \ref{sec:theory1}, the random matrix theory.
Recall that the data matrix $\rmA \in \mathbb{R}^{n_{\text{tr}} \times d_{\text{in}}}$ is random with entries drawn from Gaussian $\mathcal{N}(0,\mathbf{I}/d_{\text{in}})$. As a result, entries in $\rmA^T\rmA$ follows the Wishart distribution and according to Corollary 3.3 in \citet{bishop2018introduction}, for $k \geq 1$,

\begin{equation}
    \label{eq: wishartM}
    \mathbb{E}[(\rmA^T\rmA)^k]= \left(\frac{n_{\text{tr}}}{d_\text{in}}\right)^k\sum\limits_{i=1}^{k}\left(\frac{d_\text{in}}{n_{\text{tr}}}\right)^{k-i}\mathcal{O}\left(1+1/d_\text{in}\right)N_{k,i}\mathbf{I},
\end{equation}
where $N_{k,i} = \frac{1}{i}\binom{k-1}{i-1}\binom{k}{i-1}$ is the Narayana number. With the help of this Corollary, we now prove a proposition on the expectation of $\rmB^k$.
\begin{prop}[Expectation of Wishart Moments]
\label{prop:wishart}
Let $i,j$ be non-negative integers, then
\begin{equation*}
    \mathbb{E}_{\rmA}[{\rmB^i(\rmA^T\rmA)^j}] = \phi(i,j)\mathbf{I},
\end{equation*}
where 
    \begin{align*}
        \phi(i,j) :=& \mathds{1}_{j=0}+\sum\limits_{k_1+k_2+k_3=i}\frac{i!}{k_1!k_2!k_3!}(-\eta)^{k_2+k_3}\rho^{k_3}\left(\frac{n_{\text{tr}}}{d_{\text{in}}}\right)^{m}\sum\limits_{l=1}^{m}\left(\frac{d_{\text{in}}}{n_{\text{tr}}}\right)^{m-l}\mathcal{O}(1+1/d_\text{in})N_{m,l},
    \end{align*}
and $m = k_2+2k_3+j$.
\end{prop}
\begin{proof}
    By Multinomial Theorem,
    \begin{align*}
        \rmB^i(\rmA^T\rmA)^j &=  \left(\sum\limits_{k_1+k_2+k_3=i}\frac{i!}{k_1!k_2!k_3!} \mathbf{I}^{k_1}(-\eta \rmA^T\rmA)^{k_2}(-\eta\rho (\rmA^T\rmA)^2)^{k_3}\right)(\rmA^T\rmA)^j\\
        &=\sum\limits_{k_1+k_2+k_3=i}\frac{i!}{k_1!k_2!k_3!}(-\eta)^{k_2+k_3}\rho^{k_3} (\rmA^T\rmA)^{k_2+2k_3+j}.
    \end{align*}
    Let $m = k_2+2k_3+j$ and taking the expectation with \eqref{eq: wishartM} gives

    \begin{align*}
      \mathbb{E}_{\rmA}[{\rmB^i(\rmA^T\rmA)^j}]&= \sum\limits_{k_1+k_2+k_3=i}\frac{i!}{k_1!k_2!k_3!}(-\eta)^{k_2+k_3}\rho^{k_3} \E_\rmA[(\rmA^T\rmA)^{k_2+2k_3+j}]\\
        &=\sum\limits_{k_1+k_2+k_3=i}\frac{i!}{k_1!k_2!k_3!}(-\eta)^{k_2+k_3}\rho^{k_3} \left(\frac{n_{\text{tr}}}{d_{\text{in}}}\right)^{m}\sum\limits_{l=1}^{m}\left(\frac{d_{\text{in}}}{n_{\text{tr}}}\right)^{m-l}\mathcal{O}(1+1/d_\text{in})N_{m,l}\mathbf{I}.
    \end{align*}
    If $j=0$, then there is a case when $k_2=k_3=0$, and the expectation of $(\rmA^T\rmA)^0$ simply becomes $\mathbf{I}$. Therefore, 
    \begin{align*}
        \mathbb{E}_{\rmA}[{\rmB^i(\rmA^T\rmA)^j}]
        &=\mathds{1}_{j=0}\mathbf{I}+\sum\limits_{k_1+k_2+k_3=i}\frac{i!}{k_1!k_2!k_3!}(-\eta)^{k_2+k_3}\rho^{k_3} \left(\frac{n_{\text{tr}}}{d_{\text{in}}}\right)^{m}\sum\limits_{l=1}^{m}\left(\frac{d_{\text{in}}}{n_{\text{tr}}}\right)^{m-l} \mathcal{O}(1+1/d_\text{in})N_{m,l}\mathbf{I}\\
        &=\phi(i,j)\mathbf{I}.
    \end{align*}
\end{proof}
\subsection{Proof of Theorem \ref{thm:varSAM}}
\label{apd:thm1}
In this subsection, we show a proof for Theorem \ref{thm:varSAM}.

\begin{proof}
    Apply Singular Value Decomposition (SVD) to obtain $\rmA =\rmV\Sigma\rmU^T$ and $\rmA^T\rmA = \rmU\Sigma^2\rmU^T$. Let $\rmD = \Sigma^2$. By Theorem \ref{thm:unroll},
    \begin{align*}
        \rvtheta_k^{SAM} &= \eta \sum_{i=0}^{k-1} \rmB^i (\rmA^T\rmA+ \rho (\rmA^T\rmA)^2)\rvtheta^* + \rmB^{k} \rvtheta_0\\
        &=\eta\sum_{i=0}^{k} \left(\mathbf{I}-\eta\rmA^T\rmA-\eta\rho(\rmA^T\rmA)^2\right)^i \left(\rmA^T\rmA+ \rho (\rmA^T\rmA)^2\right)\rvtheta^*+ \rmB^{k} \rvtheta_0\\
        &=\eta\sum_{i=0}^{k-1} \rmU(\mathbf{I}-\eta\rmD-\eta\rho\rmD^2)^i\rmU^T \rmU(\rmD+ \rho \rmD^2)\rmU^T\rvtheta^*+ \rmB^{k} \rvtheta_0\\
        &=\eta\rmU \cdot \textbf{diag} \left(\left\{{\sum_{i=0}^{k-1} (1-\eta d_j-\eta\rho d_j^2)^i} (d_j+\rho d_j^2)\right\}_{j=1}^{d_\text{in}}\right)\rmU^T\rvtheta^*+ \rmB^{k} \rvtheta_0\\
        &=\eta\rmU \cdot \textbf{diag} \left(\left\{\frac{1-(1-\eta d_j-\eta\rho d_j^2)^k}{\eta d_j + \eta \rho d_j^2} (d_j+\rho d_j^2)\right\}_{j=1}^{d_\text{in}}\right)\rmU^T\rvtheta^*+ \rmB^{k} \rvtheta_0\\
        &=\rmU \left(\mathbf{I}-(\mathbf{I}-\eta \rmD -\eta \rho \rmD^2)^k\right)\rmU^T\rvtheta^*+ \left(\mathbf{I}-\eta \rmA^T\rmA -\eta \rho (\rmA^T\rmA)^2\right)^k \rvtheta_0\\
        &=\rvtheta^*+\left(\mathbf{I}-\eta \rmA^T\rmA -\eta \rho (\rmA^T\rmA)^2\right)^k(\rvtheta_0-\rvtheta^*).
    \end{align*}
    As a result, $\E_{\rvtheta_0}[\rvtheta_k^{SAM}] = \rvtheta^*-\left(\mathbf{I}-\eta \rmA^T\rmA -\eta \rho (\rmA^T\rmA)^2\right)^k\rvtheta^* =\rvtheta^*-\rmB^k\rvtheta^*$. By definition,
    \begin{align*}
        n_{\text{te}}\text{Bias}^2(\rvtheta_k^{SAM}) 
        &=\E_{\rmA,\rmT}[\sum\limits_{i=1}^p
        (\E_{\rvtheta_0}[f(\rvtheta_k^{SAM};\rmT_i)]-y^{(\rmT)}_i)^2]\\
        &= \E_{\rmA,\rmT}[(\E_{\rvtheta_0}[\rvtheta_k^{SAM}] - \rvtheta^*)^T \rmT^T\rmT (\E_{\rvtheta_0}[\rvtheta_k^{SAM}] - \rvtheta^*)]\\
        &= \E_{\rmA}[(\E_{\rvtheta_0}[\rvtheta_k^{SAM}] - \rvtheta^*)^T \E_\rmT[\rmT^T\rmT] (\E_{\rvtheta_0}[\rvtheta_k^{SAM}] - \rvtheta^*)]\\
        &= \frac{n_{\text{te}}}{d_\text{in}} \E_{\rmA}[(\rvtheta^*)^T \rmB^{2k}\rvtheta^*]\\
        &=\frac{n_{\text{te}}}{d_\text{in}}\phi(2k,0) \|\rvtheta^*\|_2^2,\\
        n_{\text{te}}\text{Error}(\rvtheta^{SAM}_k) &= \E_{\rmA,\rmT,\theta_0}[\sum\limits_{i=1}^p(y^{(\rmT)}_i - f(\rvtheta_k^{SAM};\rmT_i))^2]\\
        &= \E_{\rmA,\rmT,\theta_0}[(\rvtheta^* - \rvtheta^{SAM}_k)^T\rmT^T\rmT (\rvtheta^*-\rvtheta^{SAM}_k)]\\
        &= \E_{\rmA,\rmT,\theta_0}[(\rvtheta^* - \rvtheta_0)^T\rmB^k\rmT^T\rmT \rmB^k(\rvtheta^*-\rvtheta_0)]\\
        &=\E_{\rmA,\rmT,\theta_0}[(\rvtheta^*) ^T\rmB^k\rmT^T\rmT\rmB^k \rvtheta^*] + \E_{\rmA,\rmT,\theta_0}[\rvtheta_0 ^T\rmB^k\rmT^T\rmT\rmB^k \rvtheta_0]\\
        &=\frac{n_{\text{te}}}{d_{\text{in}}}\phi(2k,0)\|\rvtheta^*\|^2_2 + n_{\text{te}}\phi(2k,0)\sigma^2,
    \end{align*}
    Since $\E_\rmT[\rmT^T\rmT] = \frac{n_{\text{te}}}{d_\text{in}}\mathbf{I}$ and $\mathbb{E}[\rvtheta_0 \rvtheta_0^T] = \sigma^2\mathbf{I}$.
    Hence, 
    \begin{equation*}
    \label{dif_result}
        \mathbb{D}(\rvtheta_k^{SAM}) = \text{Var}\left(f(\rvtheta_k^{SAM};\rmT)\right) =\frac{1}{n_{\text{te}}}\left(n_{\text{te}}\text{Error}(\rvtheta_k^{SAM}) -  n_{\text{te}}\text{Bias}^2(\rvtheta_k^{SAM})\right)= \phi(2k,0)\sigma^2.
    \end{equation*}

    Recall that given a perturbation radius $\rho_0$, the sharpness is defined as 
    \begin{equation*}
        \kappa(\rvtheta_k) = \mathbb{E}_{A}[\max_{\|\rvvarepsilon\|_2 \leq \rho_0} f \left( \mathbb{E}_{\rvtheta _0} \left[ \rvtheta_k \right] + \rvvarepsilon \right) - f \left( \mathbb{E}_{\rvtheta_0} \left[ \rvtheta_k \right] \right)].
    \end{equation*}
    We first compute
    \begin{align}
        f \left( \E_{\rvtheta_0} \left[ \rvtheta_k^{SAM} \right] + \rvvarepsilon ;\rmA\right) &= \frac{1}{2}(\mathbb{E}_{\rvtheta_0} \left[ \rvtheta_k^{SAM} \right] + \rvvarepsilon-\rvtheta^*)^T{A}^T {A} (\mathbb{E}_{\rvtheta_0} \left[ \rvtheta_k^{SAM} \right] + \rvvarepsilon-\rvtheta^*)\nonumber\\
        &= \frac{1}{2}( \rvvarepsilon -\rmB^k\rvtheta^*)^T\rmA^T\rmA (\rvvarepsilon -\rmB^k\rvtheta^*)\nonumber\\
        &=\frac{1}{2}\rvvarepsilon^T \rmA^T\rmA \rvvarepsilon - \rvvarepsilon^T \rmB^k\rmA^T\rmA\rvtheta^*+\frac{1}{2}(\rvtheta^*)^T\rmB^{2k}\rmA^T\rmA\rvtheta^*.\label{eq:f(ew+e)}
    \end{align}
    Similarly, 
    \begin{equation}
        \label{eq:f(ew)}
        f \left( \E_{\rvtheta_0} \left[ \rvtheta_k^{SAM} \right] ;\rmA\right) = \frac{1}{2} (\rvtheta^*)^T \rmB^{2k}\rmA^T\rmA\rvtheta^*.
    \end{equation}
    Let $\lambda_{min}$ be the least eigenvalue of $\rmA^T\rmA$. By subtracting \eqref{eq:f(ew+e)} with \eqref{eq:f(ew)}, we have
    \begin{align*}
        \kappa_k^{SAM} &= \mathbb{E}_{\rmA}[\max_{\|\rvvarepsilon\|_2 \leq \rho_0} \frac{1}{2}\rvvarepsilon^T\rmA^T\rmA\rvvarepsilon - \rvvarepsilon^T \rmB^k \rmA^T\rmA\rvtheta^*]\\
        &\geq \mathbb{E}_{\rmA}[\max_{\|\rvvarepsilon\|_2 = \rho_0} \frac{1}{2}\lambda_{min}\|\rmU^T\rvvarepsilon\|_2^2 - \rvvarepsilon^T \rmB^k \rmA^T\rmA\rvtheta^*]\\
        &\geq \mathbb{E}_{\rmA}[\max_{\substack{\|\rvvarepsilon\|_2 = \rho_0 \\ {\rvvarepsilon} = {\rmU}{\rvv}}} \frac{1}{2}\lambda_{min}\|\rmU^T\rvvarepsilon\|_2^2 - \rvvarepsilon^T \rmB^k \rmA^T\rmA\rvtheta^*]\\
        &= \mathbb{E}_{\rmA}[\max_{\|\rvv\|_2=\rho_0} \frac{1}{2}\lambda_{min}\|\rvv\|_2^2 - \min_{\|\rvvarepsilon\|_2=\rho_0}\rvvarepsilon^T \rmB^k \rmA^T\rmA\rvtheta^*]\\
        &= \mathbb{E}_{\rmA}[\frac{1}{2}\lambda_{min}\rho_0^2 + \rho_0\|\rmB^k \rmA^T\rmA\rvtheta^*\|_2].
    \end{align*}

The smallest singular value $\lambda_{min}$ of a random $n \times d_{\text{in}}$ matrix $\rmA$ can be bounded by the following inequality on the smallest singular value $\sigma_{min}(A)$ by \citet{Vershynin_2018}, assuming $n_{\text{tr}} \geq d_{\text{in}}$, then almost surely

\begin{equation*}
    \mathbb{E}_\rmA[\sigma_{min}(\rmA)] \geq \sqrt{\frac{n_{\text{tr}}}{d_{\text{in}}}} -1.
\end{equation*}
Therefore, $ \mathbb{E}_A[\lambda_{min}] \geq  \mathbb{E}_A[\sigma_{min}(A)]^2 \geq \left(\sqrt{\frac{n_{\text{tr}}}{d_{\text{in}}}} -1\right)^2$. Now we show a lower bound on $\E_\rmA[\rho_0\|\rmB^k\rmA^T\rmA\rvtheta^*\|_2]$. By \citet{gaox2019bounds}, the Jensen gap $(\E[Z])^{1/2} - \E[(Z)^{1/2}]$ is upper bounded by $\frac{\text{Var}(Z)}{2}$ when $Z$ is non-negative and $\E[Z] = 1$. Notice that 
\begin{align*}
    \E_\rmA[ \rho_0 \|\rmB^k\rmA^T\rmA\rvtheta^*\|_2]
    &= \rho_0 \E_\rmA[\left((\rvtheta^*)^T\rmB^{2k}(\rmA^T\rmA)^2\rvtheta^*\right)^{1/2}],
\end{align*}
and we let $Z = (\rvtheta^*)^T\rmB^{2k}(\rmA^T\rmA)^2\rvtheta^*$. Then $\E_\rmA[Z] =\phi(2k,2)\|\rvtheta^*\|^2_2$ and 
\begin{equation*}
    \text{Var}[Z] = \left(\phi(4k,4)-\phi(2k,2)^2\right)\|\rvtheta^*\|^2_2.
\end{equation*}
By normalizing $Z$ and applying the Jensen gap upperbound, we have
\begin{equation*}
    \E_\rmA [\rho_0 \|\rmB^k\rmA^T\rmA\rvtheta^*\|_2]
    \geq \rho_0\sqrt{\phi(2k,2)} \|\rvtheta^*\|^2_2 - \frac{\phi(4k,4)-\phi(2k,2)^2}{2\phi(2k,2)^{3/2}\|\rvtheta^*\|_2}.
\end{equation*}
As a result,
\begin{equation*}
    \kappa_k^{SAM} \geq \frac{\rho_0^2}{2}\left(\sqrt{\frac{n_{\text{tr}}}{d_{\text{in}}}} - 1\right)^2+\rho_0\sqrt{\phi(2k, 2)}\|\rvtheta^*\|_2 - \frac{\phi(4k,4)-\phi(2k,2)^2}{2\phi(2k,2)^{3/2}\|\rvtheta^*\|_2}.
\end{equation*}
The derivation of the upper bound follows from a similar proof, ignoring the Jensen gap.

\end{proof}

\subsection{Proof of Theorem \ref{thm:sub}}
\label{apd:proof2}
Below we show a proof of Theorem \ref{thm:sub}.

\begin{proof}
    We apply SVD to $\rmA_s$ to obtain $\rmA_s = \rmV_s\Sigma_s\rmU_s^T$ and $\rmA_s^T\rmA = \rmU_s\Sigma_s^2\rmU_s^T$. Let $\rmD_s = \Sigma_s^2$ and $\rmB_s = \mathbf{I}-\eta\rmA_s^T\rmA_s-\eta\rho(\rmA_s^T\rmA_s)^2$. By Theorem \ref{thm:unroll} and a similar derivation in the proof of Theorem \ref{thm:varSAM},
    \begin{align*}
        \rvtheta_k^{SharpBal} &= \eta \sum_{j=0}^{k-1} \rmB_s^j \left(\rmA_s^T\rmA_s+ \rho (\rmA_s^T\rmA_s)^2\right)\rvtheta^* + \rmB_s^{k} \rvtheta_0\\
        &=\rvtheta^*+\left(\mathbf{I}-\eta \rmA_s^T\rmA_s -\eta \rho (\rmA_s^T\rmA_s)^2\right)^k(\rvtheta_0-\rvtheta^*).
    \end{align*}
    As a result, $\E_{\rvtheta_0,s}[\rvtheta^{Sharpbal}_k] = \E_{s}[\rvtheta^*-\rmB_s^k\rvtheta^*] = \rvtheta^*-\frac{1}{S}\sum\limits_{s=1}^S\rmB_s^k\rvtheta^*$.
    
    Applying Proposition \ref{prop:wishart}, we have 
    \begin{equation*}
        \E_\rmA[\rmB_s^i(\rmA_s^T\rmA_s)^j] = \phi'(i,j),
    \end{equation*}
    where 
    \begin{equation*}
        \phi'(i,j) = \mathds{1}_{j=0}+\sum\limits_{k_1+k_2+k_3=i}\frac{i!}{k_1!k_2!k_3!}(-\eta)^{k_2+k_3}\rho^{k_3}\left(\frac{n_{\text{tr}}}{Sd_{\text{in}}}\right)^{m}\sum\limits_{l=1}^{m}\left(\frac{Sd_{\text{in}}}{n_{\text{tr}}}\right)^{m-l}N_{m,l}.
    \end{equation*}
    Then,
    \begin{align*}
    n_{\text{te}}\text{Bias}^2(\rvtheta_k^{SAM}) &= \E_{\rmA,\rmT}[\left(\E_{\rvtheta_0,s}[\rvtheta^{Sharpbal}_k]-\rvtheta^*\right)^T\rmT^T\rmT\left(\E_{\rvtheta_0,s}[\rvtheta^{Sharpbal}_k]-\rvtheta^*\right)]\\
    &=\frac{n_{\text{te}}}{d_\text{in}}\E_{\rmA}[(-\frac{1}{S}\sum\limits_{s=1}^S\rmB_s^k\rvtheta^*)^T(-\frac{1}{S}\sum\limits_{s'=1}^S\rmB_{s'}^k\rvtheta^*)]\\
    &= \frac{n_{\text{te}}}{d_\text{in}S^2}\E_{\rmA}[\sum\limits_{s=1}^S\rmB_s^k\rvtheta^*\sum\limits_{s'=1}^S\rmB_{s'}^k]\|\rvtheta^*\|^2_2\\
    &= \frac{n_{\text{te}}}{d_\text{in}S}\left(\phi'(2k,0)+(s-1)\phi'(k,0)^2\right)\|\rvtheta^*\|^2_2.
\end{align*}
The last equality is the result of applying $\E_\rmA[B_s^i] = \phi'(i,0)$ with different combinations of $\rmB_s$, $\rmB_{s'}$, counting multiplicity. Similarly,
\begin{align*}
    n_{\text{te}}\text{Error}(\rvtheta_k^{Sharpbal}) &= \E_{\rmA,\rmT,\rvtheta_0,s}[(\rvtheta^*) ^T\rmB_s^k\rmT^T\rmT\rmB_s^k \rvtheta^*]+\E_{\rmA,\rmT,\rvtheta_0,s}[\rvtheta_0 ^T\rmB_s^k\rmT^T\rmT\rmB_s^k \rvtheta_0]\\
    &=\frac{n_{\text{te}}}{d_\text{in}}\phi'(2k,0)\|\rvtheta^*\|^2_2 + n_{\text{te}}\phi'(2k,0)\sigma^2.
\end{align*}

 Therefore, 
    \begin{align*}
        \text{Var}\left(f(\rvtheta_k^{SharpBal};\rmT)\right) =&\frac{1}{n_{\text{te}}}\left(n_{\text{te}}\text{Error}(\rvtheta_k^{SharpBal}) -  n_{\text{te}}\text{Bias}^2(\rvtheta_k^{SharpBal})\right)\\
        =&\phi'(2k,0)\sigma^2 + \frac{S-1}{d_\text{in}S}\left(\phi'(2k,0)-\phi'(k,0)^2\right)\|\rvtheta^*\|^2_2.
    \end{align*}
When the model is trained on the submatrix, the sharpness of model $\rvtheta_k^{SharpBal}$ is defined as 
    \begin{equation*}
        \kappa_k^{SharpBal} = \mathbb{E}_{\rmA}[\max_{\|\rvvarepsilon\|_2 \leq \rho_0} f \left( \mathbb{E}_{\rvtheta_0, s} \left[ \rvtheta_k^{SharpBal} \right] + \rvvarepsilon;\rmA \right) - f \left( \mathbb{E}_{\rvtheta_0, s} \left[ \rvtheta_k^{SharpBal} \right] ;\rmA\right)].
    \end{equation*}
From a similar analysis of the proof for Theorem \ref{thm:varSAM},
    \begin{equation*}
        \kappa_k^{SharpBal} \leq \frac{\rho_0^2}{2}\left(\sqrt{\frac{n_{\text{tr}}}{d_{\text{in}}}} + 1\right)^2+ \frac{\rho_0}{S}\mathbb{E}_{\rmA}[\| \sum\limits_{s=1}^{S}\rmB_s^k \rmA^T\rmA\rvtheta^*\|_2],
    \end{equation*}
and with $r = \frac{n_{\text{tr}}}{Sd_\text{in}}$,
\begin{align*}
    \mathbb{E}_{\rmA}[\| \sum\limits_{s=1}^{S}\rmB_s^k \rmA^T\rmA\rvtheta^*\|_2] =&\E_{\rmA}[((\rvtheta^*)^T \rmA^T\rmA\sum\limits_{s=1}^{S}\rmB_s^k \sum\limits_{s'=1}^{S}\rmB_s'^k\rmA^T\rmA\rvtheta^*)^{1/2}]\\
    \leq&\left((\rvtheta^*)^T\E_{\rmA}[\sum\limits_{j=1}^S\rmA_j^T\rmA_j \sum\limits_{s=1}^{S}\rmB_s^k \sum\limits_{s'=1}^{S}\rmB_s'^k \sum\limits_{l=1}^S\rmA_l^T\rmA_l ] \rvtheta^*\right)^{1/2}\\
    =&(S\phi'(2k,2)+2rS(S-1)\phi'(2k,1)+2S(S-1)\phi'(k,2)\phi'(k,0)\\
    &+r(1+r)S(S-1)\phi'(2k,0)+2S(S-1)\phi'(k,1)\phi'(k,1)\\
    &+\frac{3}{2}r(1+r)S(S-1)(S-2)\phi'(k,0)^2\\
    &+\frac{3}{2}r^2S(S-1)(S-2)\phi'(2k,0)\\
    &+3rS(S-1)(S-2) \phi'(k,1) \phi'(k,0)\\
    &+r^2S(S-1)(S-2)(S-3)\phi'(k,0)^2)^{1/2}\|\rvtheta^*\|_2.
\end{align*}

The last equality is the result of applying $\E_\rmA[B_s^i(\rmA_s^T\rmA_s)^j] = \phi'(i,j)$ with different combinations of $\rmB_s$, $\rmB_{s'}$, $\rmA_j^T\rmA_j$, and $\rmA_l^T\rmA_l$, counting multiplicity and the fact that $\E_\rmA[(\rmA_s^T\rmA_s)^2] = r(1+r)\mathbf{I}$. In conclusion,

\begin{equation*}
        \kappa_k^{SharpBal} \leq \frac{\rho_0^2}{2}\left(\sqrt{\frac{n_{\text{tr}}}{d_{\text{in}}}} + 1\right)^2+ \frac{\rho_0}{S}\sqrt{C}\|\rvtheta^*\|_2,
\end{equation*}
where
\begin{align*}
        C =& S\phi'(2k, 2)+2rS(S-1)\phi'(2k, 1) + 2S(S-1)  \phi'(k, 2)  \phi'(k, 0)\\
        & +r(1+r) S(S-1)   \phi'(2k, 0)+ 2S(S-1) \phi'(k, 1) \phi'(k, 1) \\ 
        &+ \frac{3}{2}r(1+r)S(S-1)(S-2)  \phi'(k, 0)^2+ \frac{3}{2}r^2S(S-1)(S-2)\phi'(2k, 0)\\
        &+ 3rS(S-1)(S-2)\phi'(k, 0)\phi'(k, 1) + r^2S(S-1)(S-2)(S-3)\phi'(k, 0)^2.
\end{align*}

\end{proof}
The claims in Theorem \ref{thm:sub} is further supported by the experimental validations with results presented in Figure \ref{fig:apx_verify}. 

\begin{figure}
\centering
\begin{tabular}{cc}
\begin{subfigure}{0.42\linewidth}
\includegraphics[width=\linewidth,keepaspectratio]{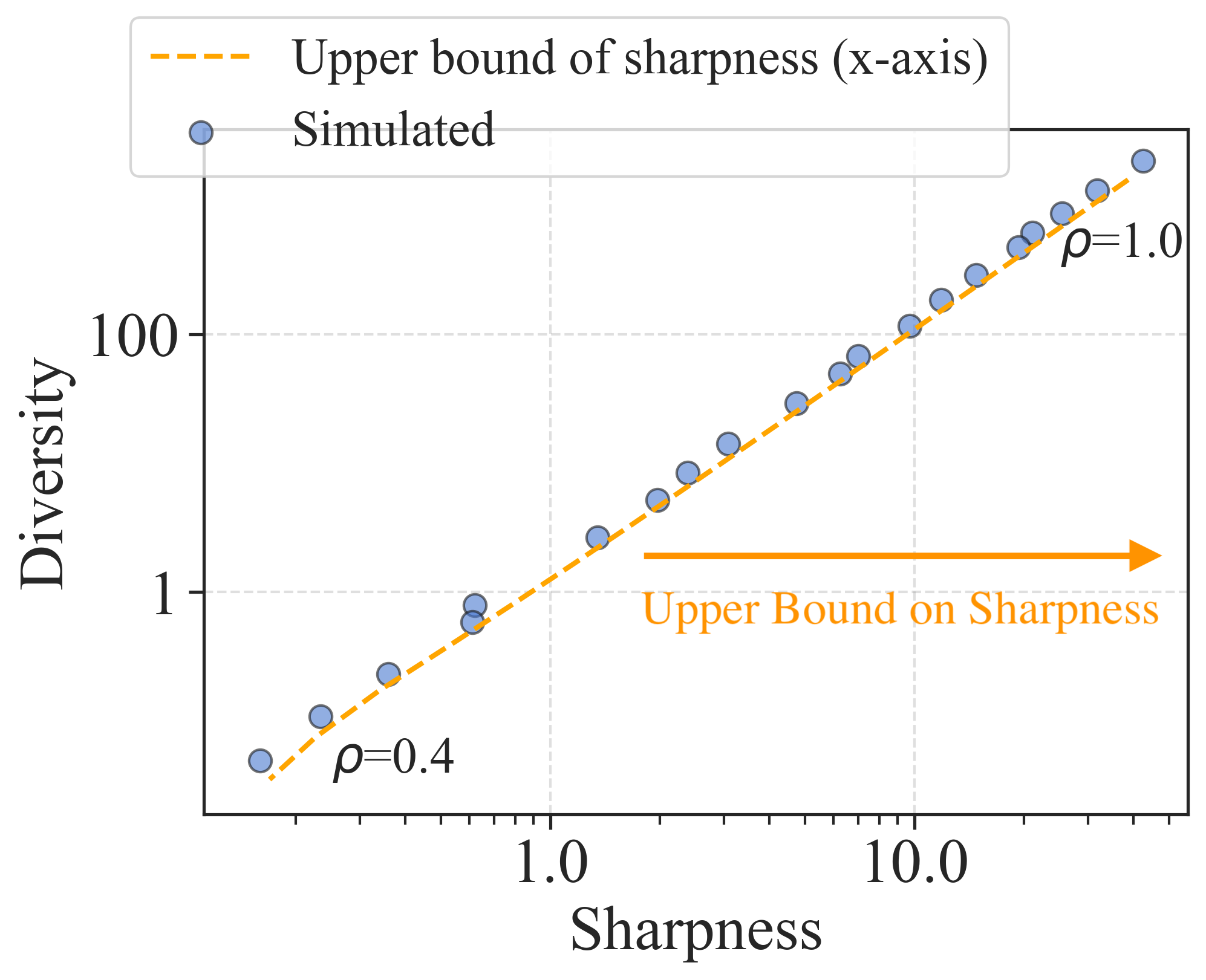}
\caption{Varying perturbation radius $\rho$ \label{subfig: theory_vary_rho}}
\end{subfigure}
&
\begin{subfigure}{0.42\linewidth}
\includegraphics[width=\linewidth,keepaspectratio]{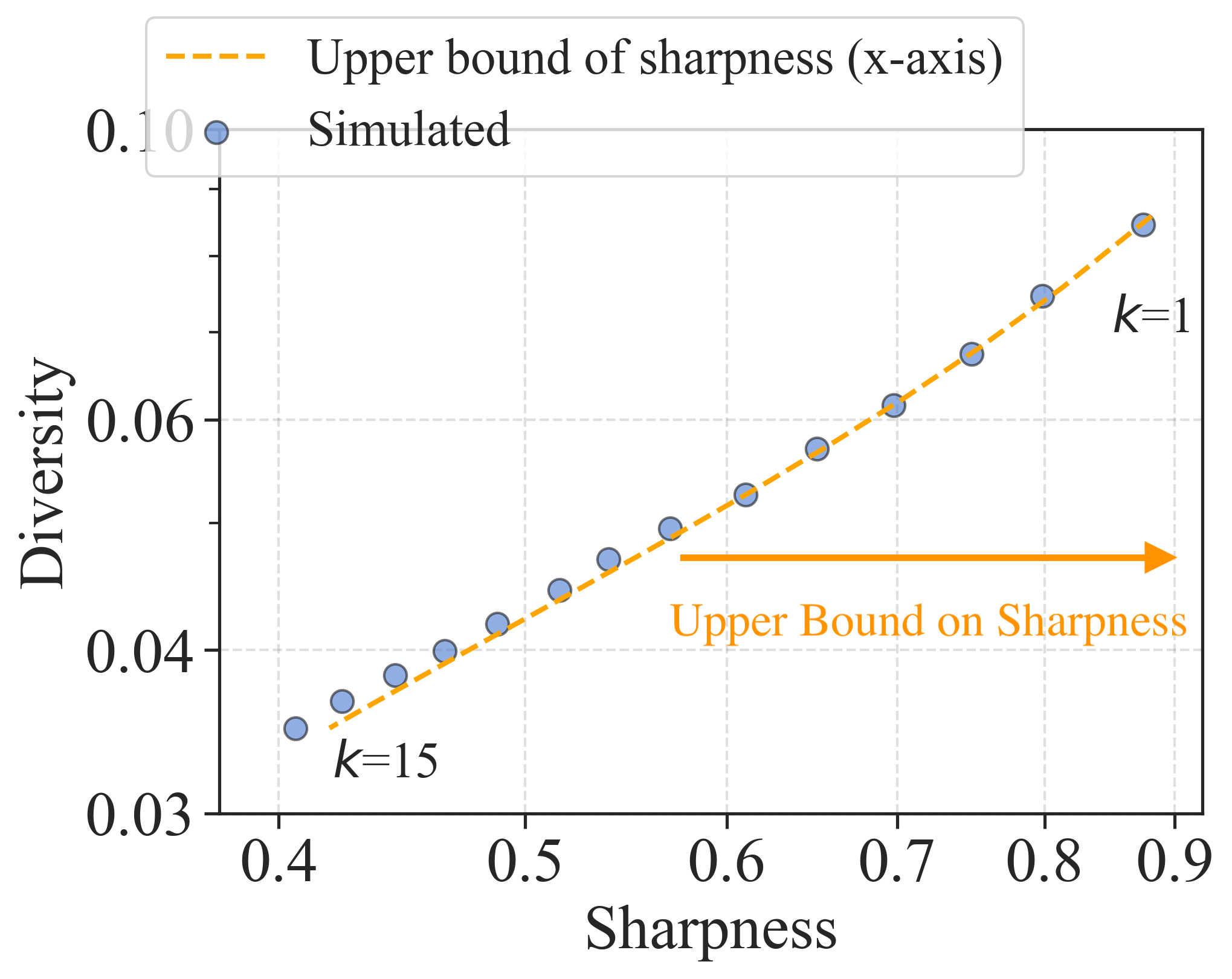}
\caption{Varying number of training iterations $k$\label{subfig: theory_vary_k}}
\end{subfigure}
\end{tabular}
\caption{
\textbf{(Theoretical vs. Simulated sharpness-diversity trade-off in \ourmethod)} This figure illustrates the relationship between sharpness(upper bound) and diversity as predicted by Theorem \ref{thm:sub} and as observed in simulations under different configurations. \textbf{(a)} validates our theoretical results by varying the perturbation radius $\rho$ from $1.0$ to $0.4$. \textbf{(b)} validates the derivation by varying number of iterations $k$ from $1$ to $15$. These results demonstrate the soundness of our derivation across a range of parameters.
 \vspace{-3mm}
\label{fig:apx_verify}}
\end{figure}

\subsection{Empirical Verification of Theorem \ref{thm:varSAM} and \ref{thm:sub}}
\label{apd:verify}
\begin{figure}[!ht]
\centering
\begin{subfigure}{0.32\linewidth}
\includegraphics[width=\linewidth,keepaspectratio]{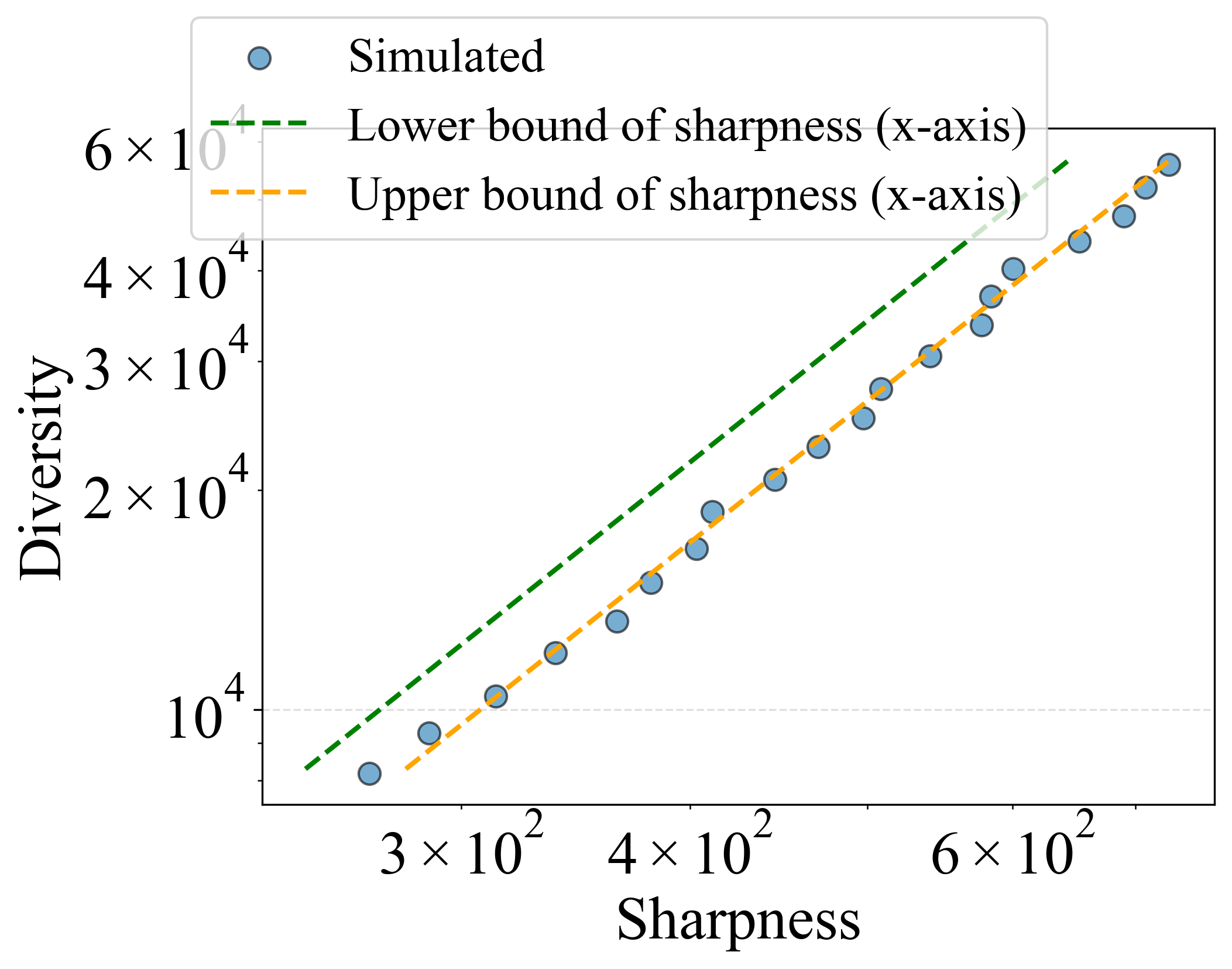} 
\caption{\footnotesize $k=2,\eta =0.1$} 
\end{subfigure}  
\centering
\begin{subfigure}{0.32\linewidth}
\includegraphics[width=\linewidth,keepaspectratio]{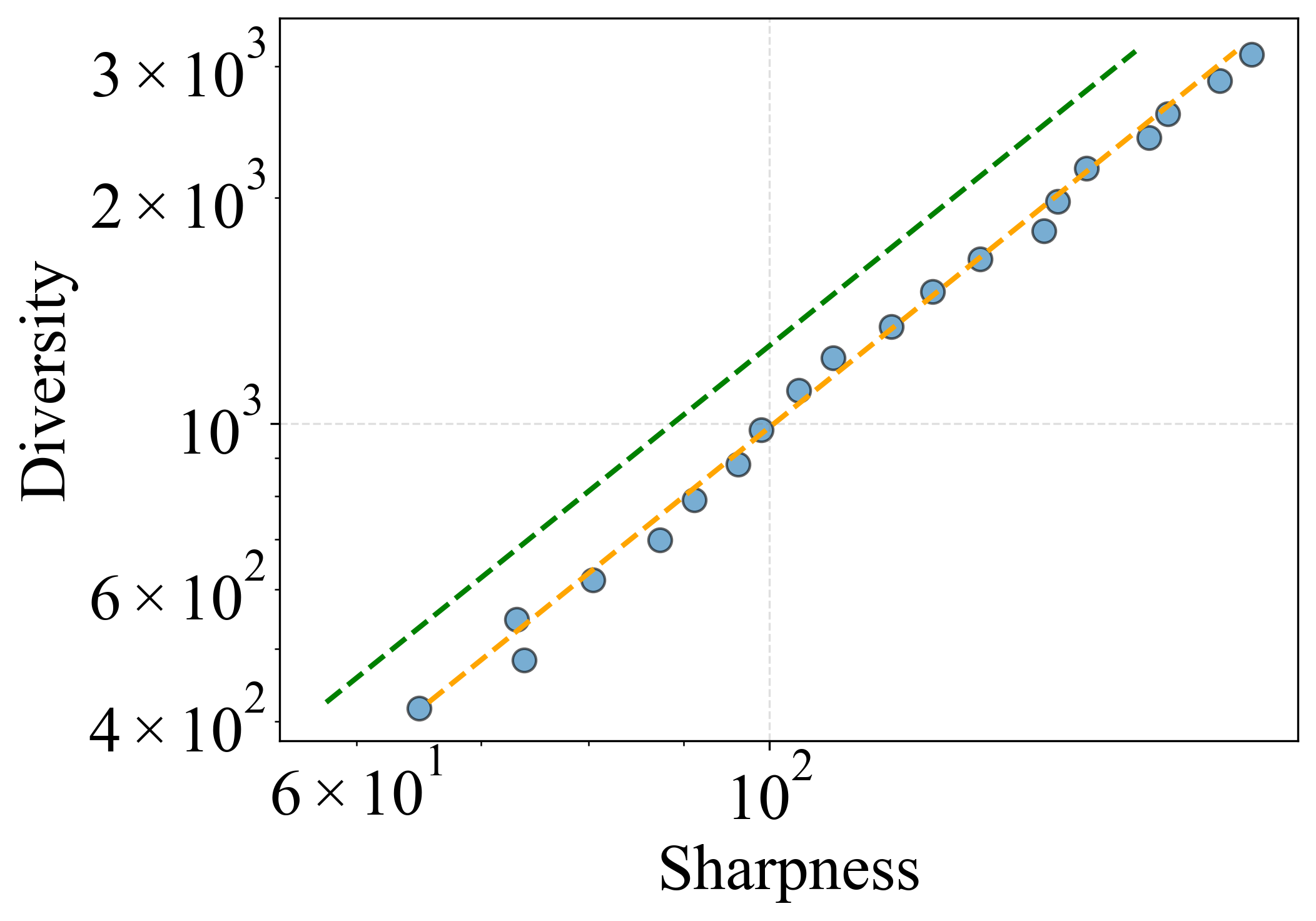} 
\caption{\footnotesize $k=2,\eta =0.05$} 
\end{subfigure}  
\centering
\begin{subfigure}{0.32\linewidth}
\includegraphics[width=\linewidth,keepaspectratio]{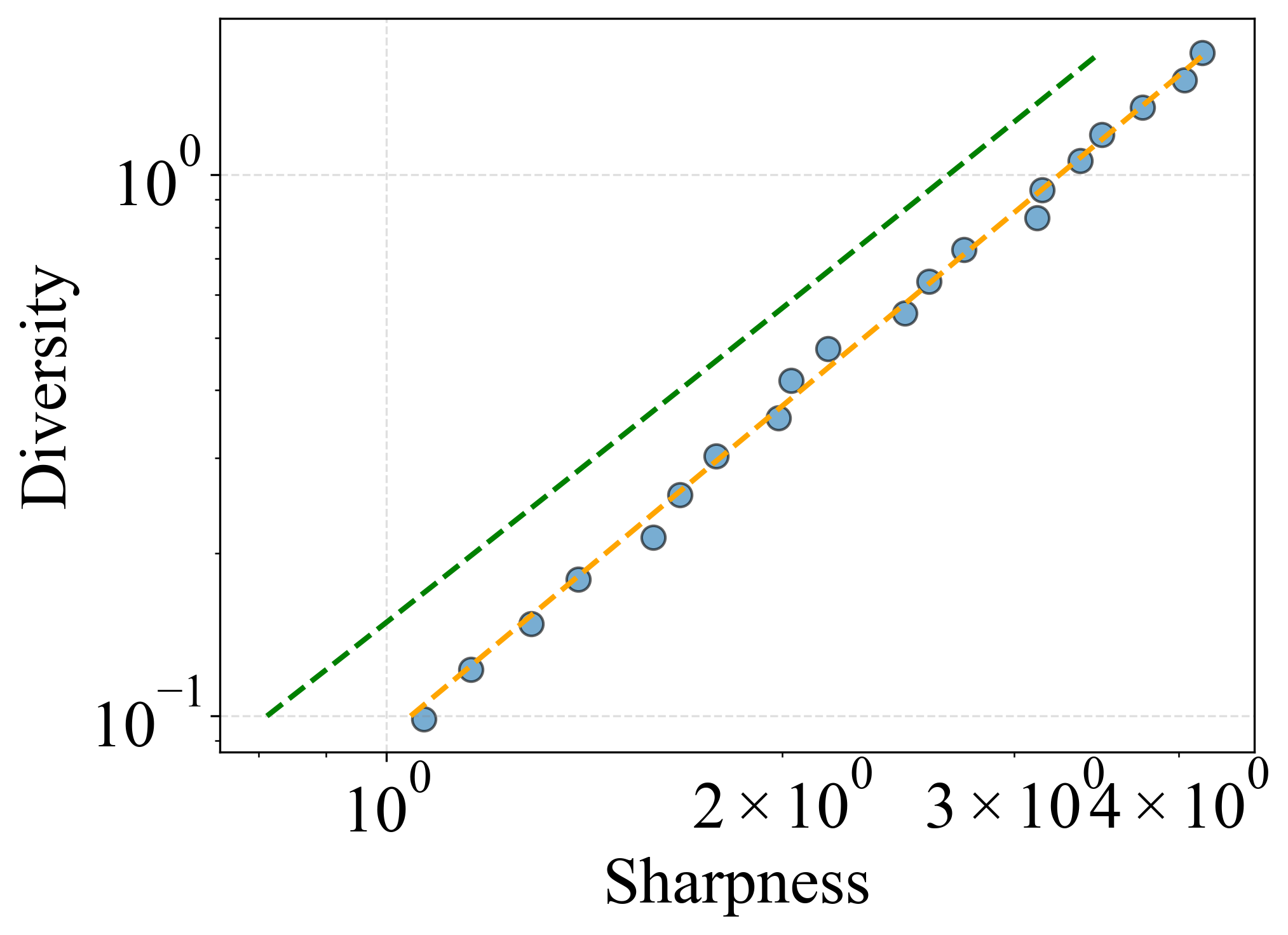} 
\caption{\footnotesize $k=2,\eta =0.01$} 
\end{subfigure}  
\centering
\begin{subfigure}{0.32\linewidth}
\includegraphics[width=\linewidth,keepaspectratio]{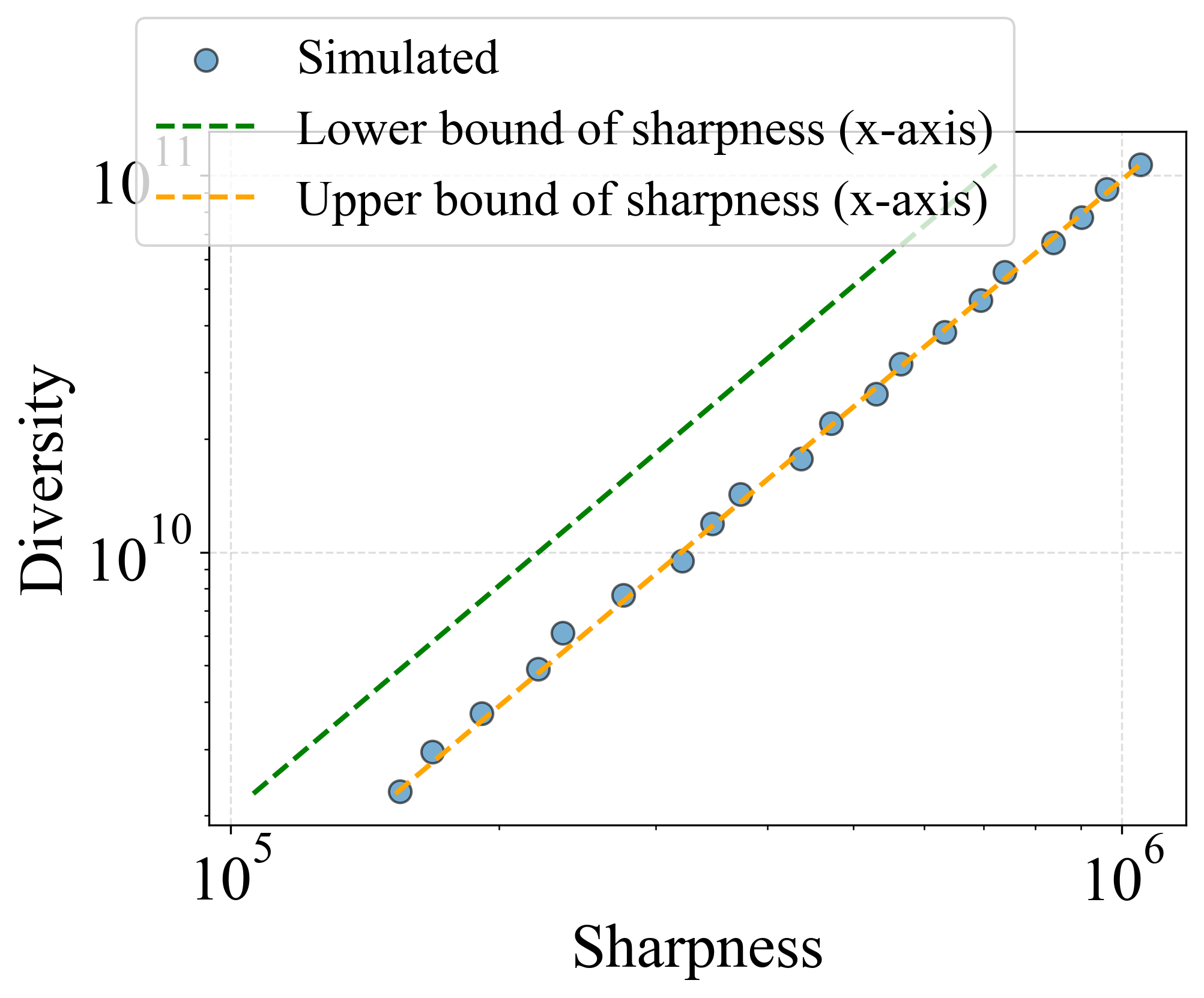} 
\caption{\footnotesize $k=4,\eta =0.1$} 
\end{subfigure}  
\centering
\begin{subfigure}{0.32\linewidth}
\includegraphics[width=\linewidth,keepaspectratio]{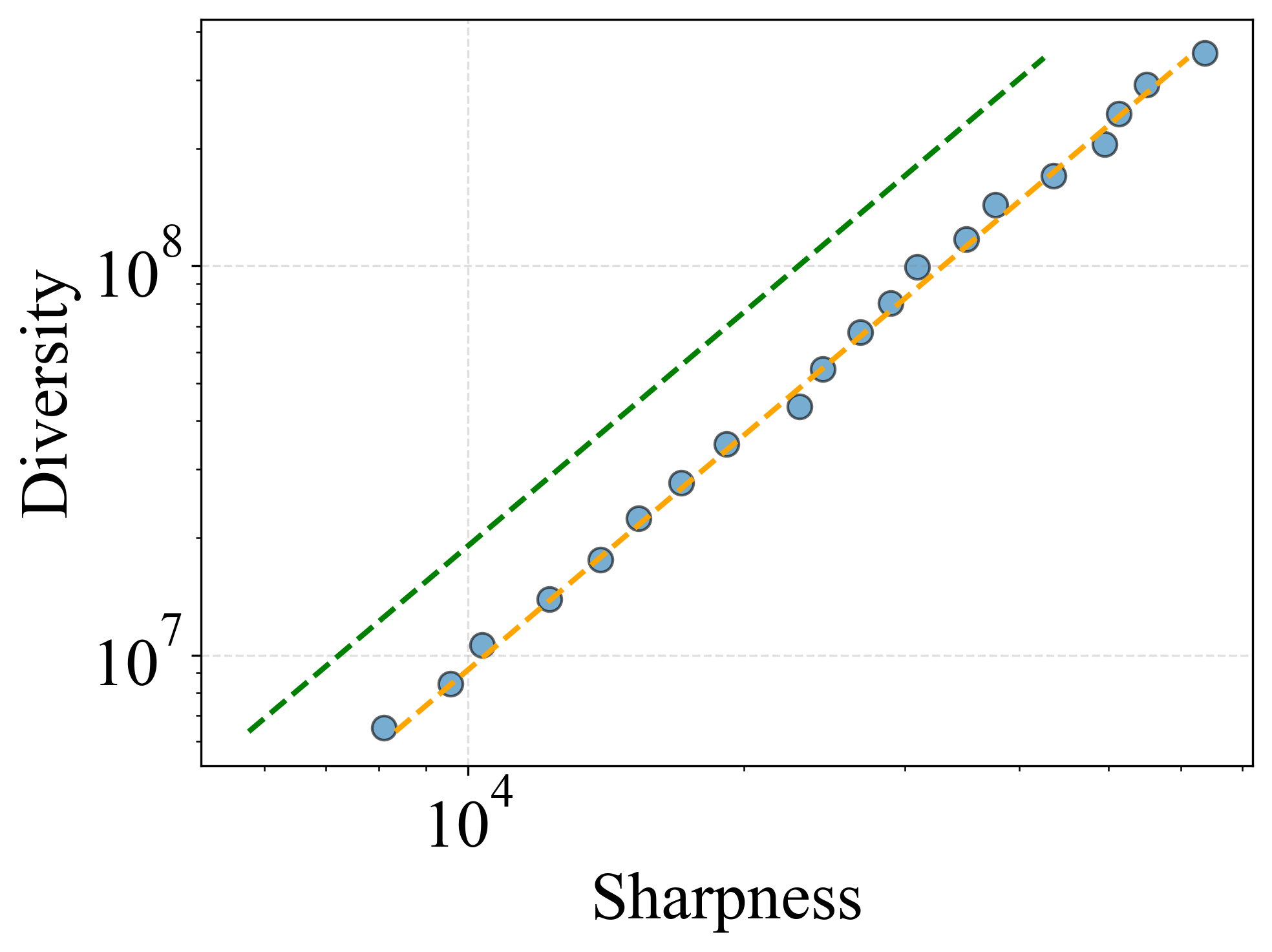}
\caption{\footnotesize $k=4,\eta =0.05$} 
\end{subfigure}  
\centering
\begin{subfigure}{0.32\linewidth}
\includegraphics[width=\linewidth,keepaspectratio]{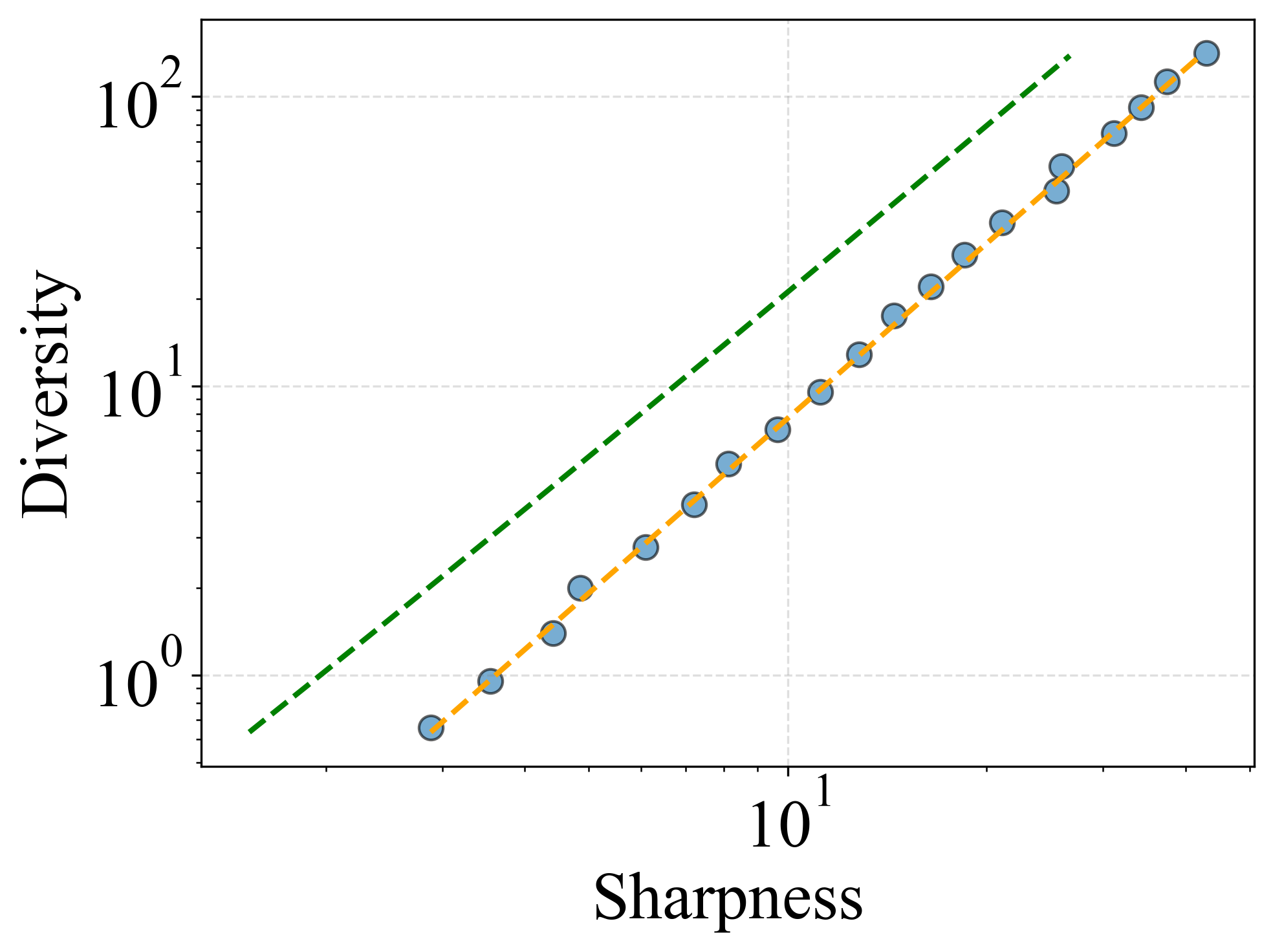} 
\caption{\footnotesize $k=4,\eta =0.01$} 
\end{subfigure}  
\centering
\begin{subfigure}{0.32\linewidth}
\includegraphics[width=\linewidth,keepaspectratio]{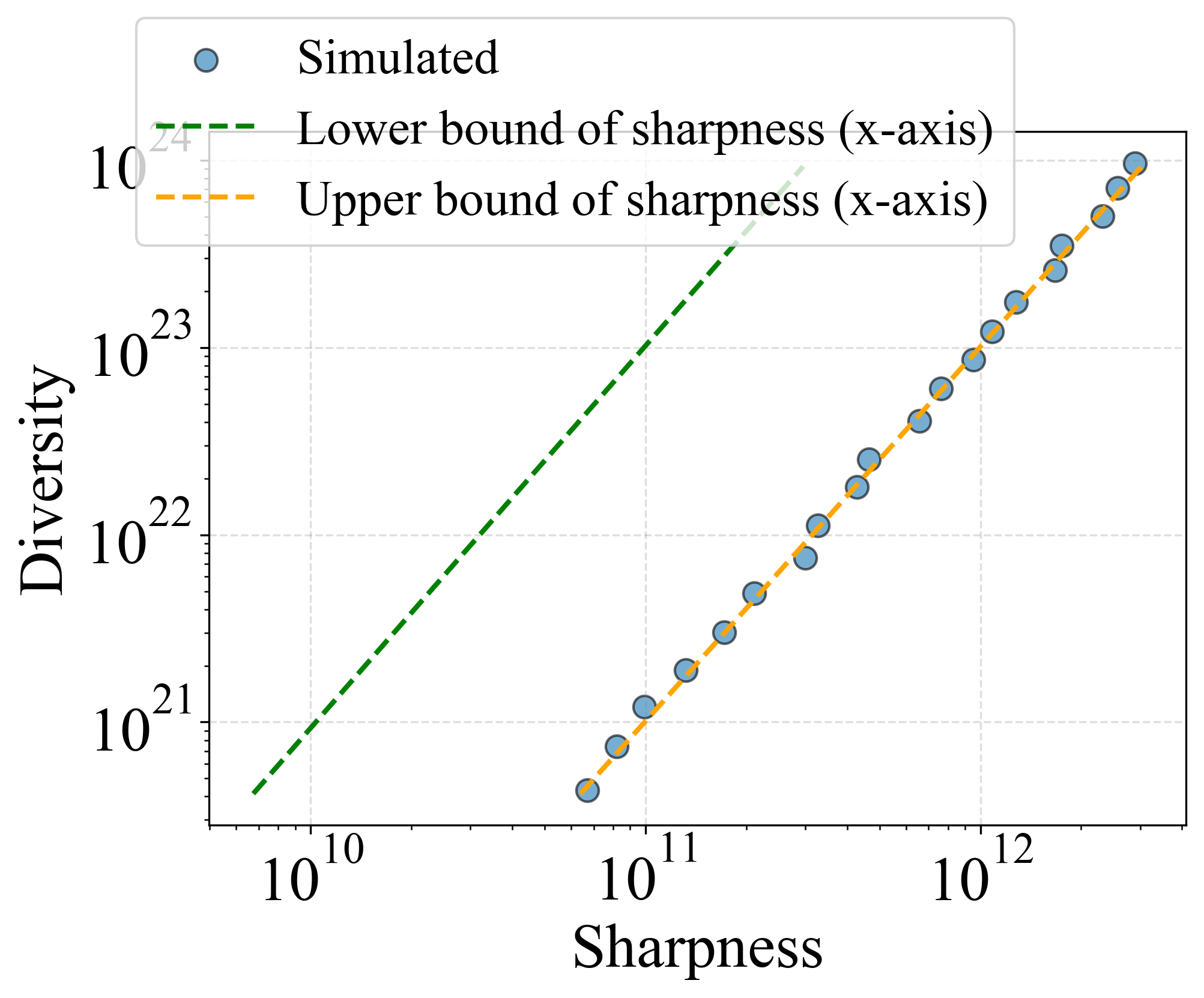} 
\caption{\footnotesize $k=8,\eta =0.1$} 
\end{subfigure}  
\centering
\begin{subfigure}{0.32\linewidth}
\includegraphics[width=\linewidth,keepaspectratio]{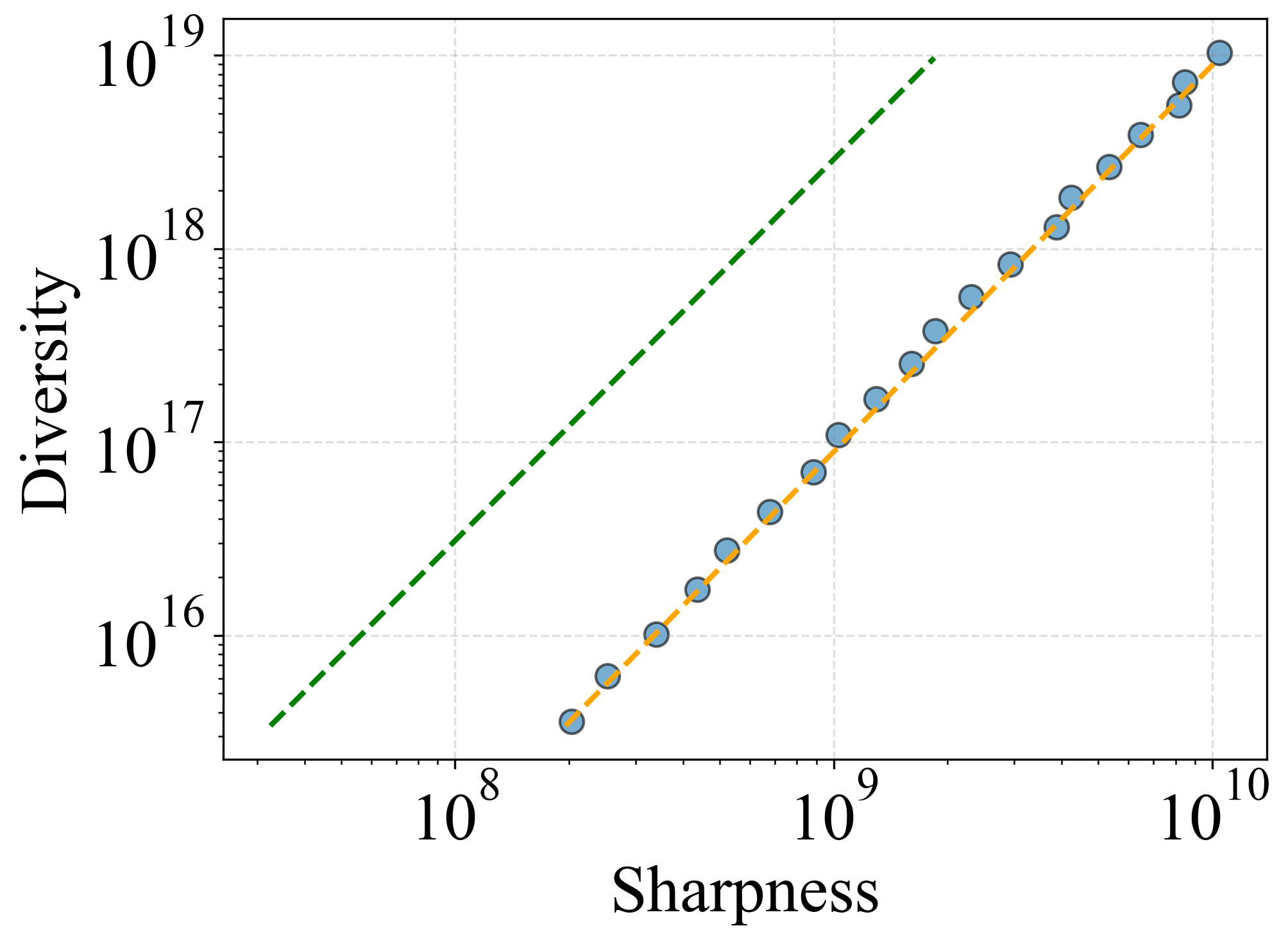} 
\caption{\footnotesize $k=8,\eta =0.05$} 
\end{subfigure}  
\centering
\begin{subfigure}{0.32\linewidth}
\includegraphics[width=\linewidth,keepaspectratio]{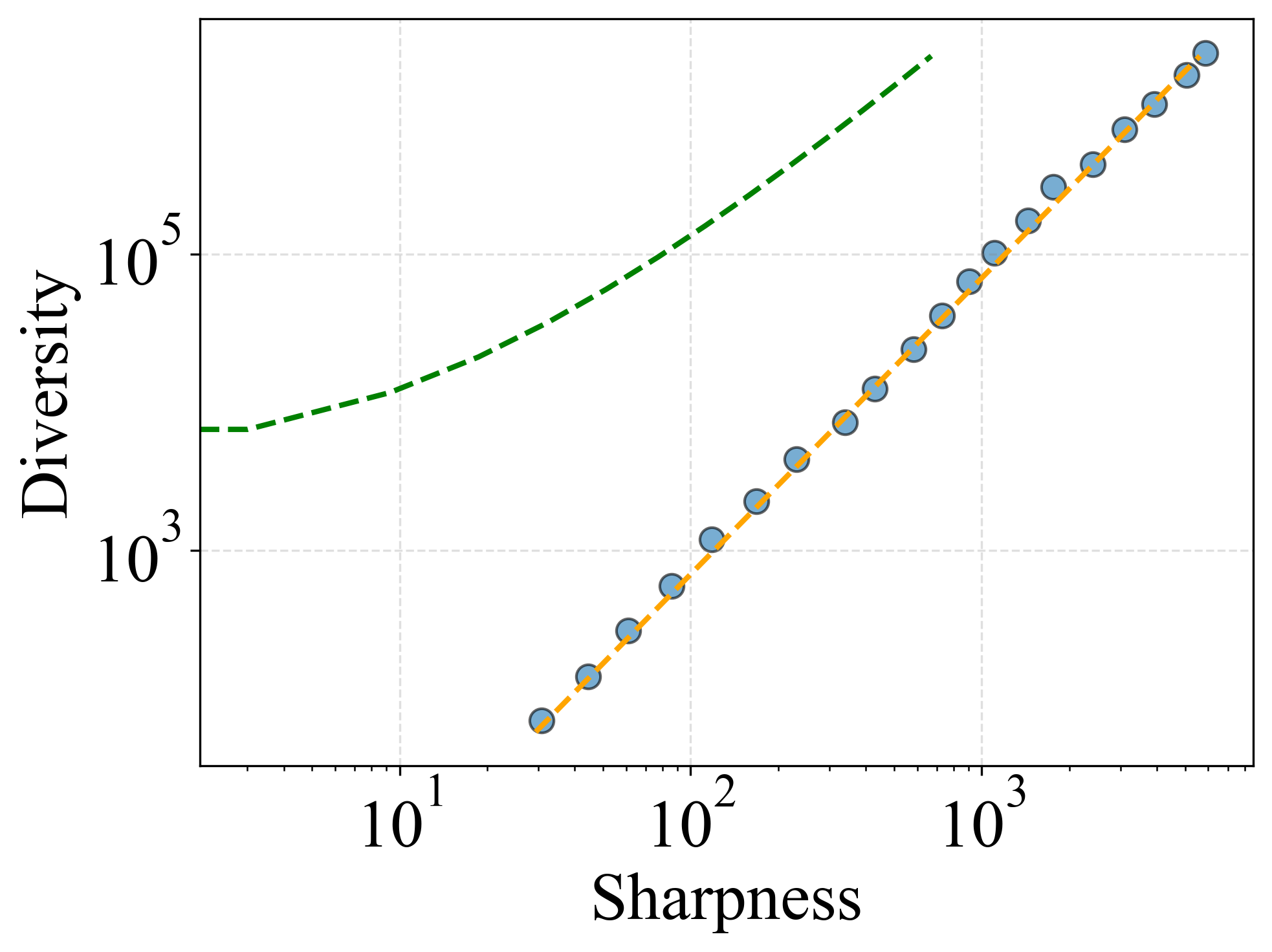} 
\caption{\footnotesize $k=8,\eta =0.01$} 
\end{subfigure}  
\caption{
\textbf{(Theoretical vs. Simulated sharpness-diversity trade-off in \sam).}
This figure compares the sharpness and diversity as predicted by Theorem \ref{thm:varSAM} and as observed in simulations under various parameter configurations. Results demonstrates the robustness of our theoretical analysis and tightness of the derived sharpness upper bound.
} 
\label{fig:SAM verify}
\end{figure}

To demonstrate the robustness and tightness of the bounds presented in Theorem \ref{thm:varSAM}, we provide verification results across a range of parameter configurations. Interestingly, the observed model behaviors closely align with the upper bound derived in Theorem \ref{thm:varSAM}, highlighting the effectiveness of our theoretical analysis in capturing the underlying dynamics of the ensemble.  Figure \ref{fig:SAM verify} illustrates these results, with each sub-figure corresponding to a specific combination of $k$ and $\eta$ with $\rho$ from range $0.5$ to $0.3$. In these experiment, we generated 50 random data matrices $\rmA$ of size $3000 \times 150$ and test data $\rmT$ of size $1000 \times 150$. For each random dataset, we initialized 50 random model weights $\rvtheta_0$ and collected the expected statistics of interest after training. To measure the sharpness $\kappa_k^{SAM}$, we employed projected gradient ascent to find the optimal perturbation, using a step size of $0.01$ and a maximum of $50$ steps. Similar experiments are performed to verify the derivations in Theorem \ref{thm:sub} with results presented in Figure \ref{fig:sharpbal verify}, with the number of partitions $S=10$.

\begin{figure}[!ht]
\centering
\begin{subfigure}{0.32\linewidth}
\includegraphics[width=\linewidth,keepaspectratio]{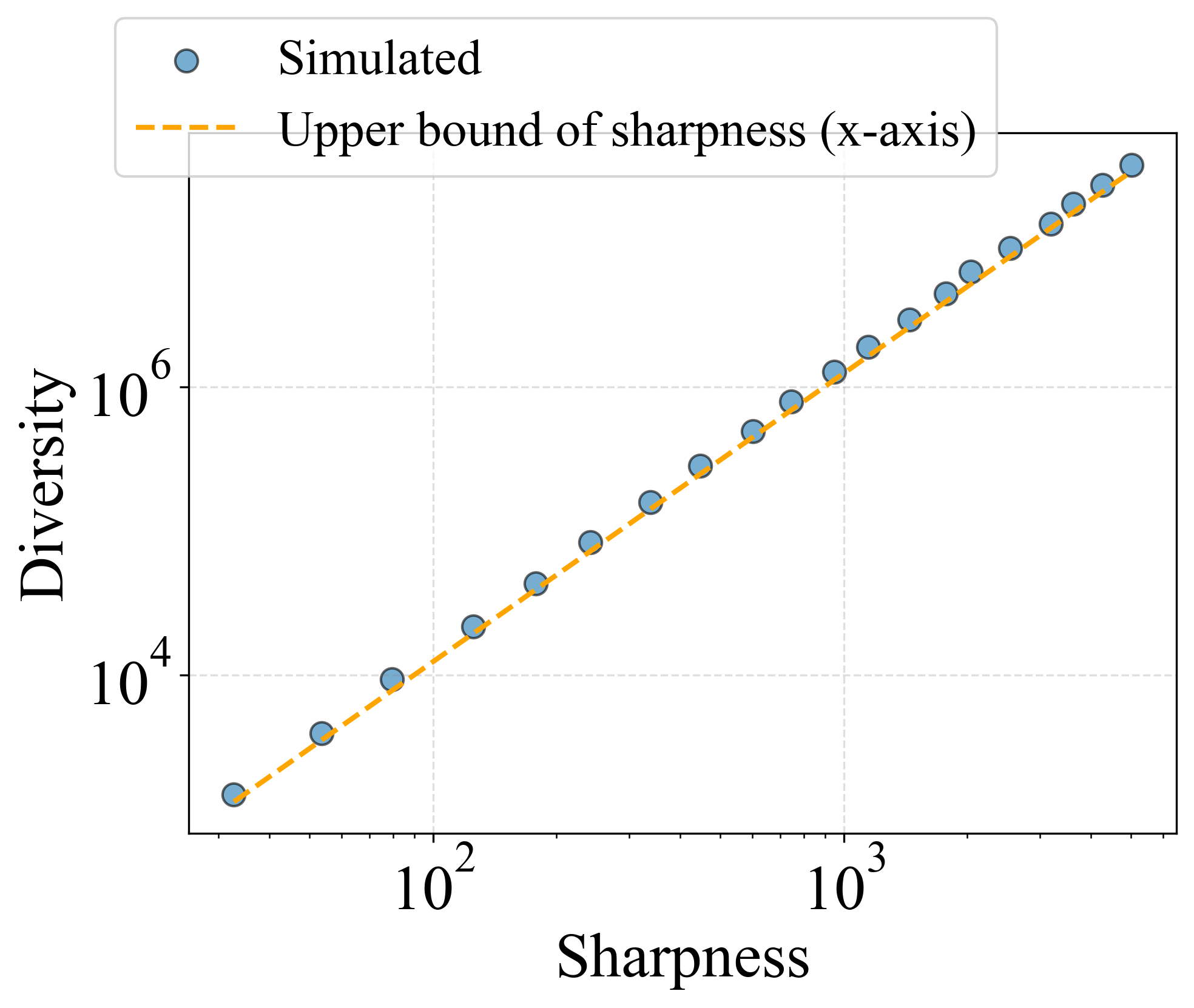}
\caption{\footnotesize $k=4,\eta =0.5$} 
\end{subfigure}  
\centering
\begin{subfigure}{0.32\linewidth}
\includegraphics[width=\linewidth,keepaspectratio]{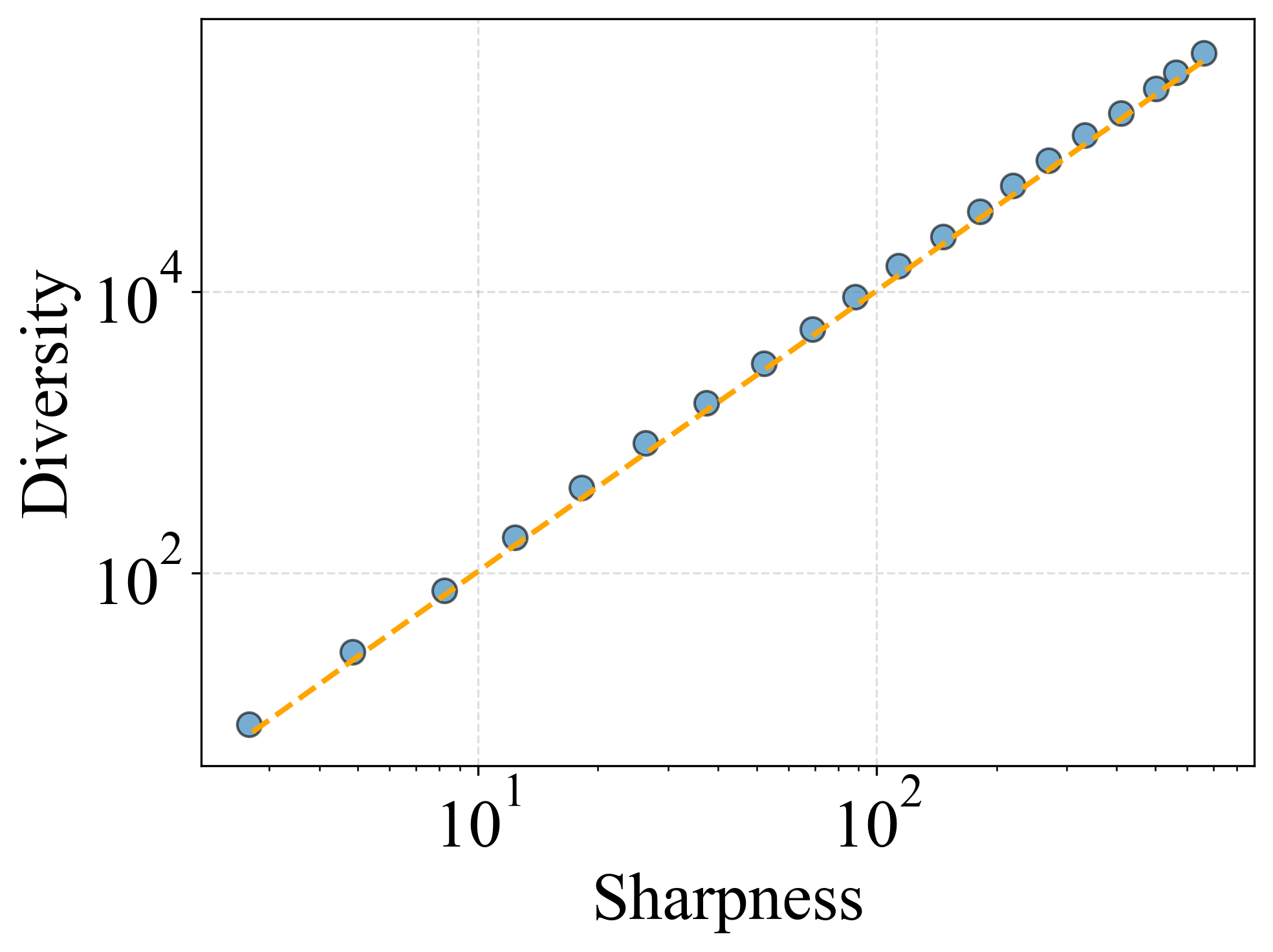}
\caption{\footnotesize $k=4,\eta =0.3$} 
\end{subfigure}  
\centering
\begin{subfigure}{0.32\linewidth}
\includegraphics[width=\linewidth,keepaspectratio]{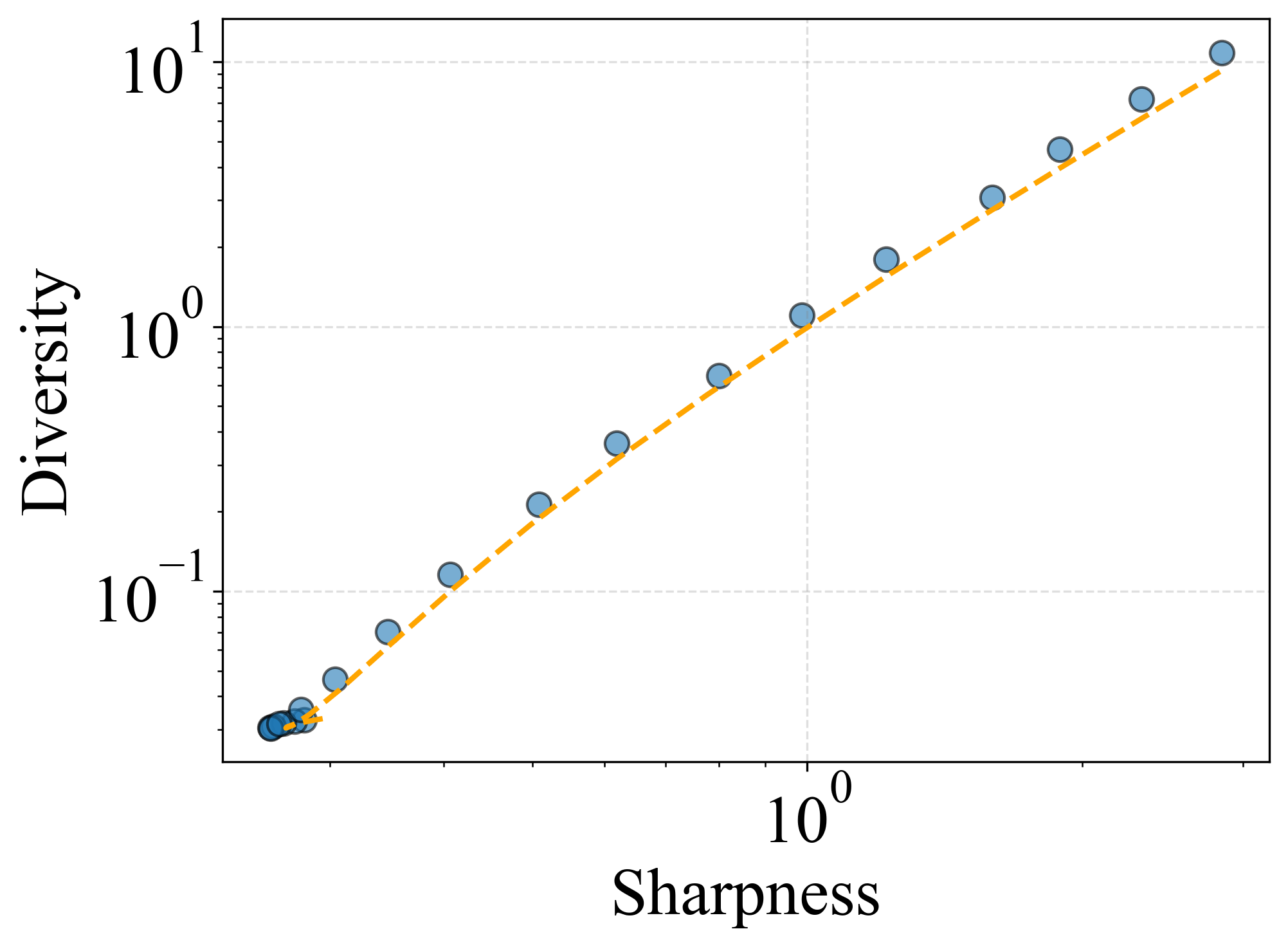}
\caption{\footnotesize $k=4,\eta =0.1$} 
\end{subfigure}  
\centering
\begin{subfigure}{0.32\linewidth}
\includegraphics[width=\linewidth,keepaspectratio]{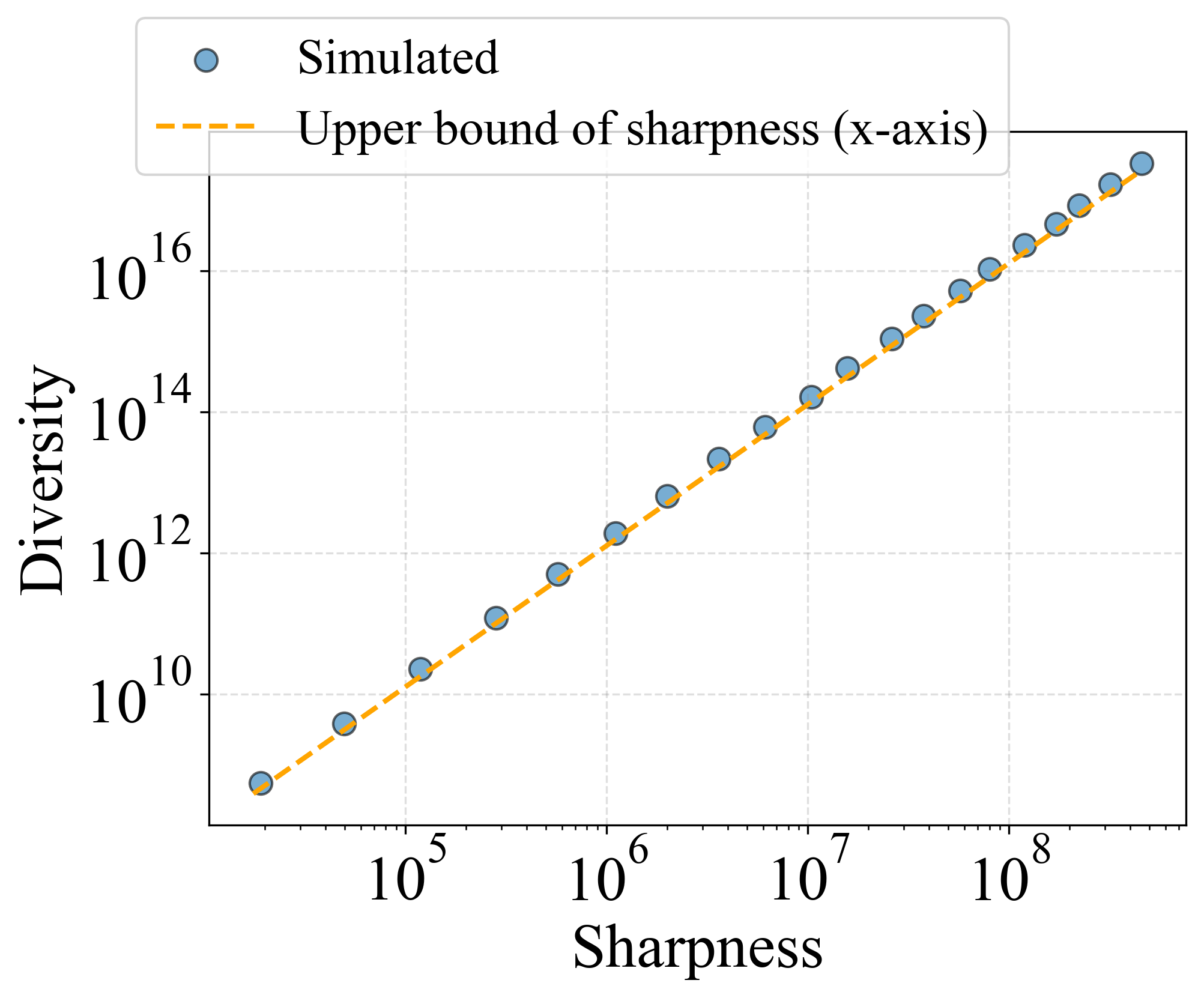} 
\caption{\footnotesize $k=8,\eta =0.5$} 
\end{subfigure}  
\centering
\begin{subfigure}{0.32\linewidth}
\includegraphics[width=\linewidth,keepaspectratio]{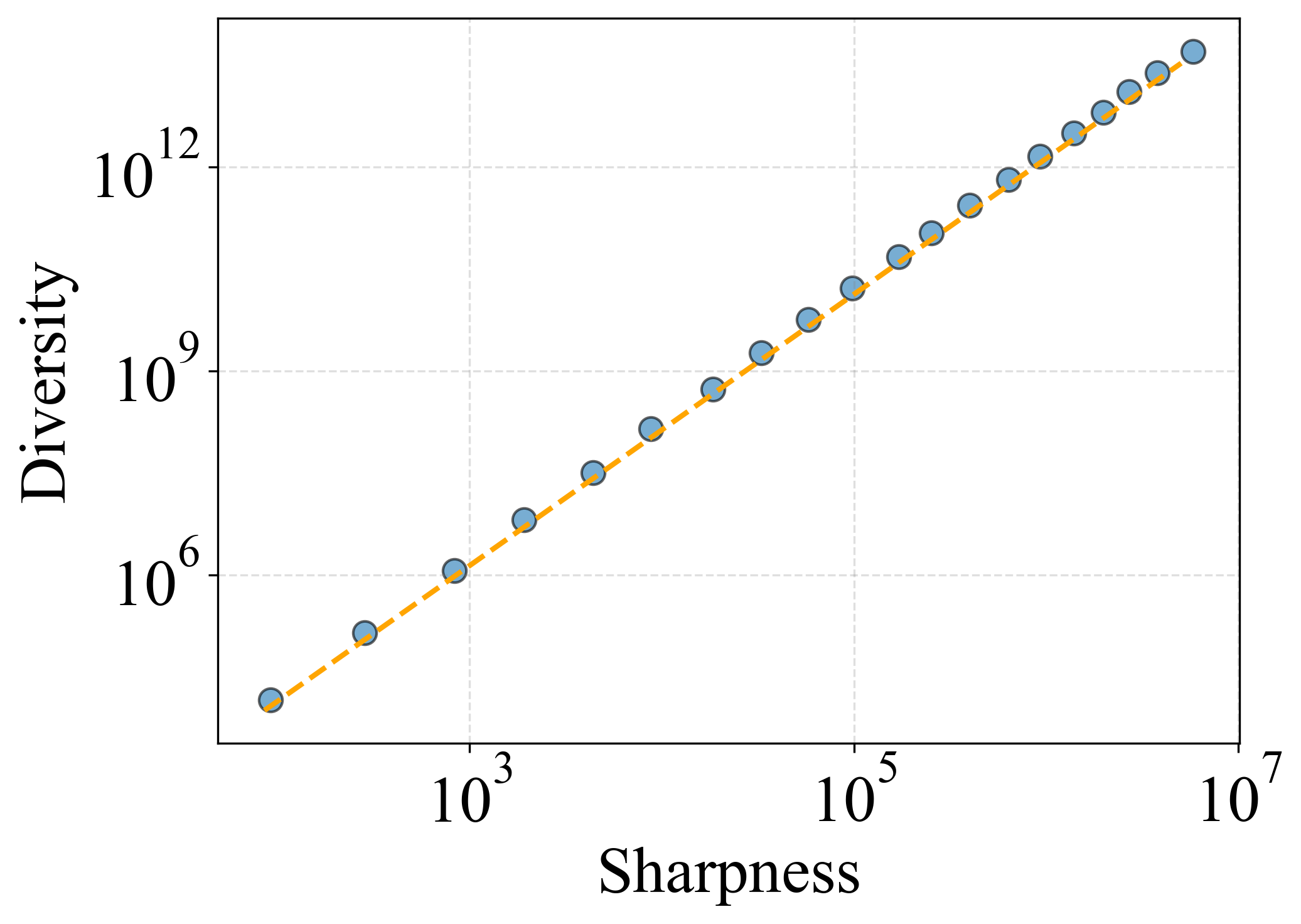}
\caption{\footnotesize $k=8,\eta =0.3$} 
\end{subfigure}  
\centering
\begin{subfigure}{0.32\linewidth}
\includegraphics[width=\linewidth,keepaspectratio]{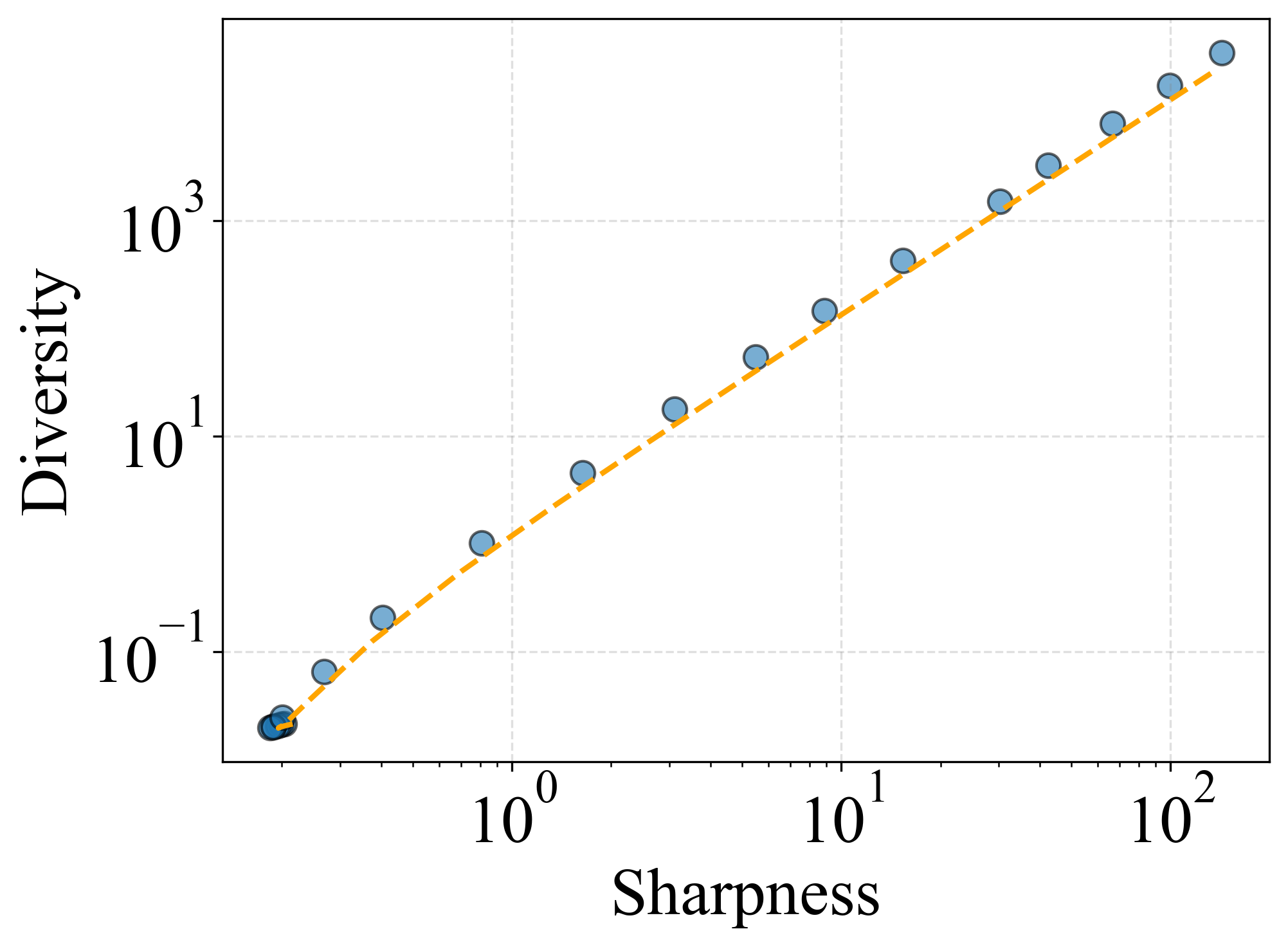} 
\caption{\footnotesize $k=8,\eta =0.1$} 
\end{subfigure}  
\centering
\begin{subfigure}{0.32\linewidth}
\includegraphics[width=\linewidth,keepaspectratio]{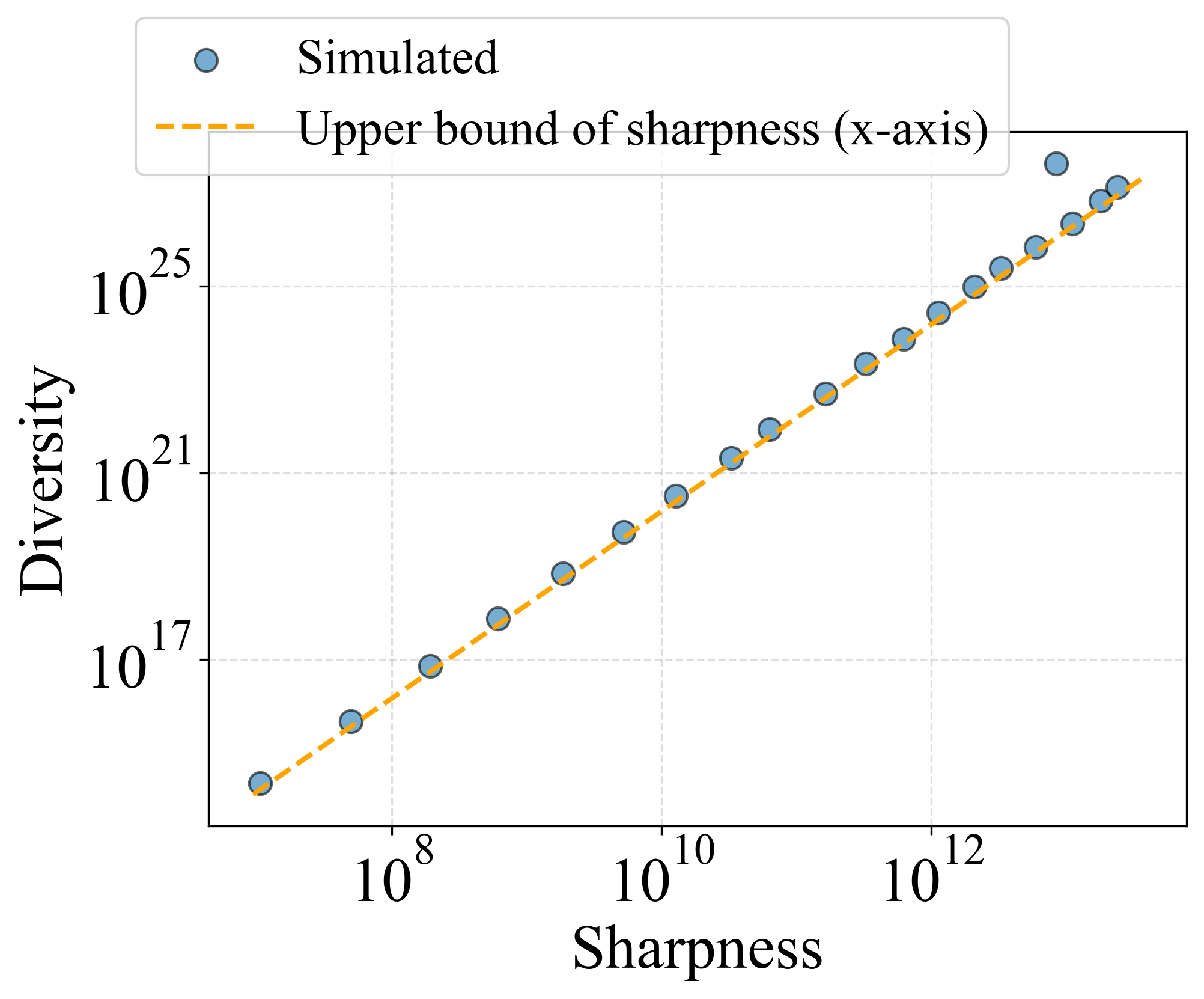} 
\caption{\footnotesize $k=12,\eta =0.5$} 
\end{subfigure}  
\centering
\begin{subfigure}{0.32\linewidth}
\includegraphics[width=\linewidth,keepaspectratio]{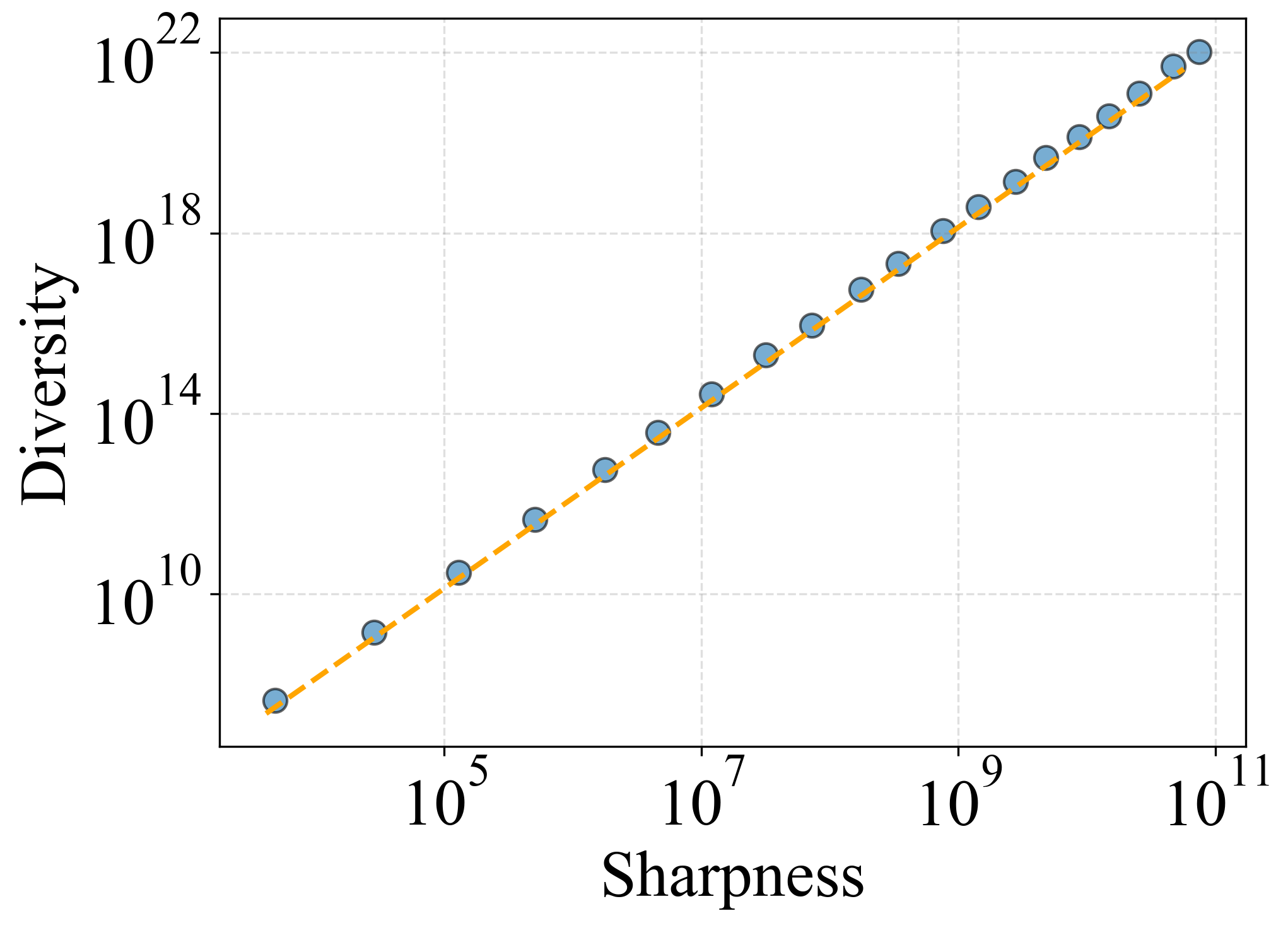} 
\caption{\footnotesize $k=12,\eta =0.3$} 
\end{subfigure}  
\centering
\begin{subfigure}{0.32\linewidth}
\includegraphics[width=\linewidth,keepaspectratio]{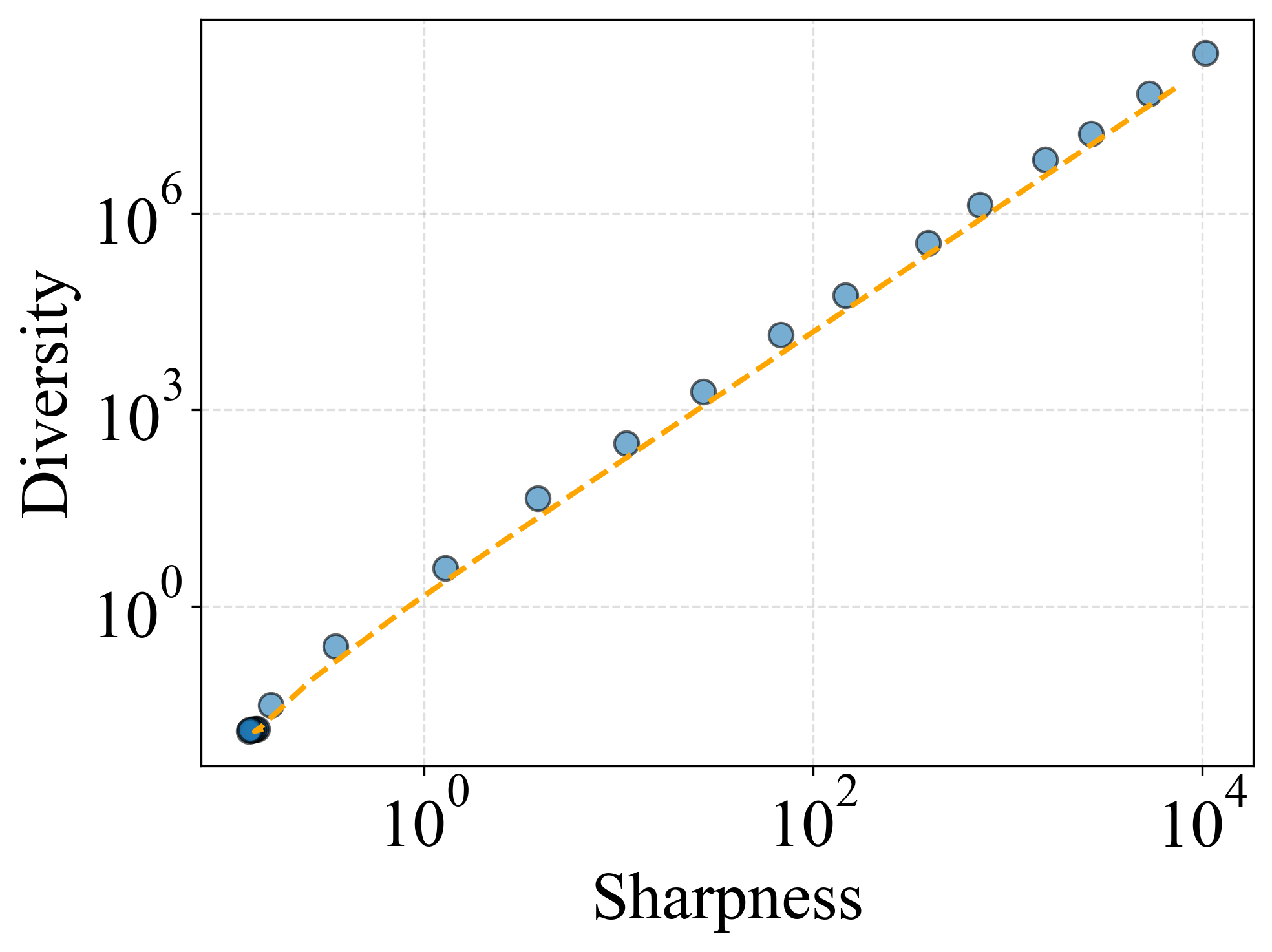} 
\caption{\footnotesize $k=12,\eta =0.1$} 
\end{subfigure}  
\caption{
\textbf{(Theoretical vs. Simulated sharpness-diversity trade-off in \ourmethod).}
This figure compares the sharpness and diversity as predicted by Theorem \ref{thm:sub} and as observed in simulations under various parameter configurations. The observed model behaviors align closely with our derived upper bounds.
} 
\label{fig:sharpbal verify}
\end{figure}
\section{Hyperparamter setting}~\label{abl:hyper}

\subsection{Datasets}
We first evaluate on image classification datasets CIFAR-10 and CIFAR-100. The corresponding OOD robustness is evaluated on CIFAR-10C and CIFAR-100C~\citep{hendrycks2019benchmarking}.
The experiments are carried out on ResNet18~\citep{he2016deep}. We 
use a batch size of 128, a momentum of 0.9, and a
weight decay of 0.0005 for model training. 
TinyImageNet is an image classification dataset consisting of 100K images for training and 10K images for in-distribution testing. 
We evaluate ensemble's OOD robustness on TinyImageNetC~\citep{hendrycks2019benchmarking}. 

\subsection{Hyperparamter setting for empirical sharpness-diversity trade-off}
Here, we provide the hyperparameter for the experiments in Section \ref{sec:emp-trade-off}. 
When using adaptive worst-case sharpness for sharpness measurement, the size of neighborhood $\gamma$ defined in \eqref{eqn:sharpness} needed to be specified, we use a $\gamma$ of 0.5 for all the results in Figure ~\ref{fig:overview}
and Figure \ref{fig:vary-overpara}. 
Additionally, when training NNs in the ensemble, we change the perturbation radius $\rho$ of \sam so that we can study the trade-off. 
The range of $\rho$ for the results in Figure \ref{fig:overview} is $\{0.01, 0.02, 0.03, 0.04, 0.05, 0.1, 0.2, 0.3\}$, the range of $\rho$ for the results in Figure \ref{fig:vary-overpara} is $\{0.01, 0.015, 0.02, 0.025, 0.03,  0.05,  0.1,0.2,0.3,0.4\}$.

\subsection{Hyperparamter setting for \ourmethod}
\textbf{Hyperparameter setting on CIFAR10/100}. For experiments on CIFAR10/100, we train a NN from scratch with basic data augmentations, including random cropping, padding by four pixels and random horizontal flipping.  We use a batch size of 128, a momentum of 0.9, and a weight decay of 0.0005. For deep ensemble, we train each model for 200 epochs. 

In addition, we use  10\% of the training set as the validation set for selecting  $\rho$ and $k$ based on ensemble's performance. 
We make a grid search for $\rho$ over $\{0.01,0.02,0.05,0.1,0.2,0.5\}$. For \ourmethod, we use the same $\rho$ as \sam and search $k$ over $\{0.2,0.3,0.4,0.5,0.6\}$. 
$T_d$ is another hyperparameter introduced by \ourmethod, we use a $T_d$ of 10 for all experiments on CIFAR10, a $T_d$ of 100 and 150 respectively when training dense and sparse models on CIFAR100. See Table \ref{tab:hyparamter} for the optimal $\rho$ and $k$ after grid search. \looseness-1

\noindent
\textbf{Hyperparameter setting on TinyImageNet}. For experiments on TinyImageNet, we adopt basic data augmentations, including random cropping, padding by four pixels and random horizontal flipping. We train each model for 
200 epochs. We use a batch size of 128, a momentum of 
0.9, a weight decay of 5e-4, a $T_d$ of 100, an initial learning
rate of 0.1, and decay it with a factor of 10 at epoch 100 and 150. We search $\rho$ and $k$ in the same range as what we 
do on CIFAR10/100. See Table \ref{tab:hyparamter} for the optimal $\rho$ and $k$ after grid search. \label{abl:hyper_tinyimagenet}

\begin{table}[!th]
\small
\centering
\resizebox{0.7\columnwidth}{!}{
\begin{tabular}{ccccccc}
\hline
\rowcolor[gray]{0.9}
\textbf{Dataset} & \textbf{Model} &\textbf{Method} &\textbf{$\rho$} &\textbf{$k$} &\textbf{$T_d$}\\ 

\multirow{3}{*}{CIFAR10} & \multirow{1}{*}{ResNet18}&Deep Ensemble & - &-&-\\ 

&ResNet18 &Deep Ensemble+\sam &  0.2 & - &- \\ 

&ResNet18 & \ourmethod    &0.2 & 0.4 & 100\\ 

\hline

\multirow{3}{*}{CIFAR100} & \multirow{1}{*}{ResNet18}&Deep Ensemble & - &- &- \\ 

&ResNet18 &Deep Ensemble+\sam & 0.2 &- &- \\ 

&ResNet18 & \ourmethod  & 0.2 &0.5 &100  \\ 

\hline
\multirow{3}{*}{TinyImageNet} & \multirow{1}{*}{ResNet18}&Deep Ensemble & - &- &- \\ 

&ResNet18 &Deep Ensemble+\sam & 0.2 &- &-  \\ 

&ResNet18 &  \ourmethod  & 0.2 & 0.3 & 100 \\ 

\hline
\end{tabular}%
}
\vspace{3mm}
\caption{Hyperparamter setting for results in Section \ref{sec:eval-perf}, we report the optimal $\rho$ and $k$ after grid search. Each result in Figure~\ref{fig:compare_results} is averaged over three ensembles, which corresponds to 9 random seeds, the random seeds we use are $\{13,17,27,113,117,127,43, 59, 223\}$. }
\label{tab:hyparamter}
\end{table}

\section{Ablation studies on loss landscape metrics}
\label{abl:abl-loss}
In this section, we show that the sharpness-diversity trade-off generalizes to different measurements of sharpness and diversity. The results are presented in Figure~\ref{fig:tradeoff-linf-DER}.\looseness-1

\begin{figure}[!ht]
\centering
\begin{subfigure}{0.32\linewidth}
\includegraphics[width=\linewidth,keepaspectratio]{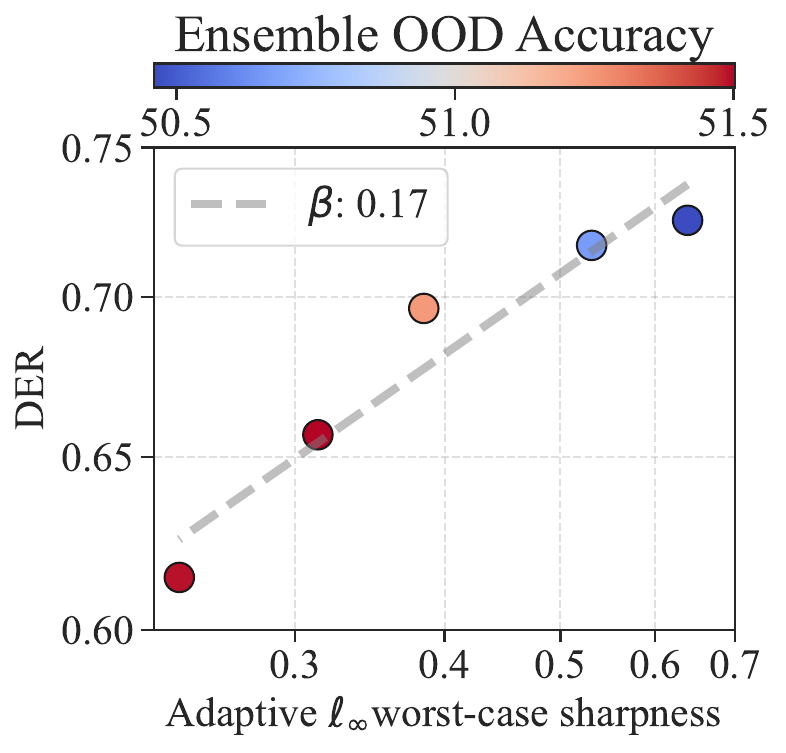}
\caption{\footnotesize Diff sharpness (CIFAR-100)}  
\end{subfigure}
\centering
\begin{subfigure}{0.32\linewidth}
\includegraphics[width=\linewidth,keepaspectratio]{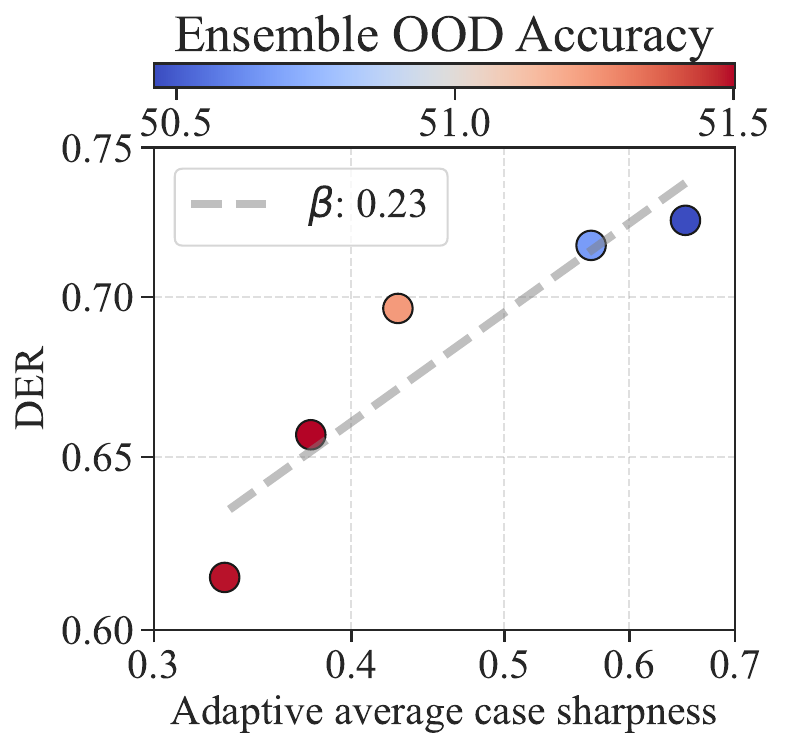}
\caption{\footnotesize Diff sharpness (CIFAR-100)}  
\end{subfigure}
\begin{subfigure}{0.32\linewidth}
\includegraphics[width=\linewidth,keepaspectratio]{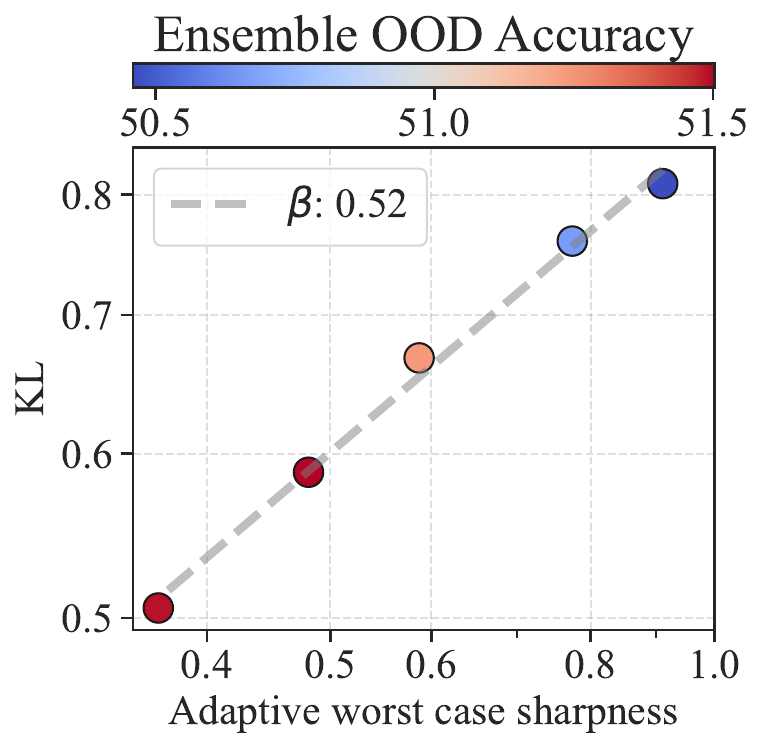}
\caption{\footnotesize Diff diversity (CIFAR-100)}  
\end{subfigure}
\centering
\begin{subfigure}{0.32\linewidth}
\includegraphics[width=\linewidth,keepaspectratio]{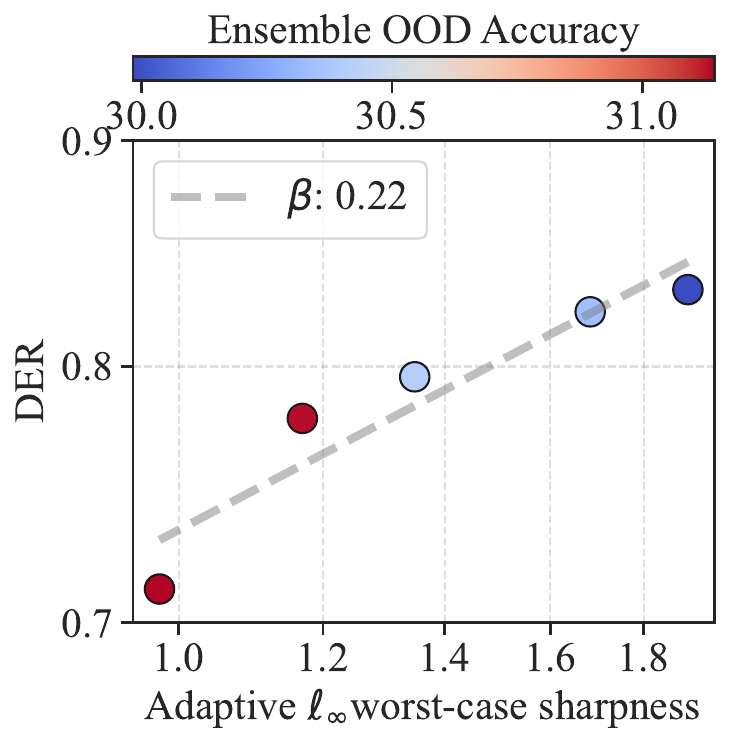} 
\caption{\footnotesize Diff sharpness (TIN)} 
\end{subfigure}  
\begin{subfigure}{0.32\linewidth}
\includegraphics[width=\linewidth,keepaspectratio]{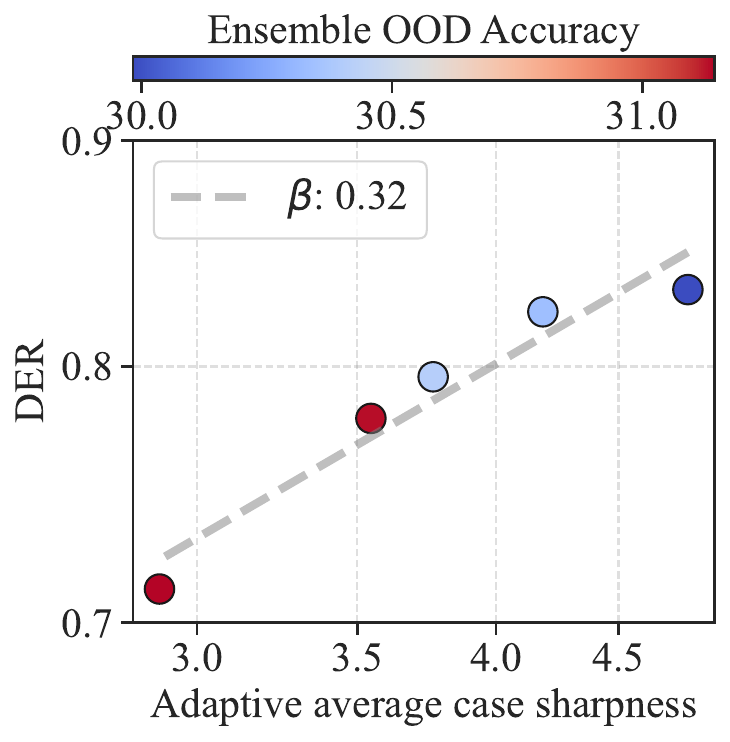} 
\caption{\footnotesize Diff sharpness (TIN)} 
\end{subfigure} 
\begin{subfigure}{0.32\linewidth}
\includegraphics[width=\linewidth,keepaspectratio]{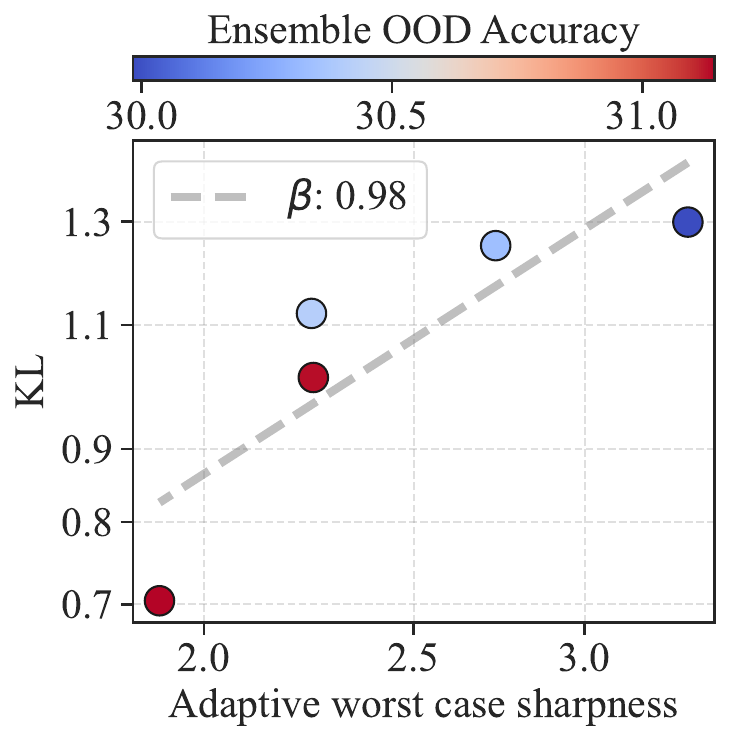} 
\caption{\footnotesize Diff diversity (TIN)} 
\end{subfigure}  
\centering
\caption{
\textbf{(Ablation study of varying sharpness and diversity metrics to corroborate existence of sharpness-diversity trade-off).}
(a)(d) Varying sharpness metric by using the adaptive $\ell_{\infty}$ worst-case sharpness.
(b)(e) Varying sharpness metric by using the adaptive $\ell_2$ average case sharpness.
(c)(f) Varying diversity metric by using the KL divergence.
The sharpness-diversity trade-off is still observed in all the settings.
The $x$-axis and $y$-axis are in log scale. 
The notation $\beta$ stands for the slope of the linear regression function fitted on all the ensembles trained by \sam.
} 
\label{fig:tradeoff-linf-DER}
\end{figure}

\noindent \textbf{Sharpness metric.}
In the main paper, we use adaptive worst-case sharpness defined in~\eqref{eqn:sharpness}, the parameter neighborhood is bounded by $\ell_2$ norm. 
In this section, we consider two more sharpness metrics~\citep{kwon2021asam, andriushchenko2023modern}: adaptive worst-case sharpness with the parameter neighborhood bounded by $\ell_{\infty}$ norm~(referred to as adaptive $\ell_{\infty}$ worst-case sharpness); and adaptive average case sharpness bounded by $\ell_{2}$ norm~(termed average case sharpness). 
\\The adaptive $\ell_{\infty}$ worst-case sharpness is defined as:
\begin{equation}\label{eqn:sharpness-linf1}
    {\underset{\| T_\rvtheta^{-1} \rvvarepsilon \| _{\infty} \leq \rho_0}{\text{max}} }\text{}\mathcal{L}_{\mathcal{D}}(\rvtheta+\rvvarepsilon)-\mathcal{L}_{\mathcal{D}}(\rvtheta) .
\end{equation} 
The average case sharpness is defined as: 
\begin{equation}\label{eqn:sharpness-linf2}
  \mathbb{E}_{\boldsymbol{\rvvarepsilon} \sim \mathcal{N}\left(0, \rho_0^2 \operatorname{diag}\left({T_\rvtheta}^2\right)\right) } \quad \mathcal{L}_{\mathcal{D}}(\rvtheta+\rvvarepsilon)-\mathcal{L}_{\mathcal{D}}(\rvtheta) ,
\end{equation} 
where $\rho_0$ is the neighborhood size of current parameter $\rvtheta$. 
$T_\rvtheta$ is a normalization operator that ensures the sharpness measure is invariant with respect to the re-scaling operation of the parameter.
The results, illustrated in Figures~\ref{fig:tradeoff-linf-DER}, corroborate our observation of a trade-off between sharpness and diversity.\looseness-1

\noindent 
\textbf{Diversity metric}. We consider 
Kullback–Leibler~(KL) Divergence~\citep{10.1214/aoms/1177729694} as an alternative diversity metric, which is also widely used in previous literature to gauge the diversity of two ensemble members~\citep{fort2019deep,liu2021deep}. Specifically, the KL-divergence between the outputs of two ensemble members given a data sample $(\vx, \vy)$ is defined as:
\begin{equation}\label{eqn:kl-divergence}
  \text{KL}\left(f_{\rvtheta_1}(\vx),f_{\rvtheta_2}(\vx)\right) = \mathbb{E}_{f_{\rvtheta_1}(\vx)}\left[\log f_{\rvtheta_1}(\vx)-\log f_{\rvtheta_2}(\vx)\right] . 
\end{equation}
We measure the KL divergence on each data sample in the test data and then average the measured KL divergence. 
The results for KL-divergence are shown in Figure~\ref{fig:tradeoff-linf-DER}, which demonstrate the trade-off remains consistent for different diversity metrics.

\section{More results}~\label{sec:more-results}
\subsection{Evaluation on different corruption severity}\label{abl:multiple_severity}
\ourmethod’s main advantage lies in OOD scenarios. As shown in Tabel~\ref{tb:severity1}-\ref{tb:severity3}, \ourmethod consistently outperforms the baselines on different levels of corruption.

\begin{table*}[!thb]
    \centering
    \caption{Results of different severity levels on CIFAR10-C.}
    \resizebox{1.0\linewidth}{!}{
    \begin{tabular}{lccccc}
        \toprule
        Corruption Severity & 1 & 2 & 3 & 4 & 5 \\
        \midrule
        Deep ensemble & 88.90 & 83.67 & 77.56 & 70.37 & 58.63  \\
        Deep ensemble+\sam & 89.44 & 84.24 & 78.16 & 71.04 & 58.77 \\
        \rowcolor{gray!15}
        \ourmethod & 89.75 \textcolor{blue}{(+0.31)} & 84.80 \textcolor{blue}{(+0.56)} & 78.98 \textcolor{blue}{(+0.82)} & 72.25 \textcolor{blue}{(+1.21)} & 60.78 \textcolor{blue}{(+2.01)} \\
        \bottomrule
    \end{tabular}
    }
    \label{tb:severity1}
\end{table*}

\begin{table*}[!thb]
    \centering
    \caption{Results of different severity levels on CIFAR100-C.}
    \resizebox{1.0\linewidth}{!}{
    \begin{tabular}{lccccc}
        \toprule
        Corruption Severity & 1 & 2 & 3 & 4 & 5 \\
        \midrule
        Deep ensemble & 65.78 & 57.77 & 51.30 & 44.33 & 34.16  \\
        Deep ensemble+\sam & 66.39 & 58.47 & 51.89 & 44.90 & 34.81 \\
        \rowcolor{gray!15}
        \ourmethod & 67.23 \textcolor{blue}{(+0.84)} & 59.53 \textcolor{blue}{(+1.06)} & 53.14 \textcolor{blue}{(+1.25)} & 46.19 \textcolor{blue}{(+1.29)} & 36.20 \textcolor{blue}{(+1.39)} \\
        \bottomrule
    \end{tabular}
    }
    \label{tb:severity2}
\end{table*}

\begin{table*}[!thb]
    \centering
    \caption{Results of different severity levels on Tiny-ImageNet-C.}
    \resizebox{1.0\linewidth}{!}{
    \begin{tabular}{lccccc}
        \toprule
        Corruption Severity & 1 & 2 & 3 & 4 & 5 \\
        \midrule
        Deep ensemble & 43.62 & 36.65 & 28.96 & 22.08 & 16.86 \\
        Deep ensemble+\sam & 45.20 & 38.04 & 30.19 & 22.98 & 17.71 \\
        \rowcolor{gray!15}
        \ourmethod & 46.48 \textcolor{blue}{(+1.28)} & 39.53 \textcolor{blue}{(+1.49)} & 31.70 \textcolor{blue}{(+1.51)} & 24.27 \textcolor{blue}{(+1.29)} & 18.69 \textcolor{blue}{(+0.98)} \\
        \bottomrule
    \end{tabular}
    }
    \label{tb:severity3}
\end{table*}

\subsection{Evaluation on WRN-40-2}\label{abl:WRN-40-2}
The results are presented in Figure~\ref{fig:compare_results_wrn} for CIFAR-10 and CIFAR-100.
These results confirm that \ourmethod consistently boosts both ID and OOD performance across the models and datasets studied.

\begin{figure}[!h]
    \centering
    \begin{subfigure}{0.8\textwidth}
    \includegraphics[width=\textwidth]{figs/results/abl_metric_legend.png}
    \end{subfigure} \\
    \centering
    \begin{subfigure}{0.45\textwidth}
    \includegraphics[width=\textwidth]{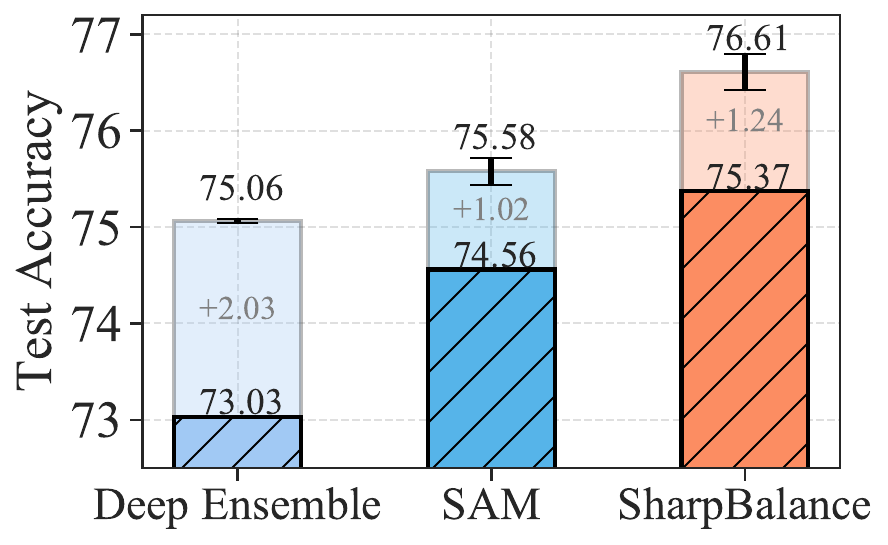}
    \caption{CIFAR10-C}
    \end{subfigure}
    \begin{subfigure}{0.45\textwidth}
    \includegraphics[width=\textwidth]{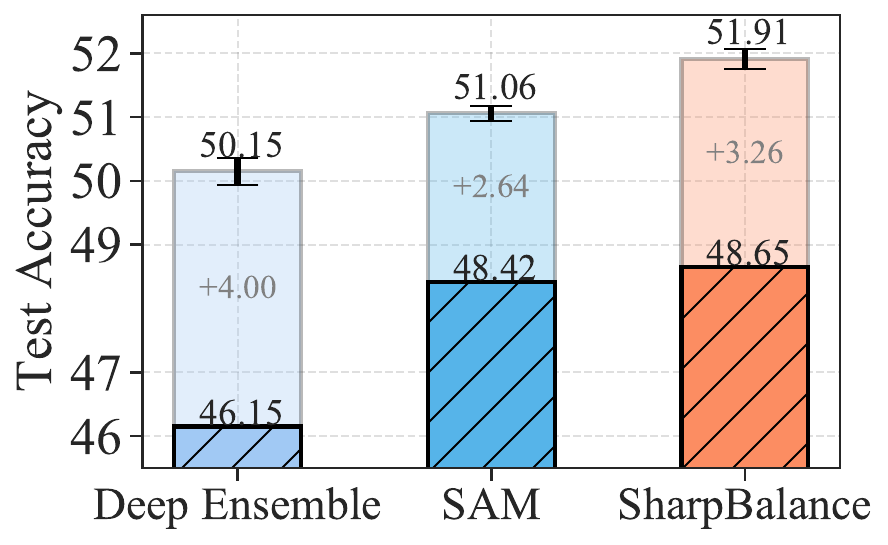}
    \caption{CIFAR100-C}
    \end{subfigure}
    \begin{subfigure}{0.45\textwidth}
    \includegraphics[width=\textwidth]{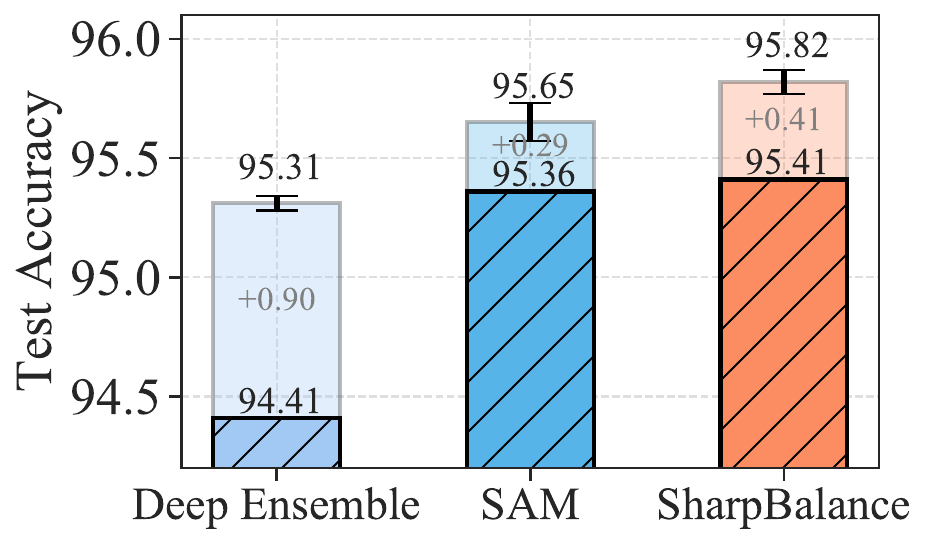}
    \caption{CIFAR10}
    \end{subfigure}
    \begin{subfigure}{0.45\textwidth}
    \includegraphics[width=\textwidth]{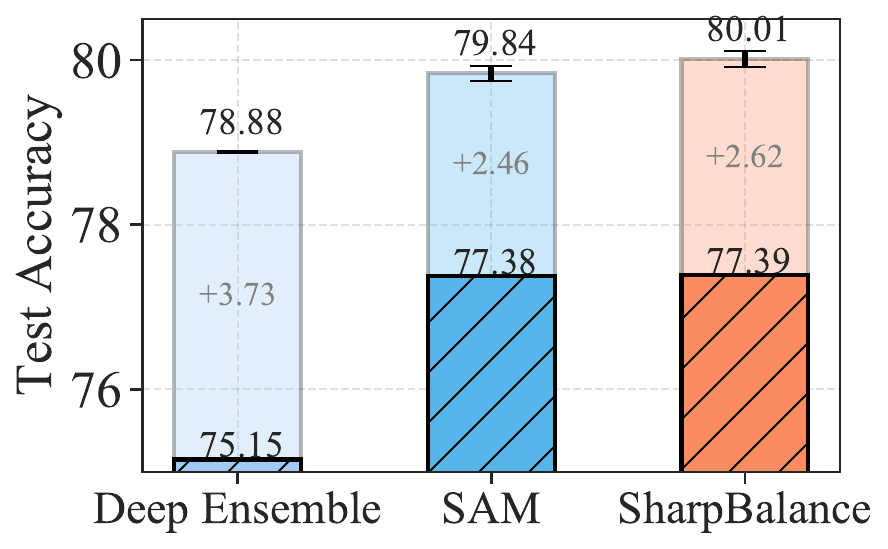}
    \caption{CIFAR100}
    \end{subfigure}
    \caption{The three-member WRN-40-2 ensemble is trained with different methods on two datasets. The first row reports the OOD accuracy and the second row reports the ID accuracy. 
    The lower part of each bar with the diagonal lines represents the individual model performance. 
    The upper part of each bar represents the ensembling improvement.
    The results are reported by averaging three ensembles, and each ensemble is comprised of three models.
 \looseness-1} \vspace{-2mm}
    \label{fig:compare_results_wrn} 
\end{figure}

\subsection{Sharpness-aware set: Hard vs easy examples}
\ourmethod aims to achieve the optimal balance by applying \sam to a carefully selected subset of the data, while performing standard optimization on the remaining samples. In our work, sharpness is determined by the curvature of the loss around the model’s weights, whereas ~\citep{garg2023samples} determines it based on the curvature of the loss around a data point. In Figure \ref{fig:rank_corr}, we rank 1000 samples using both metrics and found a strong correlation between these two.

\begin{figure}[!ht]
\centering
\includegraphics[width=0.5\linewidth,keepaspectratio]{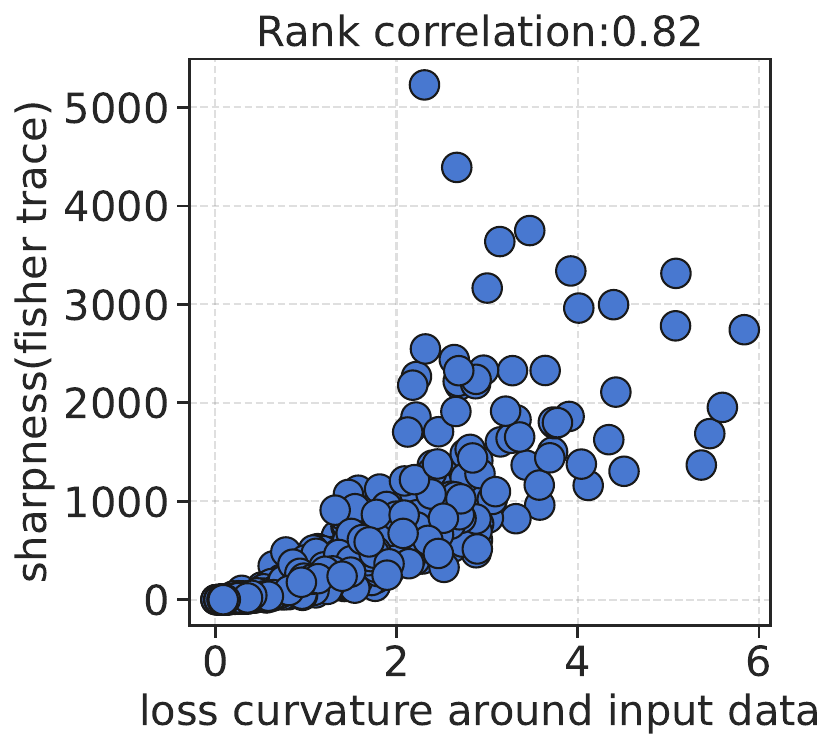}
\caption{Rank correlation between fisher trace and loss curvature around input data} 
\label{fig:rank_corr}
\end{figure}

\subsection{Comparison with more baselines}
We compare \ourmethod with EoA~\citep{arpit2022ensemble} for in-distribution and OOD generalization. We carefully tuned the hyperparameters for EoA~\citep{arpit2022ensemble}. Original EoA~\citep{arpit2022ensemble} fine-tuned a pretrained model; and in our paper, all models are trained from scratch. EoA~\citep{arpit2022ensemble} performs less well on models trained from scratch. Besides, we compare \ourmethod with another \sam baseline: \sam$+$, where three individual models are trained with different $\rho$ values, e.g., 0.05, 0.1, and 0.2, respectively.
From Table~\ref{tb:baselines}, \ourmethod outperforms other baselines botn in-distribution and OOD generalization.

\begin{table*}[!thb]
    \centering
    \resizebox{0.6\linewidth}{!}{
    \begin{tabular}{llcc}
        \toprule
        Dataset & Method & ACC & cACC \\
        \midrule
        & \sam$+$ & 96.03 & 76.29\\
        CIFAR10 & EoA & 95.55 & 75.57 \\
        \rowcolor{gray!15}
        & \ourmethod & 96.18 \textcolor{blue}{(+0.15)} & 77.32 \textcolor{blue}{(+1.03)} \\
        \midrule
        & \sam$+$ & 79.67 & 51.28 \\
        CIFAR100 & EoA & 79.53 & 51.45 \\
        \rowcolor{gray!15}
        & \ourmethod & 79.84 \textcolor{blue}{(+0.17)} & 52.46 \textcolor{blue}{(+1.01)}    \\
        \bottomrule
    \end{tabular}
    }
    \caption{\ourmethod outperforms EOA and \sam$+$ both in-distribution and OOD generalization on CIFAR10 and CIFAR100.}
    \label{tb:baselines}
\end{table*}
\section{Experiments Compute Resources}~\label{sec:computation}

All codes are implemented in PyTorch, and the experiments are conducted on 3 Nvidia Quadro RTX 6000 GPUs for training an ensemble of 3 models. Compared to \sam, our method adds a minimal computational cost. The extra time comes from using Fisher trace to compute the per-sample sharpness. Therefore, computing the per-sample sharpness requires one single forward pass and one backward pass. We report the additional training cost in Table~\ref{tb:computation}. \ourmethod only increases the training time by 1\%: $0.83 \ (84.48 - 0.83) \times 100\% \approx 1\% $.

\begin{table*}[!thb]
    \centering
    \resizebox{0.5\linewidth}{!}{
    \begin{tabular}{cc}
        \toprule
        Additional training cost & Total training cost \\
        \midrule
        0.83 min & 84.48 min \\
        \bottomrule
    \end{tabular}
    }
    \caption{Additional training cost introduced by \ourmethod. We train a ResNet18 on CIFAR10 for 200~epochs.}
    \label{tb:computation}
\end{table*}

\end{document}